\documentclass{article}

\usepackage{microtype}
\usepackage{graphicx}
\usepackage{subfigure}
\usepackage{booktabs} \usepackage[T1]{fontenc}
\usepackage{lmodern}

\usepackage{hyperref}

\usepackage{amsmath}
\usepackage{amssymb}
\usepackage{mathtools}
\usepackage{amsthm}

\usepackage[capitalize,noabbrev]{cleveref}

\theoremstyle{plain}
\usepackage{thm-restate}
\newtheorem{theorem}{Theorem}[section]
\newtheorem{proposition}[theorem]{Proposition}
\newtheorem{lemma}[theorem]{Lemma}

\newtheorem{corollary}[theorem]{Corollary}
\newtheorem{conjecture}[theorem]{Conjecture}
\theoremstyle{definition}
\newtheorem{definition}[theorem]{Definition}

\newtheorem{ec}[theorem]{Empirical Claim}
\theoremstyle{remark}
\newtheorem{remark}[theorem]{Remark}

\usepackage[textsize=tiny]{todonotes}

\usepackage{macros}
\usetikzlibrary{tikzmark,calc}

\newif\ificml
\icmlfalse

\usepackage{fullpage}
\usepackage[round]{natbib}
\usepackage{authblk}

\title{Transformers, parallel computation, and logarithmic depth}

\author[1]{Clayton Sanford}
\author[1]{Daniel Hsu}
\author[2]{Matus Telgarsky}

\affil[1]{Department of Computer Science, Columbia University, New York, NY, USA}
\affil[2]{Courant Institute, New York University, New York, NY, USA}

\usepackage{float}

\begin{document}
\maketitle

\begin{abstract}

We show that a constant number of self-attention layers can efficiently simulate---and be simulated by---a constant number of communication rounds of \emph{Massively Parallel Computation}.
As a consequence, we show that logarithmic depth is sufficient for transformers to solve basic computational tasks that cannot be efficiently solved by several other neural sequence models and sub-quadratic transformer approximations. We thus establish parallelism as a key distinguishing property of transformers.
 \end{abstract}

\section{Introduction}

The transformer \citep{vsp+17} has emerged as the dominant neural architecture for many sequential modeling tasks such as machine translation~\citep{Radford2019LanguageMA} and protein folding~\citep{jumper2021highly}.
Reasons for the success of transformers include suitability to modern hardware and training stability: unlike in recurrent models, inference and training can be efficiently parallelized, and training is less vulnerable to vanishing and exploding gradients. 
However, the advantages of transformers over other neural architectures can be understood more fundamentally via the lens of \emph{representation}, which regards neural nets as parameterized functions and asks what they can efficiently compute.

Many previous theoretical studies of transformers establish (approximation-theoretic and computational) universality properties, but only at large model sizes~\citep{ybrrk20,perez2021attention}.
These results are not unique to transformers and reveal little about which tasks can be solved in a \emph{size-efficient} manner.
Several other works \citep[e.g.,][]{hahn20, ms22-log-prec, sht23} give fine-grained representational results in the scaling regime where context length grows but model depth is constant.
In this regime, basic algorithmic tasks like matching parentheses and evaluating Boolean formulas are impossible.

In this work, we identify parallelism as a key to distinguishing transformers from other architectures.
While recurrent architectures process their inputs serially, transformers allow independent interactions between the input tokens, mediated by the inner products between query and key embeddings in self-attention units.
We leverage this property of self-attention to establish a formal connection between transformers and \emph{Massively Parallel Computation (MPC)}~\citep{ksv10}.
Concretely, we design transformers that simulate MPC protocols (and vice versa), and in doing so,
we exhibit a wide range of computational tasks that are solved by logarithmic-depth transformers, including tasks that cannot be efficiently solved with other architectures such as graph neural nets and recurrent models.

\subsection{Our results}
We advance the understanding of transformers' representational capabilities with the following results.
\begin{enumerate}[leftmargin=*,itemsep=0pt,topsep=0pt,parsep=3pt]
\item The algorithmic capabilities and limitations of logarithmic-depth transformers are captured by the MPC model (\Cref{sec:mpc-equivalence}).
  \item There is a simple sequential task that (i) is solved by (and, empirically, learned from data using) logarithmic-depth transformers, but (ii) \emph{cannot} be efficiently solved by several alternative architectures (\Cref{sec:k-hop,sec:other-models}).
\end{enumerate}

In more detail, our first collection of results, Theorems~\ref{thm:mpc-simulation} and \ref{thm:transformer-simulation}, show that any $R$-round MPC protocol can be implemented by a transformer of depth $O(R)$, and that any depth-$L$ transformer can be simulated by an $O(L)$-round MPC protocol.
The former implies that several graph problems are solved by logarithmic-depth transformers (\Cref{cor:connected}); the latter implies the near-optimality of these transformers (\Cref{cor:connectivity-hardness}) conditional on a well-known conjecture about the limitations of MPC algorithms (\Cref{conj:cycle}).
A key technical step (\Cref{lemma:routing-block}) shows how transformers can implement the simultaneous message-passing used in MPC protocols to communicate between machines.
While previous works \citep{sht23} have used communication complexity to understand the representational limitations of self-attention layers, our results show the benefits of the communication lens for understanding the strengths of transformers as well.

Our second set of results concern the \emph{$k$-hop induction heads} task, a synthetic sequential task that draws inspiration from the induction heads primitive of \citet{eno21}.
The theoretical results of \Cref{sec:k-hop} prove that depth $L = \Theta(\log k)$ is necessary and sufficient for efficient transformer representation.
An accompanying empirical investigation reveals that transformers trained on the task obey the same threshold and recover a similar model to the theoretical construction.
In contrast, \Cref{sec:other-models} illustrates that non-parallelizable recurrent architectures---including state-space models like Mamba \citep{gd23}---are unable to solve the task in a size-efficient manner.
Moreover, well-known transformer models with computationally-efficient alternatives to self-attention, like Performer \citep{cld+22} and Longformer \citep{bpc20}, and shallow transformers with chain-of-thought prompting sacrifice their abilities to implement parallel algorithms, as evidenced by their proven inability to solve this task.

\subsection{Related work}
Some of the types of lower bounds we sought in this work were inspired by
the literature on depth-separation for feed-forward neural networks~\citep[e.g.,][]{es16, daniely17, telgarsky16}, which exhibit functions that are efficiently approximated by deep networks, but not by shallower networks. 

Many theoretical approaches have been used to understand the representational capabilities of transformers and self-attention units in various scaling regimes. 
Some works model (variants of) transformers as machines for recognizing formal languages, such as the Dyck languages~\citep{hahn20, bag20, yppn21, haf22} and star-free regular languages \citep{acy23}.
These approaches reveal inability of fixed-size transformers to handle arbitrarily long inputs.
Other works show how transformers can simulate finite-state automata \cite{lagkz22} with logarithmic depth, and Turing machines with (unrolled) depth (or chain-of-thought length) scaling polynomially with total runtime \cite{wcm21, malach23, ms23-cot}.
However, it is unclear if these results are near optimal or even transformer-specific.

Theoretical results about the limitations of constant-depth transformers have been articulated by way of analogy to circuit complexity \citep{ms22-parallelism, mss22, ms22-log-prec, strobl23, smwca23}, implying the inability of constant-depth transformers to solve tasks like graph connectivity and Boolean formula evaluation.
Other works characterize the representational capabilities of one-layer transformers~\citep{lcw21, sht23}, but these approaches do not apply to deeper models.
\citeauthor{sht23} study multi-headed attention using communication complexity, a framing that informs this work's connection to distributed computing.

The MPC model~\cite{ksv10,beame2017communication,goodrich2011sorting,andoni2014parallel,im2023massively} was introduced to study distributed computing frameworks such as MapReduce~\cite{dean2004mapreduce}.
A major goal is to design protocols that use few rounds of communication for setups in which each machine's local memory is sublinear in the input size.
Many advances have been made in MPC algorithms for important problems~\citep[see, e.g.,][for a recent survey]{im2023massively}.
However, a basic problem that has resisted progress is connectivity in sparse graphs, where all MPC protocols in this memory regime appear to require $\Omega(\log n)$ rounds for input graphs on $n$ vertices.
Lower bounds in MPC and related models were studied by \citet{beame2017communication}, \citet{rvw18}, and \citet{cmt20}.
The conjectured impossibility of $o(\log n)$-round protocols for connectivity is now used as basis for conditional lower bounds~\citep{gku19}.

Simulation of transformers by recurrent models~\citep{ohas24} and simulation of graph neural nets (GNNs) by transformers~\citep{jts22} offer some coarse-grain insight into the relationship between these architectures, but separations are not implied by these previous works.
Our connection between transformers and MPC is most similar to that established by \citet{loukas19} between GNNs and the \textsc{Congest} model of distributed computation.
Both works establish positive and negative results by identifying neural architectures with communication protocols.
In \Cref{sec:gnns}, we show that the MPC connection allows transformers solve graph connectivity more efficiently than GNNs.

Our $k$-hop induction heads task is designed as a $k$-fold composition of its standard analogue \citep{eno21}.
It is similar to a special case of the LEGO reasoning task \cite{zbbegw23}, which reveals the super-linear benefit of depth with respect to $k$; in our case, we theoretically and empirically exhibit an exponential benefit. 
We also draw a connection to the well-studied problem of pointer-chasing \citep{papadimitriou1982communication,duris1984lower,nw93}, which enables the proof of our separation between parallel and serial architectures.
Our fine-grained empirical interpretability analysis for synthetic tasks draws inspiration from similar approaches for the analysis of sequential algorithms like sorting and reversal \citep{lm22}.

\section{Preliminaries}\label{sec:prelims}

\subsection{Transformers}\label{ssec:transformers}

We first define a self-attention head, the core primitive of a transformer.
The \emph{softmax} operator is $\sm(v) =  (\exp(v_1), \dots, \exp(v_N)) / \sum_{j=1}^N \exp(v_j)$ for $v \in \R^N$.
We apply softmax to matrices $A \in \R^{N \times N}$ row-wise, i.e. $\sm(A)_i = \sm((A_{i,1},\dotsc,A_{i,N}))$.

\begin{definition}[Self-attention head]\label{def:attn}
A \emph{self-attention head} is a mapping $f_{Q, K, V}: \R^{N \times m} \to \R^{N \times m}$ defined by
\[f_{Q, K, V}(X)= \sm(Q(X) K(X)^\T) V(X) \]
and parameterized by row-wise
\emph{query}, \emph{key}, and \emph{value embeddings} $Q, K, V \colon \R^{N\times m} \to \R^{N \times m}$ (e.g., $Q(X) = (Q_1(X_1), \dots, Q_N(X_N))$.
Let $\attn{m}{N}$ denote the set of all self-attention heads with embedding dimension $m$ and context length $N$.
\end{definition}

A transformer composes $L$ layers of $H$ self-attention heads per layer, plus an output multi-layer perceptron (MLP). 

\begin{definition}[Transformer]\label{def:tran}
A \emph{transformer} is a mapping $T: \R^{N \times \din} \to \R^{N \times \dout}$ specified by self-attention heads $(f_{\ell, h} \in \attn{m}L)_{\ell \in [L], h \in [H]}$ and an element-wise output MLP $\psi = (\psi_1, \dots, \psi_N): \R^{N \times m} \to \R^{N \times \dout}$.
Upon input $X \in \R^{N \times \din}$, the transformer computes intermediate embeddings $X^0, \dots, X^L \in \R^{N \times m}$ with $X^0 = X$ and \[X^\ell = X^{\ell -1} + {\sum}_{h=1}^H f_{\ell,h}(X^{\ell -1}),\]
and returns $T(X) = \psi(X^L)$ as output.
Let $\tran{m, L, H, \din, \dout}N$ denote the set of all such transformers, and $\tran{m, L, H}N := \tran{m, L, H, 1, 1}N$.
\end{definition}

\paragraph*{Modeling assumptions.}

We treat the transformer as a computational model that permits arbitrary element-wise computation,
but restricts the manner in which multiple elements are processed together.
This manifests in our decisions to model query/key/value embeddings and MLPs as arbitrary functions on the embedding space;
\citet{loukas19} employs a similar modeling assumption for GNNs. Note that the element-wise embeddings and MLPs may be index-specific, obviating the need for positional embeddings.

Our theoretical results
cover the scaling regime where
the context length $N$
is
the
main asymptotic parameter;
while
the embedding dimension $m$, the number of heads $H$, and the depth $L$ grow sub-linearly in $N$.
This reflects real-world trends in large-language models, where context length has sharply increased in recent years.

Throughout, we assume all intermediate computations in transformers are represented by $p$-bit precision numbers for $p = \Theta(\log N)$.
Limiting the precision is consistent with recent practice of using low-precision arithmetic with transformers~\citep[e.g.,][]{wang2022quantformer,dettmers2022llm}.
We discuss this precision assumption in greater detail in \Cref{assec:transformers}, along with other minor technical assumptions (such as the inclusion of a ``start token'' for mathematical convenience).

\paragraph*{Masked transformers.}

We also consider \emph{masked self-attention}, where only certain inner products influence the softmax output.
Let $\Lambda \in \set{-\infty, 0}^{N \times N}$ be a \emph{masking matrix} with at least one zero entry in every row.
Then, a \emph{$\Lambda$-masked self-attention} unit is defined by 
\[f^\Lambda_{Q, K, V}(X)= \sm(Q(X) K(X)^\T + \Lambda) V(X).\]
Let $\lattn{m}{N}$ and $\ltran{m, L, H}N$, respectively, denote the sets of all $\Lambda$-masked self-attention heads and all transformers comprised of those heads.
We define \emph{causally-masked transformers} by $\mattn{m}{N} := \Gamma\mhyphen\attn{m}{N}$ and $\mtran{m, L, H}N := \Gamma\mhyphen\tran{m, L, H}N$, where $\Gamma$ is the lower-triangular mask with $\Gamma_{i, j} = 0$ iff $i \geq j$.

\subsection{Massively Parallel Computation model}

We use the definition of MPC from \citet{asswz18}.

\begin{definition}[MPC protocol]\label{def:mpc}
For any global and local memory constants $\gamma, \delta > 0$, a \emph{$(\gamma, \delta)$-MPC protocol} for a function $f: \pword^{\nin} \to \pword^{\nout}$ specifies a distributed computing protocol for $q = \Theta(\nin^{1 + \gamma - \delta})$ machines, each with $s = O(\nin^\delta)$ words\footnote{We assume the word size is $p = \Theta(\log \nin)$ bits.
    For convenience, we regard words as elements of $\pword$ (integers mod $2^p$).} of local memory to jointly compute $f(\inp)$ for any given $\inp \in \pword^{\nin}$ as follows.
The $\inp \in \pword^{\nin}$ is distributed across the local memories of the first $\lceil\nin / s\rceil$ machines.
  Computation proceeds in rounds.
  In each round, each machine computes an arbitrary function of its local memory to prepare at most $s$ words to send to other machines;
  messages are simultaneously transmitted, and the protocol ensures that each machine receives at most $s$ words at the end of the round.
  After the final round, the $\outp = f(\inp) \in \pword^{\nout}$ is in the local memories of the first $\lceil\nout / s\rceil$ machines.
  See Figure~\ref{fig:mpc} for details.
\end{definition}

\ificml
\else
\begin{figure*}
  \small
  \noindent
  \fbox{\begin{minipage}{\textwidth}
      \setlength{\parskip}{0pt}
      \begin{itemize}[leftmargin=*,]
        \item $\inp = (\inp_1, \dots, \inp_{\nin}) \in \pword^{\nin}$ is distributed across local memories of machines $1 \leq i \leq \lceil\tfrac{\nin}s\rceil$:
          \[
            \machin{1}_i = \{ (\inp_{\iota}, \iota): \iota \in \set{(s{-}1)i{+}1, \dots, \min\set{\nin, si}} \}.
          \] 

        \item For round $r = 1,\dotsc,R$:
          \begin{itemize}[leftmargin=*]
            \item Each machine $i$ computes messages $(\msgout{r}_{i, j})_{j=1,2,\dots}$ to send to machines $(\dest{r}_{i, j})_{j=1,2,\dots}$ as function of $\machin{r}_i$:
              \begin{align*}
                \machout{r}_i & = \local_{r,i}(\machin{r}_i) = \{ (\msgout{r}_{i, j}, \dest{r}_{i, j}) \in \pword^{d_j} \times [q] : j = 1,2,\dotsc \} ; \\
                              & \text{${\sum}_j d_j \leq s$ is ensured} .
              \end{align*}

            \item All messages are simultaneously transmitted; the messages in local memory of machine $i$ for round $r+1$ are:
              \begin{align*}
                \machin{r+1}_i & = \{ (\msgg,\src) : (\msgg, i) \in \machout{r}_{\src} \} ; \\
                               & \text{${\sum}_{(\msgg,\src) \in \machin{r+1}_i} |\msgg| \leq s$ is ensured} .
              \end{align*}

          \end{itemize}

        \item $\outp = f(\inp)$ comes from
          \begin{equation*}
            \machin{R+1}_i = \set{(\outp_{\iota}, \src): \iota \in \set{(s - 1)i{+}1, \dots, \min\set{\nout, si}}}
          \end{equation*}
          for $1\leq i \leq \lceil\tfrac{\nout}s\rceil$.

      \end{itemize}
  \end{minipage}}
  \caption{\normalsize Formal execution of an MPC protocol for computing $f \colon \pword^{\nin} \to \pword^{\nout}$. ($|\msgg|$ is the number of words in $\msgg$.)}
  \label{fig:mpc}
\end{figure*}
\fi

Our negative results in \Cref{ssec:tr-simulation} are conditional on the well-known ``one-versus-two cycle'' conjecture \citep{beame2017communication,rvw18,gku19}.

\begin{conjecture}[see, e.g., \citealp{gku19}]\label{conj:cycle}
For any $\gamma > 0$, $\delta < 1$, and $N$, if $\pi$ is an $(\gamma, \delta)$-MPC protocol that distinguishes a single cycle on $N$ nodes and a union of two cycles each on $N/2$ nodes, then $\pi$ uses $\Omega(\log N)$ rounds.
\end{conjecture}

\subsection{Graphs as sequential inputs}
When providing a graph $G = (V, E)$ as input to transformers or MPC protocols, we serialize $G$ as a sequence in $[|V|]^{2|E|}$ that encodes each edge as a pair of vertex tokens.
The resulting transformer has $N = 2 |E|$ and $\din = 1$, and the resulting MPC protocol has $\nin = 2 |E|$.

\section{Relating transformers and MPC}\label{sec:mpc-equivalence}

We coarsely characterize the computational power of transformers in a certain size regime by establishing a bidirectional relationship between transformers and MPC.
Theorems~\ref{thm:mpc-simulation} and \ref{thm:transformer-simulation} show that any MPC protocol can be simulated by a transformer, and vice versa.
As corollaries (\Cref{cor:connected,cor:connectivity-hardness}), we obtain tight upper and lower bounds on the depth of bounded-size transformers for computing connected components in graphs.

\subsection{Simulation of MPC protocols by transformers}\label{ssec:mpc-simulation}

The following theorem shows that any MPC protocol $\pi$ with sublinear local memory can be simulated by a transformer whose depth $L$ is linear in the number of rounds $R$ of $\pi$, and embedding dimension $m$ is polynomial in the local memory size $s = O(N^\delta)$ of machines used by $\pi$.

\begin{theorem}\label{thm:mpc-simulation}
For constants $0 < \gamma < \delta < 1$ and any deterministic $R$-round $(\gamma, \delta)$-MPC protocol $\pi$ on $\nin$ input words and $\nout \leq \nin$ output words, there exists a transformer $T \in \tran{m, L, H}{N}$ with $N = \nin, m = O(\nin^{4\delta}\log \nin), L = R + 1, H = O(\log\log \nin)$ such that $T(\inp)_{:\nout} = \pi(\inp)$ for all $\inp \in \pword^N$.
\end{theorem}

The theorem provides a non-trivial construction in the strongly sub-linear local memory regime when $s = O(N^{1/4 - \epsilon})$ for any $\epsilon>0$.\footnote{Applying \Cref{thm:mpc-simulation} when $\delta \geq \frac14$ yields transformers with embedding dimension $m \geq N$, which trivializes the transformer architecture and negates any advantages of depth under our MLP universality assumption. This is due to the fact a transformer with $N$-dimensional embeddings could aggregate the entire input sequence $X \in \R^N$ in a single embedding and use its output MLP to compute any arbitrary function on that input.}
Whether the simulation can be improved to $m = O(N^{1-\epsilon'})$ for some $\epsilon'>0$ whenever $s = O(N^{1-\epsilon})$ is an interesting question for future work.

\paragraph*{\Cref{thm:mpc-simulation} proof overview.}
At a high level, the proof in \Cref{assec:thm-mpc-simulation} entails simulating each round of parallel computation with a single-layer transformer and applying those constructions serially to $\inp$. 
The local computation on each machine
(represented by $\machout{r}_i = \local_{r,i}(\machin{r}_i)$)
is directly encoded using element-wise query/key/value embeddings.

The crux of the proof involves the simulation of a \emph{routing protocol}  to determine $\machin{r+1}$ from $\machout{r}$.
We construct a self-attention unit that ensures that an encoding of a sequence of addressed messages from each machine are properly routed to their destinations.\footnote{This routing between machines uses the all-pairs structure of self-attention and may not admit a subquadratic approximation.}

For any message size $\beta$, message count bound $s$, and number of tokens $N$, we say that $(\outg, \inc) \in \R^{N \times m} \times \R^{N \times m}$
is a \emph{valid $(\beta, s)$-routing} if, for each $i \in [N]$,
the $i$-th row of $\outg$ (resp.~$\inc$) is the vector encoding of some $\outg_i \subset \pword^\beta \times [N]$ (resp.~$\inc_i \subset \pword^\beta \times [N]$) such that
\[\inc_{i} = \set{(\msg, \src): (\msg, i) \in \outg_{\src}},\]
and each of $\inc_{i}$ and $\outg_{i}$ has cardinality at most $s$.\footnote{We abuse notation by writing $\dst \in \outg_i$ to mean there exists some $\msg$ such that $(\msg, \dst) \in \outg_i$.}

\begin{restatable}{lemma}{lemmaroutingblock}\label{lemma:routing-block}
For any $\beta, s, N \in \N$, there exists a transformer $\routeb_{\beta, s} \in \tran{m, 1, 1}{N}$ with $m = O(s^4 \beta \log N)$ satisfying $\routeb_{\beta, s}(\outg) = \inc$ for any valid $(\beta, s)$-routing $(\outg, \inc)$.
\end{restatable}

The proof of \Cref{lemma:routing-block} appears in \Cref{assec:lemma-routing-block} and combines two key techniques: sparse propagation and multiple hashing.
The former is a simple variant of the ``sparse averaging'' task of \citet{sht23}, which simultaneously computes $N$ averages over subsets of inputs; this task is solved a single self-attention head with small embedding dimension (\Cref{prop:qsp}).
Using sparse propagation, we construct a self-attention head that averages the $\leq s$ encodings of each $\inc_\src$ for every $\src \in \inc_i$. 
In order to ensure that we can decode that average of encodings, we apply error-correction by encoding each $\outp_i$ in a sparse and redundant manner, where each outgoing messages appears as multiple copies of the same addressed ``packet.''

\paragraph*{Application: connectivity with log-depth transformers.}
As an immediate consequence of \Cref{thm:mpc-simulation}, any graph problem solvable with a logarithmic number of rounds of MPC computation (and local memory $s$) is also computable by a logarithmic depth transformer (and embedding dimension $\tilde O(s^4)$).
The following result---which bounds transformer depth needed to compute connected components of a graph $G$---follows from Theorem~6.2 of \citet{cc22}, which derandomizes an MPC algorithm of \citet{behnezhad2019massively}, and \Cref{thm:mpc-simulation}.

\begin{corollary}
\label{cor:connected}
For any constant $\epsilon \in (0, 1)$ and any $D \leq N$, there exists a transformer in $\tran{m, L, H}{N}$ with $m = O(N^\epsilon)$, $H = O(\log\log N)$, and $L = O(\log D)$ that identifies the connected components of any input graph $G = (V, E)$ with $|V|, |E| = O(N)$ where each connected component has diameter at most $D$.
\end{corollary}

\citeauthor{cc22} also give efficient MPC algorithms for other related problems (e.g., spanning forest), so we obtain efficient transformers for these problems, too (\Cref{assec:gralgos}).

\subsection{Simulation of transformers by MPC protocols}\label{ssec:tr-simulation}

The following theorem shows
that MPC protocols can simulate transformers and prove depth lower bounds on transformers, conditioned on \Cref{conj:cycle}.
We get, as a corollary, the conditional optimality of the transformer depth bound in \Cref{cor:connected}.

\begin{theorem}\label{thm:transformer-simulation}
For any transformer $T \in \tran{m,L, H}N$ (or $\ltran{m,L,H}N$) with $m H = O( N^\delta)$ for $\delta \in (0,1)$ and any $\delta' \in (\delta, 1)$, there exists a $O(\frac{L}{\delta' - \delta})$-round $(1 + \delta', \delta')$-MPC protocol with $q = O(N^2)$ machines with $s = O(N^{\delta'})$ local memory for computing $T$.
\end{theorem}

\Cref{thm:transformer-simulation} demonstrates that the algorithmic capabilities of transformers are no stronger than those of MPC protocols with a quadratic scaling in the number of machines.   
While \Cref{thm:mpc-simulation,thm:transformer-simulation} do not jointly provide a sharp characterization of the two computational models, the reductions are tight enough to provide strong evidence for the optimality of the connected components construction of \Cref{cor:connected}.   

\paragraph*{\Cref{thm:transformer-simulation} proof overview.}

At a high-level, the proof constructs an MPC protocol that simulates a self-attention layer by separating the computation of MLPs and attention matrices into three separate categories of machines.
\begin{itemize}[leftmargin=*,itemsep=0pt,topsep=0pt,parsep=3pt]
\item Each input token is provided to its own \emph{token machine}, responsible for preparing the query/key/value embeddings.
\item Each pair of tokens is associated with an \emph{inner product machine} that will compute the inner product between their respective query and key embeddings.
\item \emph{Propagation machines} ensure that embeddings are routed to the proper inner product machine and compute outputs of each softmax unit.
\end{itemize}
The proof gives the communication protocol for these machines, shows how they simulate a layer of self-attention in $O(1/(\delta'-\delta))$ rounds, and establishes the sufficiency of $O(N^2)$ machines with $O(N^{\delta'})$ local memory.

\paragraph*{Application: conditional optimality of \Cref{cor:connected}.}

Assuming the well-established \Cref{conj:cycle},
we prove an $\Omega(\log D)$ lower bound on the depth of parameter-efficient transformers for
determining connectivity of graphs where connected components may have diameter up to $D$.
\begin{restatable}{corollary}{corconnectivityhardness}\label{cor:connectivity-hardness}
  Let $\epsilon \in (0,1)$ be any constant, and let $D \geq N^{\epsilon}$.
  Assume \Cref{conj:cycle}, and suppose there exists $T \in \tran{m, L, H}N$ with $mH = O(D^{1-\epsilon})$ that decides connectivity of any input graph with connected components having diameter $\leq D$.
  Then $L = \Omega(\log D)$.
\end{restatable}

\section{Transformers for $k$-hop induction heads}
\label{sec:k-hop}

We complement the generality of \Cref{sec:mpc-equivalence} by studying, both empirically and theoretically, a specific toy sequential modeling task which will also serve (in \Cref{sec:other-models}) as a problem to separate the representational capabilities of transformers from that of other neural architectures.

This task, called the \emph{$k$-hop induction heads} task, draws inspiration from the original \emph{induction heads} task defined and analyzed on trained language models and in synthetic environments by \citet{eno21} \citep[see also][]{bcbjg23}.
The standard induction heads task completes bigrams auto-regressively by predicting the token that follows the last previous occurrence of the final token in the sequence.
For example, given the input $X = \texttt{\textcolor{blue}{b}\textcolor{red}{a}e\textcolor{blue}{b}\textcolor{magenta}{c}\textcolor{red}{a}\textcolor{blue}{b}e\textcolor{blue}{b}de\textcolor{red}{a}}$,
the standard induction heads task is to complete the final bigram by predicting $\texttt{\textcolor{blue}{b}}$ for the final token.

The $k$-hop induction heads tasks generalizes this mechanism by repeatedly using the completion of a bigram to determine the next bigram to complete. 
In the previous example, the $2$-hop induction heads task is to predict \texttt{\textcolor{magenta}{c}} for the final token:

\vspace{-1em}
\begin{equation*}
  \texttt{\textcolor{blue}{b}\textcolor{red}{a}e\textcolor{blue}{b}\tikzmarknode{c}{\textcolor{magenta}{c}}\textcolor{red}{a}\tikzmarknode{b}{\textcolor{blue}{b}}e\textcolor{blue}{b}de\tikzmarknode{a}{\textcolor{red}{a}}} .
\end{equation*}
\begin{tikzpicture}[remember picture, overlay]
  \draw[-latex,black] ([yshift=0.05em]a.north) to[bend right] ([yshift=0.05em]b.north);
  \draw[-latex,black] ([yshift=0.05em]b.north) to[bend right] ([yshift=0.05em]c.north);
\end{tikzpicture}
\vspace{-1em}

\begin{definition}
  For any finite alphabet $\Sigma$, define the map $\khop \colon \Sigma^N \to (\Sigma \cup \{\bot\})^N$ by $\khop(X)_i = X_{\find_X^k(i)}$ if $\find_X^k(i) \neq 0$ and $\perp$ otherwise, where
  \begin{align*}
    \find_X^1(i) & = \max(\set{0} \cup \set{ j \in \N : j \leq i ,\, X_{j-1} = X_i }) ; \\
    \find_X^k(i) & = \find_X^1(\find_X^{k-1}(i)) \quad \text{for $k\geq2$} .
  \end{align*}
The \emph{$k$-hop induction heads task}
is to compute, for each $i=1,\dotsc,N$, the value of $\khop(X)_i$ from $(X_1,\dotsc,X_i)$.
\end{definition}

We note a similarity to the LEGO tasks of \cite{zbbegw23}, who empirically study the ability of transformers to learn sequential operations on Abelian groups and observe the ability to perform more operations than the depth of the network.

\subsection{Log-depth transformer for $k$-hop induction heads}\label{ssec:khop-theory}

Although $\khop$ appears to requires $k$ steps to solve, we show that it is solved by a transformer of depth $O(\log k)$.

\begin{restatable}{theorem}{thmkhopconstruction}\label{thm:k-hop-construction}
  For any $k \in \N$ and alphabet $\Sigma$ with $|\Sigma| \leq N$, there exists $T \in \mtranscal$ that computes $\khop \colon \Sigma^N \to (\Sigma \cup \{\bot\})^N$ with $m=O(1)$, $L = \floor{\log_2 k} + 2$, and $H = 1$.
\end{restatable}

In contrast to Corollary~\ref{cor:connected}, this construction has constant embedding dimension and is achieved by a causally-masked transformer.
As such, its proof in \Cref{assec:k-hop-construction} depends on other techniques that exploit the simplicity of the problem and build on the induction heads construction of \citet{bcbjg23}, rather than simply applying \Cref{thm:mpc-simulation}.

We give evidence for the optimality of this construction by proving a conditional lower bound using \Cref{thm:transformer-simulation}, as was done in \Cref{cor:connectivity-hardness}.

\begin{restatable}{corollary}{corkhophardness}\label{cor:k-hop-hardness}
  Assuming \Cref{conj:cycle}, for any constants $\xi \in (0,1/2]$ and $\epsilon \in (0,1)$, and any even $k = \Theta(N^\xi)$, every transformer $T \in \mtran{m, L, H}N$ with $mH = O(k^{1-\epsilon})$ that computes $\khop$ has depth $L = \Omega(\log k)$. 
\end{restatable}

\subsection{Log-depth transformer learned from data}\label{ssec:khop-emp}

We empirically assess whether the representational trade-offs elucidated by tasks efficiently solved by parallelizable algorithms have implications for optimization and generalization properties of transformers.
To that end, we trained auto-regressive transformer architectures of varying sizes to solve $\khop(X)$ for a variety of values of $k$ in order to understand how changing depth impacted the performance of the learned models,
the goal being to verify the sufficiency of logarithmic depth, just as in our theory.

In brief, we trained transformers with 500K to 5M parameters and depths $\set{2, 3, 4, 5, 6}$ with Adam to solve $\khop(X)$ for $k \in \set{0, \dots, 16}$ with context length $|N| = 100$ and alphabet size $|\Sigma| = 4$.
We trained the transformers in a multi-task setting, where a single model was trained to predict the sequence $\khop(X)$ auto-regressively when provided with $X$ and $k$ drawn at random. Further experimental details can be found in \Cref{assec:exp-details}, and the experimental code is available at \url{https://github.com/chsanford/hop-induction-heads}.

We found that transformers are indeed capable of learning $\khop$ given sufficient training time, and that the largest learnable $k$ grows exponentially with the depth.
As can be seen in Figure~\ref{fig:depth-body}, a six-layer neural network performs well on all $k \leq 16$, a five-layer on $k \leq 8$, a four-layer on $k \leq 4$, and so forth. 
We further explore these experimental results in \Cref{assec:exp-depth} and observe a performance threshold appears to specifically lie at $\floor{\log_2 k} + 2$ that coincides with \Cref{thm:k-hop-construction}.
This logarithmic dependence of the depth on $k$ persists in a larger-width regime, which is explored in \Cref{assec:exp-width}. 
In the finite sample regime where neural networks are prone to overfit, our investigations in \Cref{assec:exp-finite} note improved generalization in deeper models, which suggests that deeper models have a favorable inductive bias for tasks like $\khop$.

\begin{figure}[t]
\centering
\ificml
\else
\includegraphics[width=0.6\linewidth]{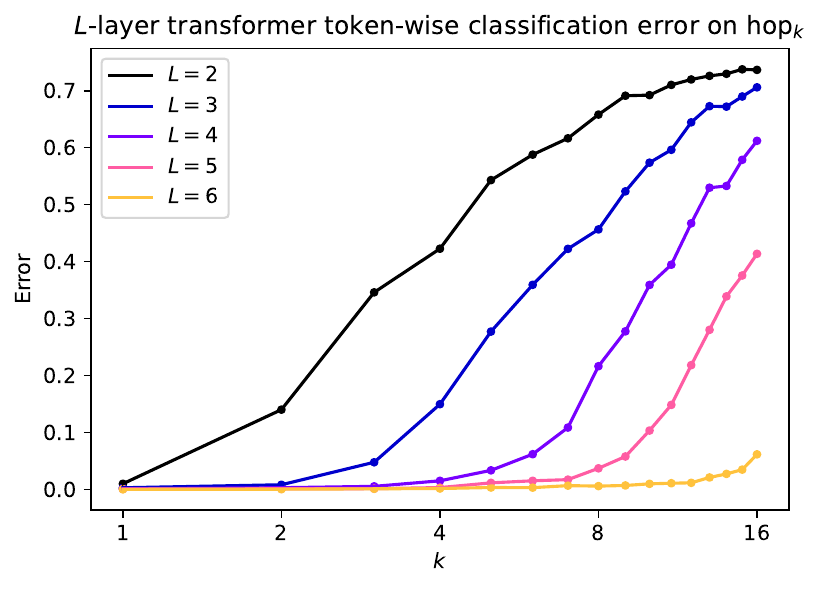}
\fi
\caption{Evaluation of transformers of depths $L \in \set{2, 3, 4, 5, 6}$ trained on a mixture of $\khop$ for $k \in \set{0, \dots, 16}$ evaluated on $n=100$ samples of size $N = 100$ from each $\khop$. Incrementing depth approximately doubles the largest $k$ such that $\khop$ is learnable with small error.}
\label{fig:depth-body}
\end{figure}

Moreover, the learned models are surprisingly interpretable.
We examined the activation patterns of attention matrices, and found close correspondences to useful intermediate products such as $\find^{j}_X$. 
Taken together, these indicate that the learned models mechanistically resemble the construction employed in the proof of  \Cref{thm:k-hop-construction}.
See \Cref{assec:exp-interp} for our investigation of model interpretability.

 \section{Separations between transformers and alternative architectures}\label{sec:other-models}

Sections~\ref{sec:mpc-equivalence} and \ref{sec:k-hop} characterize the representational capability of transformers by providing algorithmic problems
they can solve
with logarithmic depth and small polynomial or constant width.
In contrast, other well-known architectures are unable to solve those same problems in a parameter-efficient manner.
This section provides lower bounds on the parameter complexity of graph neural networks (GNNs), recurrent neural architectures, transformers with computationally efficient alternatives to softmax self-attention, and single-layer transformers with autoregressive chain-of-thought tokens needed to solve graph connectivity and
the $k$-hop task.

\subsection{GNNs need polynomial depth for graph connectivity}\label{sec:gnns}
The bidirectional relationship between transformers and MPC draws inspiration from past work drawing a similar connection between message passing graph neural networks ($\mathsf{GNN}_{mp}$) and the \textsc{Congest} distributed computing model \cite{loukas19}.
Their computation model of $\mathsf{GNN}_{mp}$ for width $m$ and depth $L$ closely resembles our $\tran{m, L, H}N$ in providing a general framework for the analysis of graph neural networks by allowing unbounded computation in each vertex with bounded communication on edges.
On some input graph $G$, vertices send neighbors messages of size at most $m$---which are aggregated and crafted into new messages with MLPs---over $L$ rounds of communication.

By restating Corollary~4.2 of \cite{loukas19}, we demonstrate a sharp contrast in the abilities of GNNs and transformers to solve graph algorithmic tasks.
\begin{theorem}[Corollary~4.2 of \cite{loukas19}]
There exists a graph $G$ with $N$ edges such that any $\mathsf{GNN}_{mp}$ with width $m$ and depth $L$ that determines whether an input subgraph $H$ either (1) is connected or (2) forms a spanning tree of $G$ requires $L \sqrt{m} = \tilde\Omega(N^{1/4})$.
\end{theorem}
While Corollaries~\ref{cor:connected} and \ref{cor:st} demonstrate the ability of transformers to determine whether any input graph is connected\footnote{While the problem of subgraph connectivity for GNNs may at first glance appear more difficult than general graph connectivity for transformers, an implementation of this exact task can be implemented by modifying the protocol \Cref{cor:connected} to remove all edges from the graph that do not belong to $H$.} or to identify a spanning tree with logarithmic depth and small polynomial width (i.e. $m = O(N^{\epsilon})$), GNNs require depth $L = \tilde\Omega(N^{1/4 - \epsilon / 2})$ in the same regime.
This gap is explainable by the fact that transformers on graph inputs $G$ are not bound to pass messages exclusively along the edges of $G$. 
By ``rewiring'' the graphical structure in each layer, transformers can perform aggregation and ``pointer passing'' tasks with greater parametric ease than GNNs.

\subsection{Suboptimality of recurrent architectures for $\khop$}\label{ssec:rnn}

The logarithmic-depth and constant-width transformer implementation of $\khop$ in \Cref{thm:k-hop-construction} cannot be replicated by recurrent neural architectures
\citep{chung2014empirical,279181,turkoglu2021gating},
including not just multi-layer recurrent neural networks (RNNs) but any sequential
    prediction procedure equivalent to them at inference time,
    which includes
state space models such as Mamba \citep{gd23}.

We first consider a family of multi-layer RNNs of depth $L$ and width $m$, consisting of arbitrary MLP units $g_\ell: \R^{m \times m} \to \R^{m \times m}$, which on input $X\in \R^{N \times \din}$ produce output $Y\in \R^{N \times \dout}$ as follows using intermediates $X = Z^0, Z^1, \dots, Z^{L-1}, Z^L = Y \in \R^{N \times m}$\footnote{We assume that $\din, \dout \leq m$ and treat $X$ and $Y$ as if they are padded with zeros.}and hidden states $H^1, \dots, H^L \in \bit^{N \times m}$ with $H^\ell_0 = \vec0$:
\begin{align*}
(Z^\ell_i, H^\ell_i) = g_\ell(Z^{\ell - 1}_i, H^\ell_{i-1}), \ \forall i \in [N], \ell \in [L].
\end{align*}
We provide a polynomial bound on the width and depth of a multi-layer RNN solving $\khop$.

\begin{restatable}{corollary}{corrnn}\label{cor:rnn}
A multi-layer RNN of depth $L$ and width $m$ as above with $Y_N = \khop(X)_N$ satisfies either $L \geq k$ or $m = \Omega(\frac{N}{k^6})$.
\end{restatable}
In contrast to \Cref{thm:k-hop-construction}, which demonstrates that depth $O(\log k)$ transformers with constant width suffice to solve $\khop$ for any $k$, \Cref{cor:rnn} demonstrates that all multi-layer RNNs with width $O(N^{1/7})$ require depth $k$ when $k = O(N^{1/7})$.

Mamba \citep{gd23} can be seen as the combination of three ideas:
(1) a continuous-time dynamics model of sequential prediction,
powerful enough to model Kalman filters, hidden markov models, and many others;
(2) a family of time-discretization schemes;
(3) an unrolling technique to enable efficient linear-time training,
using ideas similar to FlashAttention \citep{flashattention}.
Ultimately, at inference time, the time-discretization step results in an RNN
\citep[see][Algorithm 2 and Theorem 1]{gd23},
and is therefore directly handled by \Cref{cor:rnn}.

This corollary is a near immediate application of a communication complexity fact about the hardness of solving multi-player \emph{pointer-chasing} problems with limited communication among players
\citep{gm09, an21}.
We provide the communication model and this result in \Cref{assec:pc}, and the reductions necessary to prove the above hardness results in \Cref{assec:rnn}.

\subsection{Suboptimality of sub-quadratic attention transformers for $\khop$}\label{ssec:sub}

Due to the quadratic computational cost of computing the attention matrix $\sm(Q(X) K(X)^T) \in \R^{N \times N}$ and the continued desire for ever-larger context lengths, there is substantial interest in improving the computational complexity of the transformer architecture while preserving its expressive capabilities and inductive biases.
As a result, a rich literature has emerged that proposes computationally-efficient alternatives to standard softmax attention.
In this section, we demonstrate how several representative examples of sub-quadratic attention mechanisms lose the ability to perform efficient parallel computation under a logarithmic-depth scaling.

\paragraph*{Kernel-based sub-quadratic attention.}

One approach to computationally-efficient approximation of 
transformers
are \emph{kernel-based sub-quadratic attention} mechanisms such as Performer \citep{cld+22}, and Poly-Sketchformer \citep{kmz23}. 
Both approximate the attention matrix $\sm (Q(X) K(X)^\T  )$ with a low-rank matrix $Q'(X) K'(X)^\T$ where $Q', K': \R^{m} \to \R^{m'}$ are applied element-wise.
For sufficiently small $m' \ll N$, $Q'(X) K'(X)^\T V(X)$ can be computed
efficiently
by first computing $K'(X)^\T V(X) \in \R^{m' \times m}$, bounding the total runtime as $O(N m m')$, rather than $O(N^2 m)$.

Let $\ktran{m, m', L, H}N$ denote all $H$-headed $L$-layer transformer whose softmax attention modules are replaced by kernel-based sub-quadratic attention.
We demonstrate the limitations of $\ktran{m, m', L, H}N$ by showing that, unlike $\tran{m, L, H}N$, they have no depth-efficient implementation of $\khop$.

\begin{restatable}{corollary}{corkernel}\label{cor:kernel}
Any $T \in \ktran{m,m', L, H}N$ with $T(X)_N = \khop(X)_N$ satisfies either $L \geq k$ or $mm' H p = \Omega(\frac{N}{k^6})$.
\end{restatable}

Under a parameter-efficient regime where $mpHL = O(N^\epsilon)$,
solving $\khop$ for $k = \Theta(N^\epsilon)$ necessitates kernel feature dimension $m' = \Omega(N^{1 - 9 \epsilon})$, which forces each attention unit to compute
an $N \times N^{1 - 9 \epsilon}$ matrix, yielding a nearly quadratic runtime. 
We prove \Cref{cor:kernel} in \Cref{assec:sub} using a similar pointer chasing reduction.

\paragraph*{Masking-based sub-quadratic attention.}

Another method that reduces the computational cost of transformers is to used masked models of $\ltran{m, L, H}N$ for a sparse mask $\Lambda$.
The Longformer architecture \citep{bpc20} introduces a particular masked architecture that combines sliding windows with sparse unmasked global tokens.
Put concretely, for window radius $w$ and global frequency $g$, let $\Lambda^{w, g} \in \set{-\infty, 0}^{N \times N}$ be masking matrix with \[\Lambda^{w, g}_{i, j} = \begin{cases}0 & \text{if} \ |i - j| \leq w \ \text{or} \ j \equiv 0 \pmod g, \\ -\infty & \text{otherwise.}\end{cases}\]
Then, the output of a single unit of $\Lambda^{w,g}$-masked attention is computable in time $O((w + \frac{N}g)Nm)$. 

\begin{restatable}{corollary}{corlong}\label{cor:long}
Any $T \in \gattn{m, L, H}N{\Lambda^{w, g}}$ with $T(X)_N = \khop(X)_N$ satisfies either $L \geq k$ or $(w + \frac{N}{gk})mHp = \Omega(\frac{N}{k^6})$.
\end{restatable}
Like kernel-based attention, sparsely-masked attention models fail to efficiently compute $\khop$.
Similarly, in the same parameter-efficient regime, a Longformer must have either $w = \Omega(N^{1 - 9\epsilon})$ or $g = O(N^{9\epsilon})$, which jointly ensures that the masked matrix has at least $\Omega(N^{2 - 9\epsilon})$ entries and diminishes any computational advantages.
This proof also appears in \Cref{assec:sub}.

\subsection{Limitations of 1-layer transformers with chain-of-thought}\label{ssec:cot}

While most of the paper considers transformers as sequence-to-sequence models, we can also frame them as auto-regressive models performing next-token-prediction with chain-of-thought prompting. 
In this regime, a single causally-masked transformer aims to compute a function of its input by repeatedly predicting the next token, appending previously predicted tokens to the end of the input.
In doing so, a function is computable if there exists an intermediate \emph{chain-of-thought} produced by the model that eventually reaches the answer.

\begin{definition}
We say that $T \in \mtran{m, L, H}{N + \Ncot}$ computes
$f : \Sigma^{N + \Ncot} \to \Sigma^N$, where the additional $N$ tokens denote chain-of-thought,
if for every $X \in \dom(f)$, there exists
$\Xcot \in \Sigma^{\Ncot}$
such that $T(X \circ \Xcot)_{N:N+\Ncot} = (\Xcot \circ f(X))$.
\end{definition}

The theoretical capabilities of chain-of-thought augmented transformers to simulate finite-state automata and Turing machines have been studied \citep{malach23, ms23-cot}, but the comparative capabilities of shallow models with chain-of-thought prompting and deep sequential models are unknown.
In contrast to the fact that any transformer with $\Ncot$ tokens can be simulated by a sequential model with depth scaled by $\Ncot$, we show that deep transformers cannot necessarily be efficiently simulated by shallow chain-of-thought models. 
We do so by demonstrating that a linear amount of chain-of-thought prompting in $k$ is necessary to solve $\khop(X)_N$,
and also sufficient.

\begin{restatable}{corollary}{corcothard}\label{cor:cot-hard}
Any $T \in \mtran{m, 1, H}{N + \Ncot}$ that computes $\khop(X)_N$ with $\Ncot$ tokens of chain-of-thought requires either $\Ncot \geq k$ or $mHp = \Omega(\frac{N}{k^6})$.
\end{restatable}

The proof appears in \Cref{assec:cot}. 
For future work, it remains to consider the comparative powers of chain-of-thought models of depths greater than one.

\section{Conclusion and future work}

This work highlights parallelism as a central feature of transformers that sets them apart from other neural architectures.
The focus on the log-depth and sublinear-width regime and specific computational tasks allows us to accentuate the benefits of parallelism, even for tasks like $k$-hop that appear inherently serial at first glance.

There is some efficiency loss in the ``compilation'' of MPC protocols to transformers that we hope to understand better in future work.
Furthermore, although we have empirically demonstrated the learnability of transformers that exploit parallelism in crucial ways, a theoretical understanding of learning such solutions remains an open question.

\bibliography{bib}

\begin{thebibliography}{63}
\providecommand{\natexlab}[1]{#1}
\providecommand{\url}[1]{\texttt{#1}}
\expandafter\ifx\csname urlstyle\endcsname\relax
  \providecommand{\doi}[1]{doi: #1}\else
  \providecommand{\doi}{doi: \begingroup \urlstyle{rm}\Url}\fi

\bibitem[Agarwal et~al.(2014)Agarwal, Chapelle, Dud{\'\i}k, and
  Langford]{agarwal2014reliable}
Alekh Agarwal, Olivier Chapelle, Miroslav Dud{\'\i}k, and John Langford.
\newblock A reliable effective terascale linear learning system.
\newblock \emph{Journal of Machine Learning Research}, 15\penalty0
  (1):\penalty0 1111--1133, 2014.

\bibitem[Andoni et~al.(2014)Andoni, Nikolov, Onak, and
  Yaroslavtsev]{andoni2014parallel}
Alexandr Andoni, Aleksandar Nikolov, Krzysztof Onak, and Grigory Yaroslavtsev.
\newblock Parallel algorithms for geometric graph problems.
\newblock In \emph{Proceedings of the forty-sixth annual ACM symposium on
  Theory of computing}, pages 574--583, 2014.

\bibitem[Andoni et~al.(2018)Andoni, Song, Stein, Wang, and Zhong]{asswz18}
Alexandr Andoni, Zhao Song, Clifford Stein, Zhengyu Wang, and Peilin Zhong.
\newblock Parallel graph connectivity in log diameter rounds.
\newblock In \emph{2018 IEEE 59th Annual Symposium on Foundations of Computer
  Science (FOCS)}. IEEE, October 2018.
\newblock \doi{10.1109/focs.2018.00070}.
\newblock URL \url{http://dx.doi.org/10.1109/FOCS.2018.00070}.

\bibitem[Angluin et~al.(2023)Angluin, Chiang, and Yang]{acy23}
Dana Angluin, David Chiang, and Andy Yang.
\newblock Masked hard-attention transformers and boolean rasp recognize exactly
  the star-free languages, 2023.

\bibitem[Assadi and N(2021)]{an21}
Sepehr Assadi and Vishvajeet N.
\newblock Graph streaming lower bounds for parameter estimation and property
  testing via a streaming xor lemma.
\newblock In \emph{Proceedings of the 53rd Annual ACM SIGACT Symposium on
  Theory of Computing}, STOC ’21. ACM, June 2021.
\newblock \doi{10.1145/3406325.3451110}.
\newblock URL \url{http://dx.doi.org/10.1145/3406325.3451110}.

\bibitem[Beame et~al.(2017)Beame, Koutris, and Suciu]{beame2017communication}
Paul Beame, Paraschos Koutris, and Dan Suciu.
\newblock Communication steps for parallel query processing.
\newblock \emph{Journal of the ACM (JACM)}, 64\penalty0 (6):\penalty0 1--58,
  2017.

\bibitem[Behnezhad et~al.(2019)Behnezhad, Brandt, Derakhshan, Fischer,
  Hajiaghayi, Karp, and Uitto]{behnezhad2019massively}
Soheil Behnezhad, Sebastian Brandt, Mahsa Derakhshan, Manuela Fischer,
  MohammadTaghi Hajiaghayi, Richard~M Karp, and Jara Uitto.
\newblock Massively parallel computation of matching and mis in sparse graphs.
\newblock In \emph{Proceedings of the 2019 ACM Symposium on Principles of
  Distributed Computing}, pages 481--490, 2019.

\bibitem[Beltagy et~al.(2020)Beltagy, Peters, and Cohan]{bpc20}
Iz~Beltagy, Matthew~E. Peters, and Arman Cohan.
\newblock Longformer: The long-document transformer, 2020.

\bibitem[Bengio et~al.(1994)Bengio, Simard, and Frasconi]{279181}
Y.~Bengio, P.~Simard, and P.~Frasconi.
\newblock Learning long-term dependencies with gradient descent is difficult.
\newblock \emph{IEEE Transactions on Neural Networks}, 5\penalty0 (2):\penalty0
  157--166, 1994.
\newblock \doi{10.1109/72.279181}.

\bibitem[Bhattamishra et~al.(2020)Bhattamishra, Ahuja, and Goyal]{bag20}
Satwik Bhattamishra, Kabir Ahuja, and Navin Goyal.
\newblock On the ability and limitations of transformers to recognize formal
  languages.
\newblock In \emph{Proceedings of the 2020 Conference on Empirical Methods in
  Natural Language Processing}, 2020.

\bibitem[Bietti et~al.(2023)Bietti, Cabannes, Bouchacourt, Jegou, and
  Bottou]{bcbjg23}
Alberto Bietti, Vivien Cabannes, Diane Bouchacourt, Herve Jegou, and Leon
  Bottou.
\newblock Birth of a transformer: A memory viewpoint, 2023.

\bibitem[Charikar et~al.(2020)Charikar, Ma, and Tan]{cmt20}
Moses Charikar, Weiyun Ma, and Li-Yang Tan.
\newblock New lower bounds for massively parallel computation from query
  complexity, 2020.

\bibitem[Choromanski et~al.(2022)Choromanski, Likhosherstov, Dohan, Song, Gane,
  Sarlos, Hawkins, Davis, Mohiuddin, Kaiser, Belanger, Colwell, and
  Weller]{cld+22}
Krzysztof Choromanski, Valerii Likhosherstov, David Dohan, Xingyou Song,
  Andreea Gane, Tamas Sarlos, Peter Hawkins, Jared Davis, Afroz Mohiuddin,
  Lukasz Kaiser, David Belanger, Lucy Colwell, and Adrian Weller.
\newblock Rethinking attention with performers, 2022.

\bibitem[Chung et~al.(2014)Chung, Gulcehre, Cho, and
  Bengio]{chung2014empirical}
Junyoung Chung, Caglar Gulcehre, KyungHyun Cho, and Yoshua Bengio.
\newblock Empirical evaluation of gated recurrent neural networks on sequence
  modeling.
\newblock \emph{arXiv preprint arXiv:1412.3555}, 2014.

\bibitem[Clark et~al.(2019)Clark, Khandelwal, Levy, and Manning]{clark2019does}
Kevin Clark, Urvashi Khandelwal, Omer Levy, and Christopher~D Manning.
\newblock What does bert look at? an analysis of bert's attention.
\newblock \emph{arXiv preprint arXiv:1906.04341}, 2019.

\bibitem[Coy and Czumaj(2022)]{cc22}
Sam Coy and Artur Czumaj.
\newblock Deterministic massively parallel connectivity.
\newblock In \emph{Proceedings of the 54th Annual ACM SIGACT Symposium on
  Theory of Computing}, STOC 2022, page 162–175, New York, NY, USA, 2022.
  Association for Computing Machinery.
\newblock ISBN 9781450392648.
\newblock \doi{10.1145/3519935.3520055}.
\newblock URL \url{https://doi.org/10.1145/3519935.3520055}.

\bibitem[Daniely(2017)]{daniely17}
Amit Daniely.
\newblock Depth separation for neural networks.
\newblock In Satyen Kale and Ohad Shamir, editors, \emph{Proceedings of the
  2017 Conference on Learning Theory}, volume~65 of \emph{Proceedings of
  Machine Learning Research}, pages 690--696. PMLR, 07--10 Jul 2017.
\newblock URL \url{https://proceedings.mlr.press/v65/daniely17a.html}.

\bibitem[Dao et~al.(2022)Dao, Fu, Ermon, Rudra, and Ré]{flashattention}
Tri Dao, Daniel~Y. Fu, Stefano Ermon, Atri Rudra, and Christopher Ré.
\newblock Flashattention: Fast and memory-efficient exact attention with
  io-awareness.
\newblock In \emph{NeurIPS}, 2022.

\bibitem[Dean and Ghemawat(2004)]{dean2004mapreduce}
Jeffrey Dean and Sanjay Ghemawat.
\newblock Mapreduce: Simplified data processing on large clusters.
\newblock In \emph{OSDI}, pages 137--150, 2004.

\bibitem[Dettmers et~al.(2022)Dettmers, Lewis, Belkada, and
  Zettlemoyer]{dettmers2022llm}
Tim Dettmers, Mike Lewis, Younes Belkada, and Luke Zettlemoyer.
\newblock Llm.int8(): 8-bit matrix multiplication for transformers at scale.
\newblock In \emph{Advances in Neural Information Processing Systems},
  volume~35, 2022.

\bibitem[Duris et~al.(1984)Duris, Galil, and Schnitger]{duris1984lower}
Pavol Duris, Zvi Galil, and Georg Schnitger.
\newblock Lower bounds on communication complexity.
\newblock In \emph{Proceedings of the Sixteenth Annual ACM Symposium on Theory
  of Computing}, page 81–91, 1984.

\bibitem[Eldan and Shamir(2016)]{es16}
Ronen Eldan and Ohad Shamir.
\newblock The power of depth for feedforward neural networks.
\newblock In Vitaly Feldman, Alexander Rakhlin, and Ohad Shamir, editors,
  \emph{29th Annual Conference on Learning Theory}, volume~49 of
  \emph{Proceedings of Machine Learning Research}, pages 907--940, Columbia
  University, New York, New York, USA, 23--26 Jun 2016. PMLR.
\newblock URL \url{https://proceedings.mlr.press/v49/eldan16.html}.

\bibitem[Elhage et~al.(2021)Elhage, Nanda, Olsson, Henighan, Joseph, Mann,
  Askell, Bai, Chen, Conerly, DasSarma, Drain, Ganguli, Hatfield-Dodds,
  Hernandez, Jones, Kernion, Lovitt, Ndousse, Amodei, Brown, Clark, Kaplan,
  McCandlish, and Olah]{eno21}
Nelson Elhage, Neel Nanda, Catherine Olsson, Tom Henighan, Nicholas Joseph, Ben
  Mann, Amanda Askell, Yuntao Bai, Anna Chen, Tom Conerly, Nova DasSarma, Dawn
  Drain, Deep Ganguli, Zac Hatfield-Dodds, Danny Hernandez, Andy Jones, Jackson
  Kernion, Liane Lovitt, Kamal Ndousse, Dario Amodei, Tom Brown, Jack Clark,
  Jared Kaplan, Sam McCandlish, and Chris Olah.
\newblock A mathematical framework for transformer circuits.
\newblock \emph{Transformer Circuits Thread}, 2021.
\newblock https://transformer-circuits.pub/2021/framework/index.html.

\bibitem[Ghaffari et~al.(2019)Ghaffari, Kuhn, and Uitto]{gku19}
Mohsen Ghaffari, Fabian Kuhn, and Jara Uitto.
\newblock Conditional hardness results for massively parallel computation from
  distributed lower bounds.
\newblock In \emph{IEEE 60th Annual Symposium on Foundations of Computer
  Science}, pages 1650--1663, 11 2019.
\newblock \doi{10.1109/FOCS.2019.00097}.

\bibitem[Goodrich et~al.(2011)Goodrich, Sitchinava, and
  Zhang]{goodrich2011sorting}
Michael~T Goodrich, Nodari Sitchinava, and Qin Zhang.
\newblock Sorting, searching, and simulation in the mapreduce framework.
\newblock In \emph{International Symposium on Algorithms and Computation},
  pages 374--383. Springer, 2011.

\bibitem[Gu and Dao(2023)]{gd23}
Albert Gu and Tri Dao.
\newblock Mamba: Linear-time sequence modeling with selective state spaces,
  2023.

\bibitem[Guha and McGregor(2009)]{gm09}
Sudipto Guha and Andrew McGregor.
\newblock Stream order and order statistics: Quantile estimation in
  random-order streams.
\newblock \emph{SIAM Journal on Computing}, 38\penalty0 (5):\penalty0
  2044--2059, 2009.
\newblock \doi{10.1137/07069328X}.
\newblock URL \url{https://doi.org/10.1137/07069328X}.

\bibitem[Hahn(2020)]{hahn20}
Michael Hahn.
\newblock Theoretical limitations of self-attention in neural sequence models.
\newblock \emph{Trans. Assoc. Comput. Linguistics}, 8:\penalty0 156--171, 2020.
\newblock \doi{10.1162/tacl\_{a}{\_{0}{0}{3}}{0}6}.
\newblock URL \url{https://doi.org/10.1162/tacl_a_00306}.

\bibitem[Hao et~al.(2022)Hao, Angluin, and Frank]{haf22}
Yiding Hao, Dana Angluin, and Robert Frank.
\newblock Formal language recognition by hard attention transformers:
  Perspectives from circuit complexity.
\newblock \emph{Trans. Assoc. Comput. Linguistics}, 10:\penalty0 800--810,
  2022.
\newblock URL \url{https://transacl.org/ojs/index.php/tacl/article/view/3765}.

\bibitem[Im et~al.(2023)Im, Kumar, Lattanzi, Moseley, Vassilvitskii,
  et~al.]{im2023massively}
Sungjin Im, Ravi Kumar, Silvio Lattanzi, Benjamin Moseley, Sergei
  Vassilvitskii, et~al.
\newblock Massively parallel computation: Algorithms and applications.
\newblock \emph{Foundations and Trends{\textregistered} in Optimization},
  5\penalty0 (4):\penalty0 340--417, 2023.

\bibitem[Jumper et~al.(2021)Jumper, Evans, Pritzel, Green, Figurnov,
  Ronneberger, Tunyasuvunakool, Bates, {\v{Z}}{\'\i}dek, Potapenko,
  et~al.]{jumper2021highly}
John Jumper, Richard Evans, Alexander Pritzel, Tim Green, Michael Figurnov,
  Olaf Ronneberger, Kathryn Tunyasuvunakool, Russ Bates, Augustin
  {\v{Z}}{\'\i}dek, Anna Potapenko, et~al.
\newblock Highly accurate protein structure prediction with alphafold.
\newblock \emph{Nature}, 596\penalty0 (7873):\penalty0 583--589, 2021.

\bibitem[Kacham et~al.(2023)Kacham, Mirrokni, and Zhong]{kmz23}
Praneeth Kacham, Vahab Mirrokni, and Peilin Zhong.
\newblock Polysketchformer: Fast transformers via sketches for polynomial
  kernels, 2023.

\bibitem[Karloff et~al.(2010)Karloff, Suri, and Vassilvitskii]{ksv10}
Howard Karloff, Siddharth Suri, and Sergei Vassilvitskii.
\newblock A model of computation for mapreduce.
\newblock In \emph{Twenty-first Annual ACM-SIAM Symposium on Discrete
  Algorithms}, pages 938--948, 12 2010.
\newblock \doi{10.1137/1.9781611973075.76}.

\bibitem[Kim et~al.(2022)Kim, Nguyen, Min, Cho, Lee, Lee, and Hong]{jts22}
Jinwoo Kim, Tien~Dat Nguyen, Seonwoo Min, Sungjun Cho, Moontae Lee, Honglak
  Lee, and Seunghoon Hong.
\newblock Pure transformers are powerful graph learners, 2022.

\bibitem[Kingma and Ba(2014)]{adam}
Diederik~P. Kingma and Jimmy Ba.
\newblock Adam: A method for stochastic optimization, 2014.

\bibitem[Li and McClelland(2022)]{lm22}
Yuxuan Li and James~L. McClelland.
\newblock Systematic generalization and emergent structures in transformers
  trained on structured tasks, 2022.

\bibitem[Likhosherstov et~al.(2021)Likhosherstov, Choromanski, and
  Weller]{lcw21}
Valerii Likhosherstov, Krzysztof Choromanski, and Adrian Weller.
\newblock On the expressive power of self-attention matrices.
\newblock \emph{arXiv preprint arXiv:2106.03764}, 2021.

\bibitem[Liu et~al.(2022)Liu, Ash, Goel, Krishnamurthy, and Zhang]{lagkz22}
Bingbin Liu, Jordan~T. Ash, Surbhi Goel, Akshay Krishnamurthy, and Cyril Zhang.
\newblock Transformers learn shortcuts to automata, 2022.

\bibitem[Loukas(2019)]{loukas19}
Andreas Loukas.
\newblock What graph neural networks cannot learn: depth vs width.
\newblock \emph{arXiv preprint arXiv:1907.03199}, 2019.

\bibitem[Malach(2023)]{malach23}
Eran Malach.
\newblock Auto-regressive next-token predictors are universal learners, 2023.

\bibitem[Merrill and Sabharwal(2022)]{ms22-log-prec}
William Merrill and Ashish Sabharwal.
\newblock A logic for expressing log-precision transformers, 2022.

\bibitem[Merrill and Sabharwal(2023{\natexlab{a}})]{ms22-parallelism}
William Merrill and Ashish Sabharwal.
\newblock The parallelism tradeoff: Limitations of log-precision transformers.
\newblock \emph{Transactions of the Association for Computational Linguistics},
  11:\penalty0 531–545, 2023{\natexlab{a}}.
\newblock ISSN 2307-387X.
\newblock \doi{10.1162/tacl_a_00562}.
\newblock URL \url{http://dx.doi.org/10.1162/tacl_a_00562}.

\bibitem[Merrill and Sabharwal(2023{\natexlab{b}})]{ms23-cot}
William Merrill and Ashish Sabharwal.
\newblock The expressive power of transformers with chain of thought,
  2023{\natexlab{b}}.

\bibitem[Merrill et~al.(2022)Merrill, Sabharwal, and Smith]{mss22}
William Merrill, Ashish Sabharwal, and Noah~A. Smith.
\newblock Saturated transformers are constant-depth threshold circuits.
\newblock \emph{Transactions of the Association for Computational Linguistics},
  10:\penalty0 843–856, 2022.
\newblock ISSN 2307-387X.
\newblock \doi{10.1162/tacl_a_00493}.
\newblock URL \url{http://dx.doi.org/10.1162/tacl_a_00493}.

\bibitem[MPICH(2023)]{mpi-doc}
MPICH.
\newblock Mpi allreduce, 2023.
\newblock URL
  \url{https://www.mpich.org/static/docs/latest/www3/MPI_Allreduce.html}.

\bibitem[Nisan and Wigderson(1993)]{nw93}
Noam Nisan and Avi Wigderson.
\newblock Rounds in communication complexity revisited.
\newblock \emph{SIAM Journal on Computing}, 22\penalty0 (1):\penalty0 211--219,
  1993.
\newblock \doi{10.1137/0222016}.
\newblock URL \url{https://doi.org/10.1137/0222016}.

\bibitem[Oren et~al.(2024)Oren, Hassid, Adi, and Schwartz]{ohas24}
Matanel Oren, Michael Hassid, Yossi Adi, and Roy Schwartz.
\newblock Transformers are multi-state rnns, 2024.

\bibitem[Papadimitriou and Sipser(1982)]{papadimitriou1982communication}
Christos~H. Papadimitriou and Michael Sipser.
\newblock Communication complexity.
\newblock In \emph{Proceedings of the Fourteenth Annual ACM Symposium on Theory
  of Computing}, page 196–200, 1982.

\bibitem[P{\'e}rez et~al.(2021)P{\'e}rez, Barcel{\'o}, and
  Marinkovic]{perez2021attention}
Jorge P{\'e}rez, Pablo Barcel{\'o}, and Javier Marinkovic.
\newblock Attention is turing complete.
\newblock \emph{Journal of Machine Learning Research}, 22\penalty0
  (1):\penalty0 3463--3497, 2021.

\bibitem[Radford et~al.(2019)Radford, Wu, Child, Luan, Amodei, and
  Sutskever]{Radford2019LanguageMA}
Alec Radford, Jeffrey Wu, Rewon Child, David Luan, Dario Amodei, and Ilya
  Sutskever.
\newblock Language models are unsupervised multitask learners.
\newblock \emph{OpenAI blog}, 1\penalty0 (8):\penalty0 9, 2019.

\bibitem[Rogers et~al.(2021)Rogers, Kovaleva, and Rumshisky]{rogers2021primer}
Anna Rogers, Olga Kovaleva, and Anna Rumshisky.
\newblock A primer in bertology: What we know about how bert works.
\newblock \emph{Transactions of the Association for Computational Linguistics},
  8:\penalty0 842--866, 2021.

\bibitem[Roughgarden et~al.(2018)Roughgarden, Vassilvitskii, and Wang]{rvw18}
Tim Roughgarden, Sergei Vassilvitskii, and Joshua Wang.
\newblock Shuffles and circuits (on lower bounds for modern parallel
  computation).
\newblock \emph{Journal of the ACM}, 65:\penalty0 1--24, 11 2018.
\newblock \doi{10.1145/3232536}.

\bibitem[Sanford et~al.(2023)Sanford, Hsu, and Telgarsky]{sht23}
Clayton Sanford, Daniel Hsu, and Matus Telgarsky.
\newblock Representational strengths and limitations of transformers, 2023.

\bibitem[Strobl(2023)]{strobl23}
Lena Strobl.
\newblock Average-hard attention transformers are constant-depth uniform
  threshold circuits, 2023.

\bibitem[Strobl et~al.(2023)Strobl, Merrill, Weiss, Chiang, and
  Angluin]{smwca23}
Lena Strobl, William Merrill, Gail Weiss, David Chiang, and Dana Angluin.
\newblock Transformers as recognizers of formal languages: A survey on
  expressivity, 2023.

\bibitem[Telgarsky(2016)]{telgarsky16}
Matus Telgarsky.
\newblock Benefits of depth in neural networks.
\newblock In Vitaly Feldman, Alexander Rakhlin, and Ohad Shamir, editors,
  \emph{29th Annual Conference on Learning Theory}, volume~49 of
  \emph{Proceedings of Machine Learning Research}, pages 1517--1539, Columbia
  University, New York, New York, USA, 23--26 Jun 2016. PMLR.
\newblock URL \url{https://proceedings.mlr.press/v49/telgarsky16.html}.

\bibitem[Turkoglu et~al.(2021)Turkoglu, D'Aronco, Wegner, and
  Schindler]{turkoglu2021gating}
Mehmet~Ozgur Turkoglu, Stefano D'Aronco, Jan~Dirk Wegner, and Konrad Schindler.
\newblock Gating revisited: Deep multi-layer rnns that can be trained.
\newblock \emph{IEEE Transactions on Pattern Analysis and Machine
  Intelligence}, 44\penalty0 (8):\penalty0 4081--4092, 2021.

\bibitem[Vaswani et~al.(2017)Vaswani, Shazeer, Parmar, Uszkoreit, Jones, Gomez,
  Kaiser, and Polosukhin]{vsp+17}
Ashish Vaswani, Noam Shazeer, Niki Parmar, Jakob Uszkoreit, Llion Jones,
  Aidan~N. Gomez, Lukasz Kaiser, and Illia Polosukhin.
\newblock Attention is all you need.
\newblock In \emph{Advances in Neural Information Processing Systems 30}, 2017.

\bibitem[Wang et~al.(2022)Wang, Wang, Xu, Zhou, and Lu]{wang2022quantformer}
Ziwei Wang, Changyuan Wang, Xiuwei Xu, Jie Zhou, and Jiwen Lu.
\newblock Quantformer: Learning extremely low-precision vision transformers.
\newblock \emph{IEEE Transactions on Pattern Analysis and Machine
  Intelligence}, 2022.

\bibitem[Wei et~al.(2021)Wei, Chen, and Ma]{wcm21}
Colin Wei, Yining Chen, and Tengyu Ma.
\newblock Statistically meaningful approximation: a case study on approximating
  turing machines with transformers, 2021.

\bibitem[Yao et~al.(2021)Yao, Peng, Papadimitriou, and Narasimhan]{yppn21}
Shunyu Yao, Binghui Peng, Christos~H. Papadimitriou, and Karthik Narasimhan.
\newblock Self-attention networks can process bounded hierarchical languages.
\newblock In \emph{Proceedings of the 59th Annual Meeting of the Association
  for Computational Linguistics and the 11th International Joint Conference on
  Natural Language Processing}, 2021.

\bibitem[Yun et~al.(2020)Yun, Bhojanapalli, Rawat, Reddi, and Kumar]{ybrrk20}
Chulhee Yun, Srinadh Bhojanapalli, Ankit~Singh Rawat, Sashank Reddi, and Sanjiv
  Kumar.
\newblock Are transformers universal approximators of sequence-to-sequence
  functions?
\newblock In \emph{International Conference on Learning Representations}, 2020.

\bibitem[Zhang et~al.(2023)Zhang, Backurs, Bubeck, Eldan, Gunasekar, and
  Wagner]{zbbegw23}
Yi~Zhang, Arturs Backurs, Sébastien Bubeck, Ronen Eldan, Suriya Gunasekar, and
  Tal Wagner.
\newblock Unveiling transformers with lego: a synthetic reasoning task, 2023.

\end{thebibliography}
\bibliographystyle{plainnat}

\newpage
\appendix

\section{Supplemental Preliminaries}\label{asec:app-prelims} 

\subsection{Further details about transformers}\label{assec:transformers}

We discuss a few minor technicalities and modifications of the self-attention unit (\Cref{def:attn}) and transformer model (\Cref{def:tran}) defined in \Cref{ssec:transformers} that are necessary for readers looking for a comprehensive understanding of the proofs of our theoretical results.

\paragraph*{Fixed-bit precision arithmetic.}
As discussed in \Cref{ssec:transformers}, we assume that all numbers that appear in the intermediate products and outputs of self-attentions are representable with $p$-bit precision arithmetic, where $p = \Theta(\log N)$.
While the details of fixed-precision arithmetic will be uninteresting to most readers, it is necessary to explain precisely what we mean in order to ensure that proofs of results like \Cref{thm:transformer-simulation} are sound.
Throughout the paper, we allow $p$ to depend on of constants, such as $\gamma$, $\delta$, and $\epsilon$.

Concretely, we assume that all query, key, and value embeddings $Q(X), K(X), V(X)$ evaluated on all inputs contain scalar values $z \in \R$ that are polynomially bounded (i.e. $|z| \leq \exp(O(p)) = N^\zeta$ for sufficiently large constant exponent $\zeta > 0$) and are inverse-polynomially discretized (i.e. $z \cdot N^\zeta \in \Z$).
Depending on the desired exponent $\zeta$, some $p = \Theta(\log N)$ can be chosen to guarantee this property.
While we do not formally analyze the precision needed to approximate the particular embeddings employed by our proofs, we note that our recurring sinusoidal embeddings (e.g. \Cref{lemma:lookup}) can be discretized without losing their central properties and that discretizations of the restricted isometry embeddings of \Cref{prop:qsp} are analyzed by \citet{sht23}.

Rather than stipulating a particular bounded-precision implementation that computes the output of a self-attention unit must be implemented, we specify a rounding constraint that any computational implementation of a self-attention unit must satisfy.
Precisely, we require that any output round to the same inverse-polynomial discretization as the true mathematical attention.
\begin{definition}\label{def:valid}
For a self-attention unit $f \in \attn{m}N$, let $\hat{f}$ be an finite-precision implementation of that unit.
We say that $\hat{f}$ is a \emph{valid implementation} if \[\sup_{X \in \R^{N \times m}} \norm[\infty]{f(X) - \hat{f}(X)} = O\paren{ \frac{1}{2^{p}}}.\]
\end{definition}
This definition is only to establishing the fact that self-attention units with sufficient margins can precisely compute hardmax outputs in \Cref{lemma:hardmax} and to showing that MPC models can indeed compute the outputs precisely in \Cref{thm:transformer-simulation}.

\paragraph*{Hardmax attention.}

While we exclusively consider attention units with the softmax, our constructions periodically rely on the exact computation of averages of embeddings.
We define the \emph{hardmax} operator to allow the consideration of discrete averaging operations.   
For some $v \in \R^N$, let \[\hm(X)_i = \begin{cases}
\frac{1}{|\imax(v)|}, & \text{if} \ i \in \imax(v) \\
0 & \text{otherwise,}
\end{cases}\]
where $\imax(v) = \set{i \in [N]: v_i = \max_{i'} v_{i'}}$.

We show that bounded-precision softmax self-attention units that satisfy a margin property can be modified slightly to have identical outputs to an analogous hardmax unit.  
\begin{restatable}{lemma}{lemmahardmax}\label{lemma:hardmax}
Let $f \in \attn{m}N$ be a self-attention unit with precision $p = \Theta(\log N)$ and embedding functions $Q, K, V$ such that for some fixed $1 \geq \xi = N^{-O(1)}$ and every $X \in \R^{N \times m}$ and $i \in [N]$:
\[A(X)_{i, i'} \leq \max_{i''} A(X)_{i, i''} - \xi, \ \forall i' \not\in \imax(A(X)_i),\]
where $A(X) = Q(X) K(X)^\T$.
Then there exists a self-attention unit $f' \in \attn{m}N$ with a valid $p'$-bit implementation with $p' = O(p)$ satisfying \[f'(X) = \hm(A(X)) V(X).\]
\end{restatable}
The proof of \Cref{lemma:hardmax} is provided in \Cref{asec:low-level}.

\paragraph*{Start tokens.}
Our technical proofs are occasionally simplified by including a ``dummy token'' whose value is passed in self-attention layers as a default or null value.
For example, in the proof of \Cref{lemma:last-occurrence}, the dummy token handles the case where the reference token does not appear previously in the sequence.  
While we believe that this extra token is not necessary for our technical arguments, we include it for the sake of simplicity.

We model this dummy token as a \emph{start-of-sequence} token $X_0$.
Concretely, if we employ $X_0$ in a self-attention $f\in \attn{m}N$ which takes as input $X$, we instead treat $f$ as an attention unit in $\attn{m}{N+1}$ that operates on $(X_0, X_1, \dots, X_N)$.
We assume that $X_0$ is constant-valued, and therefore never both to pay attention to its outputs; it's only relevance is via its key and value embeddings $K_0(X_0), V_0(X_0) \in \R^{m}$.
If $X_0$ is unmentioned, we assume that it does not exist, or is set such that its key embedding inner products are all zero.

\paragraph*{Supplemental chain-of-thought tokens.}
We periodically (see \Cref{thm:mpc-simulation-general} and the proofs of Corollaries~\ref{cor:connectivity-hardness} and \ref{cor:k-hop-hardness}) consider transformers with supplemental blank ``chain-of-thought'' tokens appended to the end of the sequence.
Unlike the start token, these are only constant \emph{at initialization} and may be used deeper in the model to perform meaningful computations.

Let $\tran{m, L, H, \din, \dout}{N, M}$ denote transformers with $M - N$ extra blank elements appended to the input sequence. 
Concretely, we represent $T \in \tran{m, L, H, \din, \dout}{N, M}$ as some $T' \in \tran{m, L, H, \din, \dout}M$ and define the output $T(X)$ for $X \in\R^{N \times \din}$ by letting $Y \in \R^{M \times \din }$ for $Y_{1:N} = X$ and $Y_{N+1:M} = \vec0$, and letting $T(X) = T'(Y)$.

\section{Proofs from \Cref{ssec:mpc-simulation}}

\subsection{Proof of \Cref{lemma:routing-block}}\label{assec:lemma-routing-block}

\lemmaroutingblock*

The proof relies on a \emph{sparse propagation} sequential primitive, which complements the sparse averaging primitive of \citet{sht23}. 
For any $Q \leq d, N$, on input $X = (X_1, \dots, X_N) \in \R^{N \times d}$ with $X_i = (z_i, S_i) \in \R^{d-Q} \times [N]^Q$ and $b_i = \abs{\set{S_j \ni i: j \in [N]}} \leq Q$, we define 
\[\qsp_{Q, d}(X)_i = \begin{cases}
\frac1{b_i} \sum_{S_j \ni i} z_j & \text{if $b_i > 0$},\\
0 & \text{otherwise.}
\end{cases}\]
Closely following the argument of \citet{sht23}, we show in \Cref{prop:qsp} that there is a self-attention unit
with
embedding dimension $m = \max(d, O(q \log N))$
that computes $\qsp_{Q, d}$.
This construction is a key component of the single-layer transformer used in the proof of \Cref{lemma:routing-block}.

\begin{restatable}{proposition}{propqsp}\label{prop:qsp} 
For any $b \leq N$ and $d$, there exists a self-attention unit $\qsp_{Q, d} \in \attn{m, p}{N}$ for $m = d + O(Q \log N)$ and $p = O(\log N)$, which, given any input $X$ with $X_i = (z_i, S_i, \vec{0}) \in \R^{d} \times {[N] \choose \leq Q}  \times \set{0}^{m - Q - d}$ such that $b_i = \abs{\set{S_j \ni i: j \in [N]}} \leq Q$ for all $i$, has output $\qsp_{Q, d}(X)$ satisfying \[\qsp_{Q, d}(X)_{i} = \frac1{b_i} \sum_{S_j \ni i} z_j.\]
\end{restatable}

The proof of \Cref{prop:qsp} appears in \Cref{asec:low-level}.

\begin{proof}[Proof of \Cref{lemma:routing-block}]
We construct a single-layer single-headed transformer with query, key, and value embeddings $Q, K, V$ and output MLP $\psi$.
$Q, K, V$ can be decomposed as $Q = Q' \circ \phi, \ K = K' \circ \phi, \ V = V' \circ \phi,$
for some input MLP $\phi$ and embeddings $Q', K', V'$.
We fix $Q', K', V'$ to be the respective embeddings of the self-attention unit with embedding dimension $m$ from \Cref{prop:qsp} that computes $Y = \qsp_{s, m}(X)$ for $X_\src = (z_\src, S_\src)$ for every $\src \in [N]$ to be determined.
Hence, the proof entails designing element-wise encoders $\phi = (\phi_1, \dots, \phi_N)$ and decoders $\psi = (\psi_1, \dots, \psi_N)$ that compute $\inc$ from $\outg$, using $\qsp_{s, m}$ as an intermediate step.
A high-level overview of the proof construction is visualized in \Cref{fig:routing}.

\begin{figure}[t]
  \centering
  \includegraphics[width=0.9\textwidth]{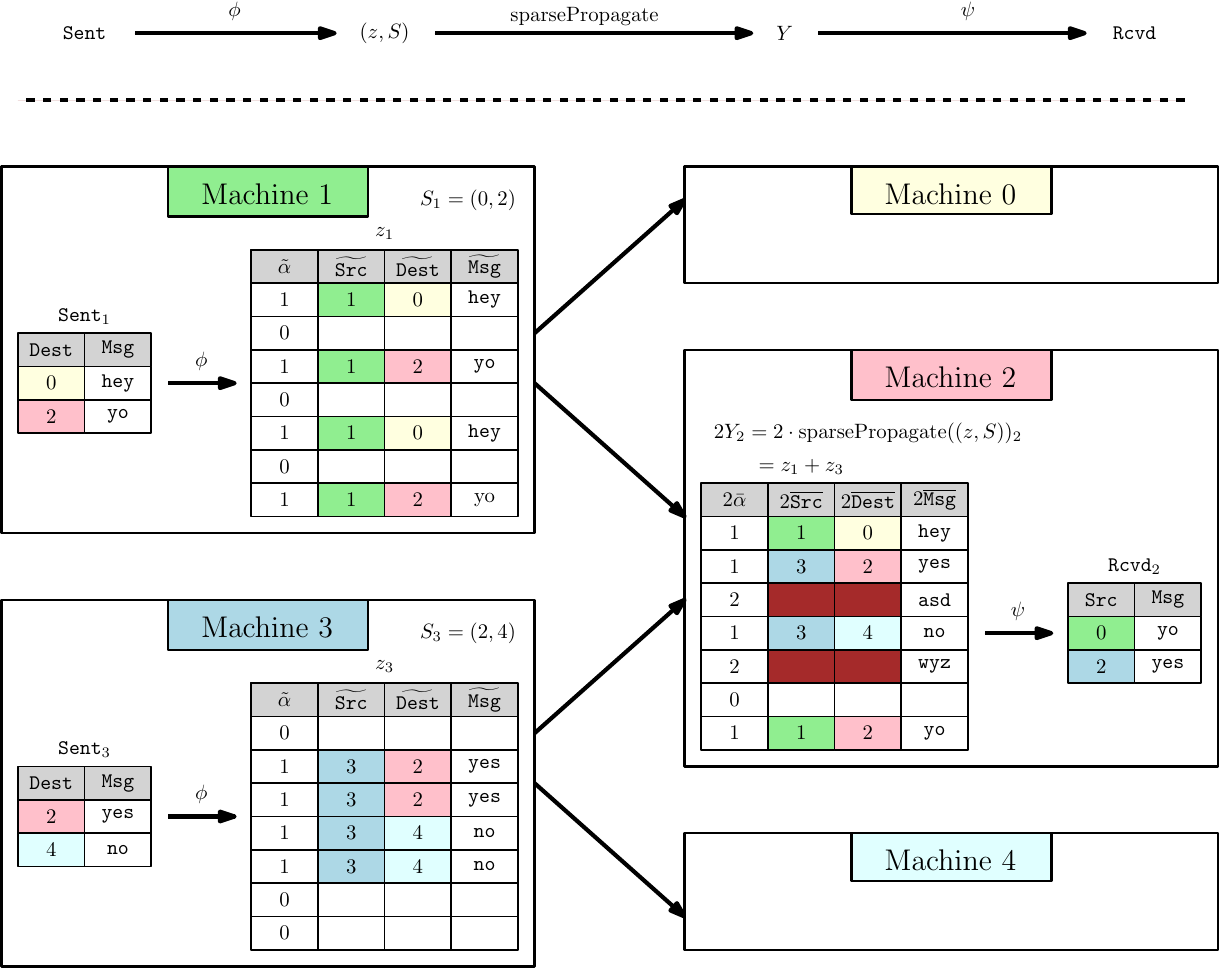}
  \caption{A visualization of the construction used to prove \Cref{lemma:routing-block} in three phases---the encoding of each input $\outg_\src$ as embedding $z_\src$ and subset $S_\src$ with $\phi$; the combination of those embeddings into $Y_\dst$ via the simulation of $\qsp_{s, m}((z, S))$; and the decoding of each $Y_\dst$ into output $\inc_\dst$ with $\psi$. 
 The figure provides an example of the encoding and decoding where machines 1 and 3 transmit messages to machine 2. ``Multiple hashing'' is used to compute $z_1$ and $z_3$ by encoding each message in multiple fixed-location ``packets'' in embedding space space.
 This redundancy ensures the possibility of machine 2 decoding $\inc_2$ from $Y_2$, due to each message occurring alone at least once in the encoding.  }
  \label{fig:routing}
\end{figure}

On input $\outg_\src$, we use the encodings $Q_\src,K_\src,V_\src$ to specify that all tokens $\dst$ with $\dst \in \outg_\src$ (or equivalently, all $\dst$ with $\src \in \inc_{\dst}$) should receive a copy of the encoding of $\outg_\src$.
That is, we set $S_\src := \set{\dst \in \outg_\src}$ for each $\src \in [N]$.
This ensures that $Y$ satisfies \[Y_\dst = \frac{1}{|\inc_\dst|} \sum_{\src \in \inc_\dst} z_{\src}.\]

While it's tempting to simply set each $z_\src \in \R^m$ equal to a $(\beta s)$-dimensional vectorization of $\outg_\src$, it is unclear how to extract $\inc_\dst$ from each $Y_\dst$, since each average performed by $\qsp_{s, m}$ will combine multiple vector embeddings in a shared space. 
In order to avoid these troubles, we employ a \emph{multiple hasing-based encoding} that treats messages as ``packets'' identified by a message, a source, a destination, and a ``validity token'' that can be used to determine whether a message is uncorrupted.
We include multiple copies of each packet in the encoding $z_\src$.
For notational ease, we represent each $z_\src \in \R^m$ as a collection of packets
\[z_\src = (\msgr_{\src, j}, \srcr_{\src, j}, \dstr_{\src, j}, \alpha_{\src, j})_{j \in [m']} \in (\pword^\beta \times [N] \times [N] \times \bit)^{m'},\]
where $m = m'(3 + \beta)$.

To sparsely and redundantly encode each $\outg_\src$ as $z_\src$, we encode outgoing messages as packets by utilizing the matrix $A$ guaranteed by the following fact (which we use with $n := N^2$, $b:= s^2$, and $m' := d = O(s^4 \log N)$).

\begin{restatable}{fact}{factsparse}\label{fact:sparse}
For any $n$, $b \leq n$, and $d \geq \ceil{12b^2 \ln n}$, there exists a binary matrix $A \in \bit^{n \times d}$ such that, for every subset $S \subseteq [n]$ with $|S| \leq b$, the columns of the sub-matrix $A_S \in \bit^{|S| \times d}$ contains all $S$-dimensional elementary vectors, i.e., $\set{e_1, \dots, e_{|S|}}$ is a subset of the columns of $A_S$.
\end{restatable}
The proof of \Cref{fact:sparse} is at the end of the section.
We use the following rule to determine which (if any) message to encode as a packet at each $\src \in [N]$ and $j \in [m']$.
We let $A_{(\src, \dst), j} = A_{N(\src -1) + \dst, j}$ for notational convenience.
\[z_{\src, j} = \begin{cases}
(\msg, \src, \dst, 1) & \text{if} \ (\msg, \dst) \in \outg_\src  \ \text{and} \ A_{(\src, \dst), j} = 1  \\ &\quad \text{and} \ A_{(\src, \dst'), j} = 0,  \ \forall \  \dst' \in \outg_\src \setminus \set{\dst}, \\
(\vec0, 0, 0, 0) & \text{otherwise.}
\end{cases}\]
In \Cref{fig:routing}, this encoding is visualized in the tables of ``Machine 1'' and ``Machine 3,'' where the entirety of each message is encoded in two fixed and distinct locations in the embeddings $z_1$ and $z_3$, alongside metadata about the source of message and the validity $\tilde\alpha$.
Each message is encoded as multiple identical packets in different embedding dimensions and a large fraction of embedding locations are left blank.
These features are critical for the proper evaluation of the decoding step $\psi$.

We analyze the $Y = \qsp_{\beta, m}(X)$ outputs, letting
\[
  Y_\dst = (Y_{\dst, 1}, \dots, Y_{\dst, m'}) , \quad Y_{\dst, j}  \in (\R^\beta \times \R \times \R \times \R)^{m'},
\]
with all numbers represented with $p$-bit fixed precision.
This analysis shows that there exists an element-wise decoder MLP $\psi$ satisfying $\psi_\dst(Y_\dst) = \inc_\dst$ for all $\dst \in [N]$.
For any $j \in [m']$, observe from the definition of $z_\src$ and $\qsp_{s, m}$ that
\begin{align*}
Y_{\dst, j} 
&=: \paren{\msgb_{\dst, j}, \srcb_{\dst, j}, \dstb_{\dst, j}, \alphab_{\dst, j}}   \\
&= \frac1{|\inc_\dst|} \sum_{\src \in \inc_\dst}\paren{ \msgr_{\src, j},\srcr_{\src, j}, \dstr_{\src, j}, \alpha_{\src, j}}.
\end{align*}
Before formally analyzing this construction, we motivate its utility with \Cref{fig:routing}.
The encoding $2Y_2$ of Machine 2 contains four ``clean'' rows $j$ with $2\alphab_{2, j} = 1$, two ``corrupted'' rows with $2\alphab_{2, j} = 2$, and one ``blank'' row with $2\alphab_{2, j} = 0$. 
\begin{itemize}
\item The \textbf{blank row} contains no information about any incoming messages, since neither Machine 1 nor Machine 3 encoded messages as packets in these locations. The fact that $2\alphab_{2, j} = 0$ certifies the blankness of this row, and hence, the decoder $\psi$ can ignore it.
\item The \textbf{corrupted rows} correspond to locations where both Machine 1 and Machine 3 saved messages as packets. As a result, the corresponding embedding $Y_{2, j} = \frac12 (z_{1, j} + z_{3, j})$ is an average of two non-zero embeddings and is hence ``corrupted.'' Because $2\alphab_{2, j} = 2$, the decoder $\psi$ detects the corruption and ignores it when computing $\inc_2$.
\item The \textbf{clean rows} are locations where exactly one of Machine 1 and Machine 3 encoded a message. Hence, these messages can be cleanly understood by the decoder $\psi$, which simply validates the ``cleanliness'' of the row with $2\alphab_{2, j} = 1$, determines whether Machine 2 is indeed the target recipient of the respective message, and saves all such messages in the decoding $\inc_2$.
\end{itemize}

We prove the validity of this intuition by ensuring that the encoding scheme successfully encodes each incoming message in a clean row and that the category of each row (blank, corrupted, or clean) can be detected by the decoder $\psi$.
We observe the following sequence of facts about every $Y_\dst$. 
Let \[\rele_\dst := \set{(\msg, \src', \dst'): \src' \in \inc_\dst, \ (\msg, \dst') \in \outg_{\src'}}\]
denote the set of \emph{all} messages sent by sources of messages sent to $\dst$.
\begin{enumerate}
\item Consider any outgoing message $(\msg, \src', \dst') \in \rele_\dst$. 
By the property of $A$ guaranteed by \Cref{fact:sparse}, there exists some $j$ such that $A_{(\src', \dst'), j} = 1$ and $A_{(\src'', \dst''), j} = 0$ for every $(\src'', \dst'') \in \rele_\dst \setminus \set{(\src', \dst')}.$
As a result of the definition of the encoding $z$ and the averaged representation of $Y_\dst$:
\begin{equation}\label{eq:srcdst}
Y_{\dst, j} = \frac1{|\inc_\dst|} \paren{\msg, \src', \dst', 1}.
\end{equation}
\item Conversely, if $\alphab_{\dst, j} = 1 / |\inc_\dst|$, then there exists a unique $(\msg, \src', \dst') \in \rele_\dst$ such that \eqref{eq:srcdst} is satisfied.
\item If at least one message is received, then the minimal nonzero value of $\alphab_{\dst}$ is $1 / |\inc_\dst|$.
\end{enumerate}

We design $\psi_\dst$ to uniquely identify $\inc_\dst$ from $Y_\dst$ as follows.
If at least one message is received, then $1 / |\inc_\dst|$ can be identified by finding the smallest nonzero value of $\alphab_\dst$.
The decoder $\psi$ inspects every $Y_{\dst, j}$ satisfying $\alphab_{\dst, j} = 1 / |\inc_\dst|$, which therefore satisfies \[|\inc_\dst| \cdot (\msgb_{\dst, j}, \srcb_{\dst, j}, \dstb_{\dst, j}) \in \rele_\dst.\]
Thus, if $|\inc_\dst|\cdot \dstb_{\dst, j} = \dst$, then $|\inc_\dst| \cdot (\msgb_{\dst, j}, \srcb_{\dst, j}) \in \inc_\dst$, and $\psi$ encodes it as such.

\end{proof}

\factsparse*
\begin{proof}
Let $\col(A)$ denote the set of columns of $A$.
We use the probabilistic method and consider $A$ with iid entries $A_{i,j} \sim \bern(\frac1{b+1})$.
We bound the probability of failure:
\begin{align*}
\pr{\exists S \in {[n] \choose \leq b} \  \text{s.t.} \ \set{e_1, \dots, e_{|S|} }\not\subset \col(A_S)}
&\leq b \cdot n^b \pr{e_i \not\in \col(A_S)} \\
&\leq n^{b+1} \paren{1 - \frac1{b+1} \cdot \paren{1 - \frac1{b+1}}^{b}}^d \\
&\leq n^{b+1} \paren{1 - \frac1{e(b+1)}}^d  \\
&\leq n^{b+1} \cdot \exp\paren{-\frac{d}{e(b+1)}} \\
&< \exp\paren{(b+1)\ln n - \frac{d}{3(b+1)}} \leq 1.
\end{align*}
Therefore, there exists a matrix $A$ with the claimed property.
\end{proof}

\subsection{Proof of \Cref{thm:mpc-simulation}}\label{assec:thm-mpc-simulation}

We give a generalization of \Cref{thm:mpc-simulation} that simulates a broader family of MPC protocol, including those with more than $n$ machines (i.e. $\gamma \geq \delta$).
We accommodate this generalization by simulating MPC protocols with the generalized transformer family $\tran{m, L, H}{N, M}$ detailed in \Cref{asec:app-prelims} with supplemental blank ``chain-of-thought'' tokens.

\begin{theorem}[Generalization of \Cref{thm:mpc-simulation}]\label{thm:mpc-simulation-general}
For constant $\gamma, \delta > 0$ and any potentially randomized $R$-round $(\gamma, \delta)$-MPC protocol $\pi$ on $\nin$ input words and $\nout \leq \nin$ output words, there exists a transformer $T \in \tran{m, L, H}{N, M}$ with $N = \nin, M = \max(\nin , O(\nin^{1 + \gamma - \delta})), m = O(\nin^{4\delta}\log \nin), L = R + 1, H = O(\log\log \nin)$ such that \[T(\inp)_{:\nout} = \pi(\inp).\]
\end{theorem}

\Cref{thm:mpc-simulation} is an immediate consequence of \Cref{thm:mpc-simulation-general} by noting that $M = N$ for sufficiently large $\nin$ when $\gamma < \delta$.
Its central construction is summarized in \Cref{fig:mpc-simulation}.

\begin{figure}[t]
  \centering
  \includegraphics[width=0.5\textwidth]{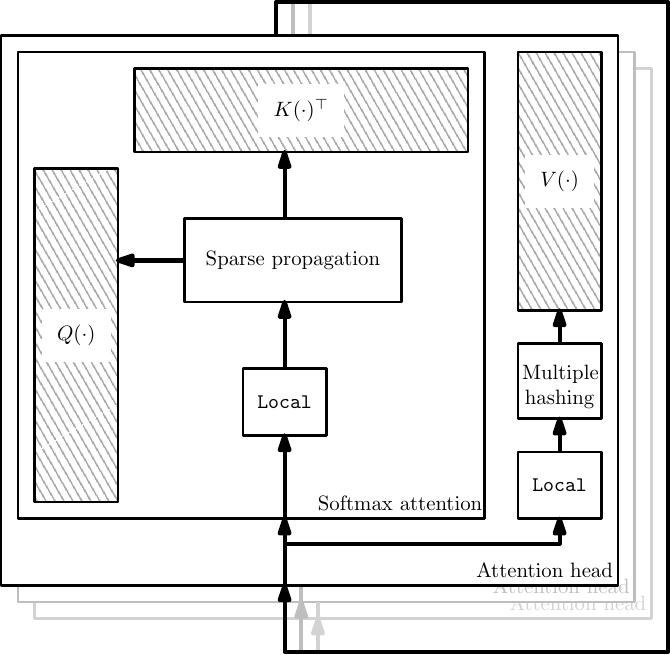}
  \caption{To simulate MPC, the local computation within each machine is pushed inside $Q(\cdot), K(\cdot), V(\cdot)$,
    and then the pairwise attention matrix performs message routing.  To ensure proper routing and also that the outputs of $Q(\cdot), K(\cdot), V(\cdot)$ are all tall-and-skinny matrices,
    the construction carefully utilizes both multiple hashing and sparse propagation.}
  \label{fig:mpc-simulation}
\end{figure}

\begin{proof}
Consider any MPC protocol $\pi$ with $q = O(\nin^{1 + \gamma - \delta})$ machines and $s = O(\nin^{\delta})$ local memory that, following the notation of \Cref{def:mpc}, maps $\inp \in \pword^{\nin}$ to $\outp \in \pword^{\nout}$ with intermediates $\machin{1},\dots \machin{R}$ and $\machout{1}, \dots, \machout{R}$ and deterministic functions $(\local_{r, i})_{r \in [R], i \in [q]}$ with \[\machout{r}_i = \local_{r, i}(\machin{r}_i).\]

To simulate the protocol, we let every machine $i \in [q]$ correspond to a particular position in the transformer's context. 
A transformer that simulates $\pi$ can then be constructed that consolidates $\inp$ onto $\ceil{\nin / s}$ machines to match $\machin{1}$; computes $\machin{r+1}$ from $\machin{r}$ for each $r = 1, \dots, R-1$; and computes and properly distributes $\outp$ from $\machin{r}$.
These three elements of the construction exist due to the following lemmas, which are proved later.

\begin{restatable}{lemma}{lemmacondense}\label{lemma:condense}
For any MPC protocol $\pi$ with local memory $s$ and $q$ machines with $\nin$-word inputs, there exists a transformer $\init \in \tran{s, 1, 1, \din,\dout}{\nin, \max(\nin, q)}$ with $\din=1$ and $\dout =s$, which, given $\inp \in \pword^n$, has output satisfying $\init(\inp) = \machin{1}$.
\end{restatable}

\begin{restatable}{lemma}{lemmaoneround}\label{lemma:one-round}
For any $R$-round MPC protocol $\pi$ with local memory $s$ and $q$ machines and any $r \in [R-1]$, there exists a transformer $\round{r} \in \tran{m, 1, H, \din, \dout}{q}$ with $H = O(\log\log q)$, $m = O(s^4 \log q)$, and $\din = \dout = s$ which, given any valid input $X = \machin{r} \in \pword^{q \times m}$ under the MPC protocol in vectorized form, has output satisfying $\round{r}(X) = \machin{r+1}$.
\end{restatable}

\begin{restatable}{lemma}{lemmaspread}\label{lemma:spread}
For any $R$-round MPC protocol $\pi$ with local memory $s$ and $q$ machines with $\nout$-word output, there exists a transformer $\final \in \tran{s, 1, 1, \din, \dout}{q, \max(\nout, q)}$ for $\din=s$ and $\dout=1$, which, given input $X = \machin{R}$, has output $\final(X)$ with $\final(X)_{i, 1} = \outp_i \in \pword$.
\end{restatable}

The proof immediate from the three lemmas.
We construct the final transformer $T$ by stacking the single-layer constructions as a single transformer with embedding dimension $m$:
\[ T = \final \circ \round{R-1} \circ \dots \circ \round{1} \circ \init.\]

The proofs of Lemmas~\ref{lemma:condense} and \ref{lemma:spread} rely on simple constructions with fixed attention matrices and appear in \Cref{asec:low-level}.
The proof of \Cref{lemma:one-round} relies on \Cref{lemma:routing-block} and is proved in the following section.
\end{proof}

\paragraph*{Proof of $\round{r}$ construction.}
To prove the existence single-layer transformer that simulates $\round{r}$, we separate the computational task into two steps: (i) obtaining $\machout{r}$ from $\machin{r}$ and (ii) obtaining $\machin{r+1}$ from $\machout{r}$.
Because the former requires no communication between machines, we can encode that conversion in the input MLP to the transformer.

The nontrivial part of the reduction is thus the latter step, which we obtain by utilizing multiple single-headed attention units $\routeb_{\beta, s}$ of \Cref{lemma:routing-block} to route messages of different sizes to their recipients.
The difficulty in this task is the mismatch in functionality between the two computational models: while the MPC model ensures that each recipient automatically receives its intended messages, transformers must implement this functionality manually, while ensuring that multiple messages do not overwrite one another.

The following lemma implements that routing functionality for all messages, using different attention heads depending on the size of the message. We prove \Cref{lemma:one-round} at the end of the section as a simple modification of \Cref{lemma:routing}.

\begin{lemma}\label{lemma:routing}
For any $R$-round MPC protocol $\pi$ with local memory $s$ and $q$ machines and any $r \in [R-1]$, there exists a transformer $\route{r} \in \tran{m, 1, H}{q}$ with $H = O(\log\log q)$ and $m = O(s^4 \log q)$, which, given any valid input $X = \machout{r} \in \pword^{q \times m}$ under the MPC protocol in vectorized form, has output satisfying $\route{r}(X) = \machin{r+1}$.
\end{lemma}

Because at most $s$ messages can be shared and received by each machine, and each message is of size at most $s$, we can prove an single-headed alternative to \Cref{lemma:routing} with a somewhat suboptimal dependence on embedding dimension.
By applying by \Cref{lemma:routing-block} with message size $\beta = s$, bounded number of messages $s$, and context length $N = q$, there exists a transformer $\routeb_{s, s}$ with $H = 1$ and $m = O(s^5 \log q)$ that computes $\machin{r+1}$ from $\machout{r+1}$ by regarding each outgoing message as belonging to $\pword^s$ by adding padding dimensions as needed.

We improve the embedding dimension to $m = O(s^4 \log q)$ by running in parallel $O(\log\log N)$ transformers guaranteed by \Cref{lemma:routing-block} that encode differently sized messages.
The number of heads $H$ increases at a doubly-logarithmic rate because of a doubling trick employed on the size of message encodings used by constituent part.

\begin{proof}
We describe an implementation of $\route{r}$ by considering any fixed input $\machout{r} \in \pword^{q \times m}$.
For each $i \in [q]$ and some integer sequence $1 = \beta_0 < \beta_1 < \dots < \beta_H = s+1$, we partition $\machout{r}_i$ into $H$ disjoint subsets as follows. 
For any $h \in [H]$, let
\begin{align*}
\outg_i^h &:= \set{(\msg, \dst) \in \machout{r}_i: \dim(\msg) \in [\beta_{h-1},\beta_h]}, \\
\inc_i^h &:= \set{(\msg, \src) \in \machin{r+1}_i: \dim(\msg) \in [\beta_{h-1},\beta_h]},
\end{align*}
and note that $\machout{r}_i = \dot{\bigcup}_{h =1}^H \outg_i^h$ and $\machin{r+1}_i =  \dot{\bigcup}_{h =1}^H \inc_i^h$.

For each $h \in [H]$, note that $\dim(\msg) \leq \beta_h$, and $\abs{\outg_i^h} = \abs{\inc_i^h} \leq s / \beta_{h-1}$.
As a result, \Cref{lemma:routing-block} guarantees the existence of a single-headed transformer $\route{r}_h$ such that $\route{r}_h(\outg^h) = \inc^h)$ with embedding dimension $m_h \leq C s^4 \beta_h \log(q) / \beta_{h-1}^4$ for some sufficiently large universal constant $C$.

We defined $\route{r}$ as the computation of $\route{r}_1, \dots, \route{r}_H$ as $H$ parallel heads of self-attention with disjoint embeddings concatenated into in $m$-dimensional embedding space with $m = \sum_{h=1}^H m_h$.
We conclude by letting \[\beta_h = \begin{cases} 1 &\text{if} \ h = 0, \\ \min(2 \beta_{h-1}^3, q+1) & \text{if} \ h \in [H],\end{cases}\] noting that $\beta_H = q+1$ for $H = O(\log\log q)$, and bounding $m$:
\begin{align*}
m
 &\leq  \sum_{h=1}^H \frac{ Cs^4 \log (q) \beta_h}{\beta_{h-1}^4} 
 \leq 2Cs^4 \log (q) \cdot \sum_{h=1}^H \frac{1}{\beta_{h-1}}  \\
 &\leq 2Cs^4 \log (q) \cdot \sum_{h=1}^H \frac{1}{2^{h-1}} 
 = O(s^4 \log  q). \qedhere
\end{align*}
\end{proof}

\begin{proof}[Proof of \Cref{lemma:one-round}]

To simulate a round of MPC protocol $\pi$ by mapping $\machin{r}$ and $\rho_{r}$ to $\machin{r+1}$, the single-layer transformer $\round{r}$ first computes $\machout{r}$ element-wise and then properly routes messages in $\machout{r}$ to their proper destination.
We can define $\round{r} = \route{r} \circ \local_r$ for $\route{r}$ in \Cref{lemma:routing} and $\local_{r, i}(\machin{r}_i, \rho_{r, i}) = \machout{r}_i$.
This can be immediately constructed as a single-layer transformer by prepending the embeddings $Q, K, V$ of the construction of $\route{r}$ with $\local_r$, using $Q \circ \local_r$, $K \circ \local_r$, $V \circ \local_r$ as the embeddings of $\round{r}$.
\end{proof}

\subsection{Additional graph problems solvable by log-depth transformers}\label{assec:gralgos}

Theorem 8.1 and Corollary 8.2 of \citet{cc22} give efficient MPC protocols for other graph problems besides connectivity, and therefore, as corollaries of \Cref{thm:mpc-simulation}, we also obtain log-depth transformers for these problems.

\begin{corollary}[Spanning forest construction]\label{cor:st}
For any constant $\epsilon \in (0, 1)$ and any $D \leq N$, there exists a transformer in $\tran{m, L, H}{N}$ with $m = O(N^\epsilon)$, $H = O(\log\log N)$, and $L = O(\log D)$ that computes a rooted spanning forest of any input graph $G = (V, E)$ with $|V|, |E| = O(N)$ where each connected component has diameter at most $D$.
\end{corollary}

\begin{corollary}[Minimum spanning forest construction]\label{cor:mst}
For any constant $\epsilon \in (0, 1)$ and any $D_{MSF} \leq N$, there exists a transformer in $\tran{m, L, H}{N}$ with $m = O(N^\epsilon)$, $H = O(\log\log N)$, and $L = O(\log D_{MSF})$ that identifies the connected components of any input graph $G = (V, E)$ with $|V|, |E| = O(N)$ and $\poly(N)$-bounded integer weights whose minimum spanning forest has diameter at most $D_{MSF}$.
\end{corollary}

\section{Proofs from \Cref{ssec:tr-simulation}}

\subsection{Proof of \Cref{thm:transformer-simulation}}\label{assec:thm-transformer-simulation}

As in \Cref{assec:thm-mpc-simulation}, we give and prove a generalized version of \Cref{thm:transformer-simulation} that broadens the family of considered transformers to include masked models and those that contain extra blank chain-of-thought tokens, using notation from \Cref{asec:app-prelims}.

\begin{theorem}[Generalization of \Cref{thm:transformer-simulation}]\label{thm:transformer-simulation-general}
For any transformer $T \in \trantape$ (or $\mtrantape$) with $m H = O( N^\delta)$ for $\delta \in (0,1)$ and $M = \Theta(N^{1 + \alpha})$ for $\alpha \geq 0$ and for any $\delta' \in (\delta, 1)$, there exists an $O(\frac{L(1 + \alpha)}{\delta' - \delta})$-round $(1 +2\alpha + \delta', \delta')$-MPC protocol with $q = O(M^{2})$ machines with $s = O(N^{\delta'})$ local memory that outputs the same sequence as $T(X)$ for all $X \in \R^N$.
\end{theorem}

\Cref{thm:transformer-simulation} is an immediate consequence by setting $M := N$ and $\alpha := 0$.

\newcommand\children{\mathtt{Children}}
\newcommand\desc{\mathtt{Descendants}}
\newcommand\parent{\mathtt{Parent}}

\begin{figure}[t]
  \centering
  \includegraphics[width=0.7\textwidth]{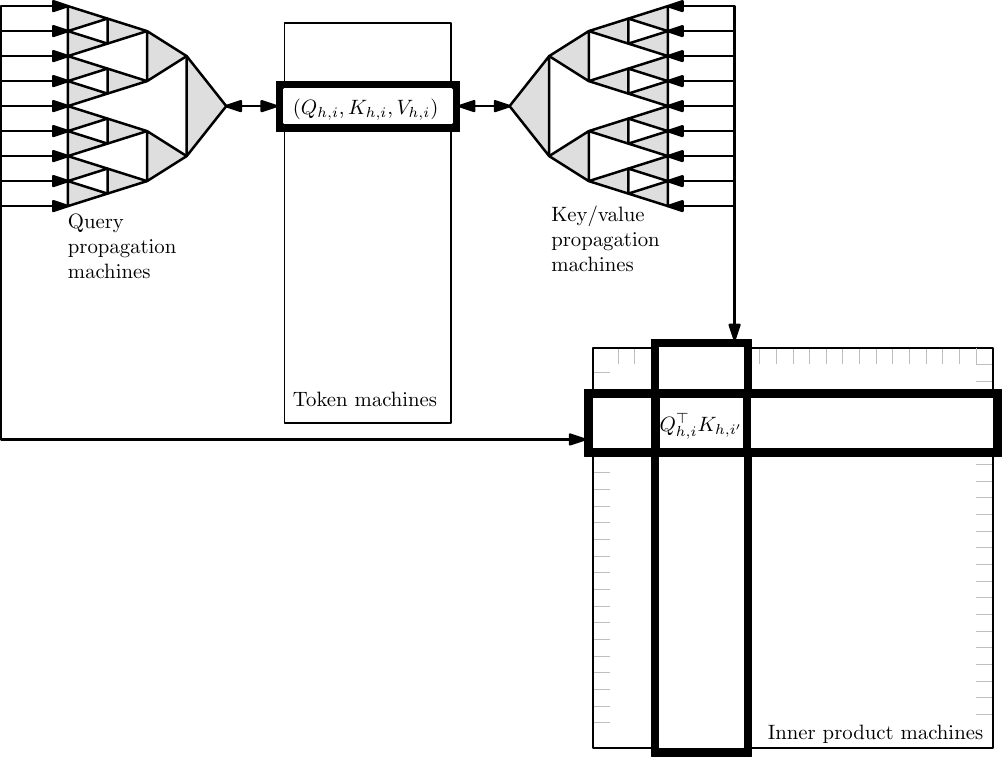}
  \caption{This construction employs $M^2$ \emph{inner product machines} to compute the entries of the softmax matrix,
    and $M$ \emph{token machines} to compute all values of $Q(\cdot), K(\cdot), V(\cdot)$.
     What is most complex about the construction
    are the additional machines and message routing needed to propagate these values efficiently between the inner product machines and the token machines,
    in particular carefully aggregating the output of the attention mechanism and computing its normalization.
  To this end, the protocol uses additional machines, organized into a tree with branching factor $b =O(N^{\delta'-\delta})$ and depth $D = O(\frac{1+\alpha}{\delta'-\delta})$.}
  \label{fig:transformer-simulation}
\end{figure}

\begin{proof}
It suffices to show that an $O(\frac{1 + \alpha}{\delta' - \delta})$-round MPC protocol $\pi$ that simulates a single-layer transformer $T \in \tran{m, m, m, 1, H}{M}$ with $m$-dimensional input and output embeddings since a depth-$L$ transformer can be constructed by applying $L$ such protocols sequentially.
Moreover, we can ignore the difference between the input context length $N$ and the context length with padding $M$ by assuming that the input contains $M$ tokens.

Concretely, we consider $H$ heads with embeddings $(Q_h, K_h, V_h)_{h \in [H]}$, element-wise output MLP $\psi = (\psi_1, \dots, \psi_M)$, and any fixed masks $\Lambda_1, \dots, \Lambda_H \in \{-\infty, 0\}^{M \times M}$.
We show that there exists some $\pi$ such that for any $\inp = X \in \R^{M \times m}$,
\[\pi(X) = \psi\paren{X + \sum_{h=1}^H \sm(Q_h(X) K_h(X)^\T + \Lambda_h) V_h(X)},\]
where numbers in $X$ and all intermediate products of the transformer computation can be represented with $p = O(\log M)$ bit precision.

Our MPC protocol $\pi$, which will use $q = O(M^2)$ machines and $s = \Theta(N^{\delta'})$ words of local memory per machine, assigns each of the $q$ machines to one of four possible roles: token machine, inner product machine, query propagation machine, and key/value propagation machine.
We describe these machines below.
For the sake of readability, we identify machines with easily interpretable descriptions and use the bijection $\id$ to map each of those to a token in $[q]$ that is used for routing messages.
Our protocol has two important parameters:
$b = \floor{s / (4mH)} = O(N^{\delta' - \delta})$ is the \emph{branching factor} of the protocol, and $D = \ceil{\log_b(M)} = O(\frac{1 + \alpha}{\delta' - \delta})$ is the \emph{depth} of the protocol.

At a high level (see \Cref{fig:transformer-simulation} for a corresponding diagram), the protocol involves computing all intermediate products of the of a transformer unit by performing MLP computations in $N$ \emph{token machines}, computing inner products in $N^2$ \emph{inner product machines}, and using $O(N^2)$ other \emph{propagation machines} arranged in trees to share information between the two in $O(D)$ rounds.
The protocol draws inspiration from Appendix~C.6.1 of \citet{sht23}, which uses a similar construction to simulate transformers with \textsc{Congest} protocols on fixed graphs.
It is also similar to the MPC implementation of the MPI AllReduce functionality~\citep{mpi-doc} described by \citet{agarwal2014reliable}.

\begin{itemize}
\item Machine $i \in [M]$ is a \emph{token machine} that performs all element-wise computation on the $i$th token embedding, including the computation of $(Q_{h, i}(X_i), K_{h,i}(X_i), V_{h,i}(X_i))_{h \in [H]}$ and the final $MLP$ output $\psi_i$. Let $\id(i) = i$.
\item Machine $(i, i') \in [M]^2$ is an \emph{inner product machine} designed to compute the inner products $(Q_{h, i}(X_i)^\T K_{h, i'}(X_{i'}))_{h \in [H]}$.
\item Machine $(\q, i, d, k)$ for token $i \in [M]$, depth $d \in [D-1]$ and position $k \in [b^d]$ is a \emph{query propagation machine}.
This machine is responsible for handling communication of query tokens $(Q_{h,i}(X_i))_{h \in [H]}$ and of all partially-computed attention outputs for the $i$th token between token machine $i$ and inner product machines $(i, i')$ for \[i' \in \desc_{d, k} := \set{b^{D-d} (k-1), \dots, b^{D-d}k} \cap [M].\]

Concretely, if $\ell = 1$, then the machine communicates with token machine $i$ and query propagation machines $(\q, i, d+1, k')$ for
\[k' \in \children_{k} := \set{b(k-1)+1, \dots, bk}.\]
If $\ell = D-1$, then it communicates with inner product machines $(i, i')$ for $i' \in \children_{k} \cap [M]$ 
and query propagation machine $(\q, i, d-1, \floor{k / b})$.
Otherwise, it communicates with query propagation machines $(\q, i, d-1, \parent_k)$, for $\parent_k := \floor{k / b}$, and 
$(\q, i, d+1, k')$ for $k' \in \children_k$.
\item Machine $(\kv, i, d, k)$ is a \emph{key/value propagation machine}.
This machine is analogous to a query propagation machine, except that it is responsible for the communication of key and value tokens $(Q_{h,i}(X_i),V_{h,i}(X_i))_{h \in [H]}$ between token machine $i$ and inner product machines $(i, i')$ for $i' \in \desc_{d, k}$.
\end{itemize}

Since the total number of machines is $q = M + M^2 + M \sum_{d=1}^{D-1} b^d = O(M^2)$, we conclude that the global memory of the protocol is $qs = O(N^{2 + 2\alpha + \delta'})$, which means the protocol is $(1 + 2\alpha + \delta', \delta')$-MPC.
We simulate the transformer using a four stage protocol using $2D + 3 = O(\frac{1 + \alpha}{\delta' - \delta})$ rounds of MPC computation.

\paragraph*{Stage 1: Token dispersion.}
Because the input to an MPC protocol $\inp = X$ is divided equally among machines $1, \dots, \ceil{MmH / s}$, the first round of MPC computation routes each input token $X_i$ to its respective token machine.
This is completed by setting $(i, X_i) \in \machout{1}_{i'}$ if $(i, X_i) \in \machin{1}_{i'}$.
Thus, $\machin{2}_i = \set{(\src, X_i)}$ for all token machines $i \in [M]$.

\paragraph*{Stage 2: Embedding propagation.}
In rounds $2, \dots, D+1$, $\pi$ computes the respective key, query, and value embeddings in each token machine and propagate them to respective inner product machines using the query and key/value propagation machines.
Concretely:
\begin{itemize}
\item In round 2, each token machine $i$ (whose memory contains $X_i$) computes $m$-dimensional embeddings embeddings $Q_i := (Q_{h, i}(X_i))_{h \in [H]}, K_i := (K_{h, i}(X_i))_{h \in [H]}, V_i := (V_{h, i}(X_i))_{h \in [H]}$. 
It transmits each embedding to the respective depth-1 query and key/value propagation machine nodes, while also preserving knowledge of its own $X_i$. (In all further rounds, we assume that $((i, X_i)) \in \machout{r}_i$ to ensure that token machine $i$ can compute the skip-level connection at the end.) 
That is,
\begin{align*}
\machout{2}_i &= \set{(i, X_i)} \\ &\quad \cup \set{(\id(\q, i, 1, k'),Q_i): k' \in \children_1} \\ &\quad \cup \set{(\id(\kv, i, 1, k'),(K_i, V_i)): k' \in \children_1}.
\end{align*}
Note that the total amount of messages sent is $b \cdot mH + 2b \cdot mH  + m \leq s$ and that the only machines receiving messages are size $m$-messages by token machines and size $\leq 4mH$ messages by query and key/value propagation machines.

\item In rounds $r \in \set{3, \dots, D}$, each query and key/value propagation machine of depth $d = r-2$ passes embeddings onto their successors.
That is, 
\begin{align*}
\machout{r}_{\id(\q, i, d, k)} &= \set{(\id(\q, i, d+1, k'), Q_i): k' \in \children_{k}}, \\
\machout{r}_{\id(\kv, i, d, k)} &= \set{(\id(\kv, i, d+1, k'), (K_i, V_i)):k' \in \children_{k}}.
\end{align*}
\item In round $D+1$, the depth-$(D-1)$ query and key/value propagation machines pass their embeddings onto their respective inner product machines.
That is,
\begin{align*}
\machout{D+1}_{\id(\q, i, D-1, k)} &= \set{(\id(i, k'), Q_i): k' \in \children_{k} \cap [M]}, \\
\machout{D+1}_{\id(\kv, i, D-1, k)} &= \set{(\id(k', i), (K_i, V_i)): k' \in k' \in \children_{k} \cap [M]}.
\end{align*}
\end{itemize}

\paragraph*{Stage 3: Softmax computation.}
In rounds $D+2, \dots, 2D+2$, computes each inner product and iteratively builds up each attention output by accumulating partial softmax computations.
For each query propagation machine $(\q, i, d, k)$ and $h \in [H]$, we let $S_{i, d, k, h}$ and $Z_{i, d, k, h}$ denote its partial normalization and softmax computations respectively. That is,
\begin{align*}
Z_{i, d, k, h} &=
\sum_{i' \in \desc_{d, k}}
\exp(Q_{h, i}(X_i)^\T K_{h, i'}(X_{i'})) \indicator{\Lambda_{i, i'} = 0} \\
&= 
\begin{cases}
\sum_{k' \in \children_{k}}
 Z_{i, d + 1, k', h} & \text{if} \ d \leq D-1,  \\
\exp(Q_{h, i}(X_i)^\T K_{h, k}(X_{k})) \indicator{\Lambda_{i, k} = 0} & \text{if} \ d = D.
\end{cases}\\
S_{i, d, k, h} &=
\frac{1}{Z_{i,d,k,h}} \sum_{i' \in \desc_{d, k}}
 \exp(Q_{h, i}(X_i)^\T K_{h, i'}(X_{i'})) V_{h, i'}(X_{i'}) \indicator{\Lambda_{i, i'} = 0} \\
&= 
\begin{cases}
\sum_{k' \in \children_{k}}
 \frac{Z_{i,d+1,k',h}}{Z_{i,d,k,h}} \cdot S_{i, d + 1, k', h} & \text{if} \ d \leq D-1,  \\
 V_{h, k}(X_{k}) \indicator{\Lambda_{i, k} = 0} & \text{if} \ d = D;
\end{cases}
\end{align*}
Note that $S_{i, 0, 1, h} = (\sm(Q_{h}(X) K_h(X)^\T + \Lambda_h) V_h(X))_i$ and let $S_{i, d, k} = (S_{i, d, k, h})_{h \in [H]} \in \R^{H \times m}$ and $Z_{i, d, k} = (Z_{i, d, k, h})_{h \in [H]} \in \R^{H}$
\begin{itemize}
\item In round $D + 2$, each inner product machine computes its respective inner products and passes its partial softmax computations to its parent query propagation machine.
As a result of round $D + 1$, each inner product machine $(i, i')$ recently received the embeddings necessary to compute the relevant inner product:
\[\machin{d+2}_{\id(i, i')} = \set{(\id(\q, i, D-1, \parent_i), Q_i), (\id(\kv, i', D-1, \parent_{i'}), (K_{i'}, V_{i'}))}.\]
It propagates the respective partial computations $S_{i, D, i'}$ and $Z_{i, D, i'}$ as follows:
\[\machout{D+2}_{\id(i, i')} = \set{(\id(\q, i, D-1, \parent_i), (S_{i, D, i'}, Z_{i, D, i'}))}. \]
Note that each depth-$(D-1)$ query propagation machine receives messages of size at most $b \cdot (m+1)H \leq s$.
\item
In rounds $r \in \set{D+3, \dots, 2D}$, partial softmax computations are received by query propagation machines of depth $d = 2D+1-r$, added together, and passed along to their parent machines.
That is, given
\[\machin{r}_{\id(\q, i, d, k)} = \set{(\id(\q, i, d+1, k'), (S_{i, d+1, k'}, Z_{i, d+1, k'})): k' \in \children_k},\]
each respective machine computes $S_{i, d, k}$ and $Z_{i, d, k}$ recursively and propagates
\[\machout{r}_{\id(\q, i, d, k)} = \set{(\id(\q, i, d-1, \parent_k), (S_{i, d, k}, Z_{i, d, k})}.\]
\item
In round $2D+1$, the top-most query propagation tokens pass their partial sums to the token machines:
\[\machout{2D+1}_{\id(\q, i, 1, k)} = \set{(i, (S_{i, 1, k}, Z_{i, 1, k}))}.\]
\item 
In round $2D + 2$, the token machines compute their respective output of the transformer, $T(X)_i$.
Given input
\[\machin{2D+2}_{i} = \set{(k', (S_{i, 1, k'}, Z_{i, 1, k'})): k' \in \children_1} \cup \set{(i, X_i)},\]
the token machine $i$ computes $S_{i, 0, 1}$ and $H_{i, 0, 1}$ and then
\[T(X)_i = \psi_i\paren{X_i + \sum_{h=1}^H \sm(Q_h(X) K_h(X)^\T + \Lambda_h)_i^\T V_h(X)} = \psi_i\paren{X_i + \sum_{h=1}^H S_{i, 0, 1, h}}. \]
This quantity is used as an intermediate product for the final phase of computation.
\end{itemize}

\paragraph*{Stage 4: Token compression.}
We invert Stage 1 by properly compressing the MPC output in the final round $2D+3$.
That is, we let $\machout{2D+2}_i = \set{(\floor{imH / s}+1, T(X)_i)}$ for each token machine $i \in [M]$, which ensures that the outputs are condensed in the proper order in machines $1, \dots, \ceil{MmH / s}$.

\paragraph*{Precision analysis.}

In order for the proof to be fully sound, care must be taken to ensure that the computation of each self-attention output $S_{i, 0, 1, h}$ is handled with proper numeric precision, as discussed in \Cref{asec:app-prelims}.
We show that each $S_{i, 0, 1, h}$ is a \emph{valid implementation} of its corresponding self-attention unit, per \Cref{def:valid}.

To do so, we let $\hat{S}_{i,d,k,h}$ and $\hat{Z}_{i, d, k, h}$ denote the $p$-bit representations of ${S}_{i,d,k,h}$ and ${Z}_{i, d, k, h}$, where scalars of $\hat{S}_{i,d,k,h}$ and $\log(\hat{Z}_{i, d, k, h})$ are represented as discretized rational numbers $z$ satisfying $|z| \leq \frac12 2^{p/2}$ and $ z \cdot 2^{p/2} \in \Z$.
For some sufficiently small $p' = \Theta(p)$, we assume that all embeddings $Q_h(X), K_h(X), V_h(X)$ have scalars $z$ satisfying $|z| \leq \frac12 2^{p'/2}$ and $ z \cdot 2^{p'/2} \in \Z$.
We prove that for each $h \in [H]$,
\[\norm[\infty]{ S_{i, 0, 1, h} - \hat{S}_{i,d,k,h}} = O\paren{\frac1{2^{p'}}}.\]

Boundedness of intermediate representations is not an issue because \[\log(Z_{i,d,k,h}) \leq O(\log(N) + \max_{i, i'} |Q(X)_i^\T K(X)_{i'}|) = \exp(O(p')),\] and \[\norm[\infty]{S_{i, d, k, h}} \leq \norm[\infty]{V(X)} \leq 2^{p'/2}.\]
It remains to show that that all intermediate representations are sufficiently close to their exact counterparts.
We prove the following via an inductive argument for $d=D,D-1, \dots, 0$:
\begin{align}
\abs{ \log(Z_{i, d,k,h}) - \log(\hat{Z}_{i,d,k,h})} &\leq \frac{(2b)^{D-d}}{2^{p/2}}\label{eq:zerr}, \\
\norm[\infty]{ S_{i, d,k,h} - \hat{S}_{i,d,k,h}} &\leq \frac{2^{p'/2}(8b)^{D-d}}{2^{p/2}}. \label{eq:serr} 
\end{align}
If \eqref{eq:serr} holds for $d = 0$, then the claim holds for sufficiently large $p = \Theta(p')$.

For the base case $D$, we verify \eqref{eq:serr} by  \[\norm[\infty]{ S_{i, D,k,h} - \hat{S}_{i,D,k,h}} =\norm[\infty]{ V_{h,k}(X_k) \indicator{\Lambda_{i,k}=0} - \hat{S}_{i,D,k,h}} \leq \frac1{2^{p/2}},  \]
due to the ability to access $V_{h,k}(X_k)$ and round it directly.
We verify \eqref{eq:zerr} due to the immediate access to and boundedness of $Q_{h, i}(X_i)^\T K_{h, k}(X_k)$:
\begin{align*}
\abs{\log (Z_{i, d, k, h})} \leq
\abs{Q_{h, i}(X_i)^\T K_{h, k}(X_k)} 
&\leq \norm[2]{Q_{h, i}(X_i)}\norm[2]{ K_{h, k}(X_k)}
\leq N \cdot 2^{p'/2}.
\end{align*}

We prove the inductive step for $d - 1$, assuming that the inductive hypothesis holds for $d$.
We first address $\hat{Z}_{i, d-1, k, h}$ by employing the Lipschitzness of the log-sum-exp function.
\begin{align*}
\abs{ \log(Z_{i, d-1,k,h}) - \log(\hat{Z}_{i,d-1,k,h})} 
&\leq \frac1{2^{p/2} }+ \abs{\log\paren{\sum_{k'} \exp(\log(Z_{i, d,k', h}))} - \log\paren{\sum_{k'} \exp(\log(\hat{Z}_{i, d,k', h}))}}\\ 
&\leq \frac1{2^{p/2}} + \sum_{k'} \abs{\log(Z_{i, d, k', h}) - \log(\hat{Z}_{i, d, k', h})} \\
&\leq \frac1{2^{p/2}} + b \cdot \frac{(2b)^{D-d}}{2^{p/2}} 
\leq  \frac{(2b)^{D-d+1}}{2^{p/2}}.
\end{align*}

To obtain \eqref{eq:serr} for $d-1$, we first note that for sufficiently large $p$:
\begin{align*}
\abs{1 - \frac{\hat{Z}_{i, d, k', h} Z_{i, d-1, k', h}}{Z_{i, d, k, h} \hat{Z}_{i, d-1, k', h}}} 
&= \abs{1 - \exp\paren{\log\paren{\frac{\hat{Z}_{i, d, k', h}}{{Z}_{i, d, k', h}}} + \log\paren{\frac{{Z}_{i, d-1, k, h}}{\hat{Z}_{i, d-1, k, h}}}}} \\
&\leq 1+ 2\paren{\abs{\log{\frac{\hat{Z}_{i, d, k', h}}{{Z}_{i, d, k', h}}}} + \abs{\log{\frac{{Z}_{i, d-1, k, h}}{\hat{Z}_{i, d-1, k, h}}}}} \\
&\leq \frac{4 \cdot (2b)^{D-d + 1}}{2^{p/2}}.
\end{align*}
We conclude by using the fact that each $S_{i, d-1, k, h}$ is a convex combination of other  $S_{i, d, k, h}$.
\begin{align*}
\norm[\infty]{ S_{i, d-1,k,h} - \hat{S}_{i,d-1,k,h}}
&\leq \frac1{2^{p/2}} + \sum_{k'} \norm[\infty]{\frac{Z_{i, d, k', h}}{Z_{i, d-1, k', h}} S_{i, d, k', h} - \frac{\hat{Z}_{i, d, k', h}}{\hat{Z}_{i, d-1, k', h}} \hat{S}_{i, d, k', h}} \\
&\leq \frac1{2^{p/2}}  + \sum_{k'} \frac{Z_{i, d, k', h}}{Z_{i, d-1, k', h}} \norm[\infty]{S_{i, d, k', h} - \frac{\hat{Z}_{i, d, k', h} Z_{i, d-1, k', h}}{Z_{i, d, k, h} \hat{Z}_{i, d-1, k', h}} \hat{S}_{i, d, k', h}} \\
&\leq \frac1{2^{p/2}} + \sum_{k'} \frac{Z_{i, d, k', h}}{Z_{i, d-1, k', h}} \norm[\infty]{S_{i, d, k', h} - \hat{S}_{i, d, k', h}} \\&\quad+ \sum_{k'} \frac{Z_{i, d, k', h}}{Z_{i, d-1, k', h}} \norm[\infty]{\hat{S}_{i, d, k', h}} \abs{1 - \frac{\hat{Z}_{i, d, k', h} Z_{i, d-1, k', h}}{Z_{i, d, k, h} \hat{Z}_{i, d-1, k', h}} } \\
&\leq \frac1{2^{p/2}}  +\frac{2^{p'/2} (8b)^{D-d}}{2^{p/2}} + 2^{p'/2} \sum_{k'}  \frac{Z_{i, d, k', h}}{Z_{i, d-1, k', h}}  \abs{1 - \frac{\hat{Z}_{i, d, k', h} Z_{i, d-1, k', h}}{Z_{i, d, k, h} \hat{Z}_{i, d-1, k', h}} } \\
&\leq 2 \cdot \frac{2^{p'/2} (8b)^{D-d}}{2^{p/2}} + 2^{p' / 2} \cdot \frac{4 \cdot (2b)^{D-d + 1}}{2^{p/2}} 
\leq \frac{2^{p'/2} (8b)^{D-d+1}}{2^{p/2}}.
\end{align*}

Owing to the fact that $D$ and $p'$ are constants and $b = N^{O(1)}$, a sufficiently large choice of $p$ guarantees that the implementation is valid.
\end{proof}

\subsection{Proof of \Cref{cor:connectivity-hardness}}\label{assec:connectivity-hardness}

\corconnectivityhardness*

We prove \Cref{cor:connectivity-hardness} by combining \Cref{thm:transformer-simulation-general} and \Cref{conj:cycle}.

\begin{proof}
Fix any $D \leq N$ with $D \geq N^{\xi}$ for some $\xi \in (0, 1]$.
Let $C_1$ denote a cycle graph on $D$ vertices, and let $C_2$ denote the union of two cycle graphs each with $D/2$ vertices.

Suppose there is a transformer $T \in \tran{m, L, H}N$ with $mH = O(D^{1- \epsilon})$ that determines the connectivity of graphs with at most $N$ edges and connected components with diameter at most $D$.
We will show that it can be used to design an $\Theta(L)$-round MPC protocol $\pi$ that distinguishes graphs $C_1$ and $C_2$ with $n = D$ edges.

Let $\pi'$ be an MPC protocol that exactly computes the output of $T$ using taking $R = O(L)$ rounds with local memory $s = O(D^{1 - \epsilon / 2})$ and $q = O(N^2)$ machines, which is guaranteed to exist by \Cref{thm:transformer-simulation-general}.

Let $n := 2\floor{\frac{D}4}$ and $k := \floor{\frac{N}n}$.
We design $\pi$ with the same local memory and machine count to determine the identity of input graph $G = (V, E) \in \set{C_1, C_2}$ provided as an arbitrary sequence of $n$ edges.
Let $u \in V$ be an arbitrary vertex in $G$.

Using a constant number of MPC rounds, $\pi$ converts $G$ into a graph $G' = (V', E')$ with $|E'| = k n + k \leq N$ and diameter $n + 2 \leq D$ such that $G'$ is connected if and only if $G = C_1$.
We do so by letting $G'$ be composed of $k$ copies $G^1, \dots, G^k$ of $G$ on separate vertices, along with $k$ extra edges connecting the vertex corresponding to $u$ in each $G^j$ (say $u^j \in G^j$) to $u^1 \in G_1$.
This ensures that the connectivity correspondence and edge count diameter bounds are met.
Since $G'$ can be produced by simply copying edges from $G$ and adding an additional edge each time an edge containing $u$ is copied, $\pi$ can produce $G'$ in $O(1)$ rounds.

Then, $\pi$ simulates $\pi'$ on $G'$ and returns its output. 
Since $G'$ is connected if and only if $G = C_1$, this protocol suffices to distinguish $C_1$ and $C_2$.
Because the protocol uses $s = O(n^{1 - \epsilon / 2})$ local memory and $q = O(n^{2 / \xi})$ machines, \Cref{conj:cycle} implies that $\pi$ (and hence $T$) only exists if $L = \Omega(\log n) = \Omega(\log N)$.
\end{proof}

\section{Proofs from \Cref{ssec:khop-theory}}

\subsection{Proof of \Cref{thm:k-hop-construction}}\label{assec:k-hop-construction}
\thmkhopconstruction*

\begin{proof}
We design a masked transformer that implements $\khop$ in two phases.
The first two layers compute $\find^1_X(i)$ for each $i \in [N]$ using a similar approach to the induction heads construction of \cite{bcbjg23}.
The subsequent layers employ a doubling trick to compute each $\find^{2^{\ell-2}}_X(i)$ after $\ell$ layers.

To do so we employ two technical lemmas (which are proved in \Cref{assec:low-level-k-hop}) that describe the implementation of  masked self-attention units that copy .
\begin{restatable}{lemma}{lemmalookup}\label{lemma:lookup}
  For some $m \geq d+2$, $\tau: [N] \times \R^m \to [N]$, and $\rho: \R^m \to \R^d$, there exists an attention head $\lookup_{\tau, \rho} \in \mattn{m}{N}$ with precision $p = O(\log N)$ and $m \geq d + 2$ satisfying $\lookup_{\tau,\rho}(X)_{i, :d} = \rho(X_{\tau(i,X_i)})$.
\end{restatable}

\begin{restatable}{lemma}{lemmalastoccurrence}\label{lemma:last-occurrence}
  For any finite alphabet $\Sigma$, $m \geq d + 2$, $\mu_1, \mu_2: \R^m \to \Sigma$, and $\rho: \R^m \to \R^d$, there exists an attention head $\lo_{\mu, \rho} \in \mattn{m}{N}$ with precision $p = O(\log(N \abs{\Sigma}))$ such that, 
\[\lo(X)_{i, :d} = \begin{cases}
\rho(\vec0) & \text{if} \ \forall\  i' < i: \mu_1(X_{i'}) \neq \mu_2(X_i), \\
\rho(X_{i'}) &\text{if} \  i' = \max\set{i' < i: \mu_1(X_{i'}) = \mu_2(X_i)}.
\end{cases}\]
\end{restatable}

The first layer obtains the previous token $X_{i-1}$ from each $X_i$. 
This is accomplished via the self-attention head $\lookup_{\tau, \rho}$ with $\tau(i,X_i) = i-1$ and $\rho(X_i) = X_i$.

The second layer retrieves $(\find^1_X(i), X_{\find^1_X(i)})$ for each $i \in [N]$ by finding the most recent token whose \emph{preceding} token is $X_i$.
It does so by employing the $\lo_{\mu_1, \mu_2, \rho}$ primitive on the intermediate state $X^1_i = (X_{i}, X_{i-1})$ with $\mu_1(X^1_i) = X_{i-1}$, $\mu_2(X^1_i) = X_i$, and $\rho(X^1_i) = (i, X_{i})$.
\begin{itemize}
  \item If $\find^1_X(i) > 0$, then $\lo_{\mu_1, \mu_2, \rho}(X^1_i) = (\find^1_X(i), X_{\find^1_X(i)})$.
  \item Otherwise, it obtains $\vec0$ and performs no further passing, returning $\perp$ after all $L$ layers.
\end{itemize}
If $k = 1$, the transformer returns $T(X)_i = X_{\find^1_X(i)} = \khop(X)_i$.

Otherwise, let $k := \sum_{j = 0}^{\floor{\log_2 k}} k_j 2^j$ for some $k_j \in \bit$, and let $k_{:\ell} = \sum_{j = 0}^{\ell} k_j 2^j$.
Construct a transformer inductively to ensure that the $i$th output of the $\ell$th layer $X^\ell_i \in \R^m$ for $\ell \geq 2$ contains an encoding of \[\paren{X_i, \find_X^{2^{\ell-2}}(i), X_{ \find_X^{2^{\ell-2}}(i)}, \find_X^{k_{:\ell-2}}(i), X_{ \find_X^{k_{:\ell-2}}(i)}}.\]
Note that the base case holds for $\ell = 2$, since $\find_X^{k_{:0}}(0) = \find_X^{1}(0)$ if $k_0 = 0$ and is $i$ otherwise.

For each $\ell =1, \dots, \floor{\log_2 k} + 1$, we assume that the inductive hypothesis holds up to layer $\ell$ and prove that it also holds for layer $\ell + 1$.
To do so, we use a $\lookup_{\tau, \rho}$ self-attention head with $\tau(i,X_i^\ell) = \find_X^{2^{\ell-2}}(i)$ and \[\rho(X_i^\ell) = (\find_X^{2^{\ell-2}}(i), X_{ \find_X^{2^{\ell-2}}(i)}, \find_X^{k_{:\ell-2}}(i), X_{ \find_X^{k_{:\ell-2}}(i)}),\]
which ensures that $X^{\ell+1}_i$ can encode 
\begin{align*}
\find_X^{2^{\ell-1}}(i) &= \find_X^{2^{\ell-2}}(\find_X^{2^{\ell-2}}(i)) \\
X_{\find_X^{2^{\ell-1}}(i)} &= X_{\find_X^{2^{\ell-2}}(\find_X^{2^{\ell-2}}(i))} \\
\find_X^{k_{:\ell-1}}(i) &= \begin{cases}
\find_X^{k_{:\ell-2}}(\find_X^{2^{\ell-2}}(i)) & \text{if} \ k_{\ell-1} = 1 \\
\find_X^{k_{:\ell-2}}(i) & \text{if} \ k_{\ell-1} = 0
\end{cases}\\
X_{\find_X^{k_{:\ell-1}}(i)} &= \begin{cases}
X_{\find_X^{k_{:\ell-2}}(\find_X^{2^{\ell-2}}(i))} & \text{if} \ k_{\ell-1} = 1 \\
X_{\find_X^{k_{:\ell-2}}(i)} & \text{if} \ k_{\ell-1} = 0.
\end{cases}
\end{align*}

As a result, the output of layer $L = \floor{\log_2 k} + 2$ contains an encoding of \[X_{\find_X^{k_{:L-2}}(i)} = X_{\find_X^{k}(i)} = \khop(X)_i\]
for each $i \in [N]$.
This is returned as the output of $T(X)$.

\end{proof}

\subsection{Proof of \Cref{cor:k-hop-hardness}}\label{assec:k-hop-hardness}
\corkhophardness*

\begin{proof}
The proof is analogous to that of \Cref{cor:connectivity-hardness}.
Let $C_1$ be a cycle on $k$ vertices, and $C_2$ be the union of two cycles each on $k/2$ vertices.
So both $C_1$ and $C_2$ have $k$ edges.
We show that the existence of $T \in \tran{m, L, H}N$ with $mH = O(k^{1-\epsilon})$ such that $T(X) = \khop(X)$ can be used to design an $\Theta(L)$-round MPC protocol $\pi$ to solve the task.

As a result of \Cref{thm:transformer-simulation-general}, there exists an MPC protocol $\pi'$ that exactly computes $T$ with $R = \Theta(L)$ rounds with local memory $s = O(D^{1 - \epsilon/2})$ and $q = O(N^2)$ machines.
On input $G  = (V, E) \in \set{C_1, C_2}$, we design a constant-round protocol that computes an sequence $X \in \Sigma^N$ such that $\khop(X)_N$ exactly determines the identity of $G$.

Since the $k$ edges are passed to $\pi$ in an unknown ordering with unknown labelings, we let $V = [k]$ and denote the edges as $e_1 = \set{u_1, v_1}, \dots, e_k = \set{u_k, v_k}$.
We define an operator $\nxt$ over the domain $\{ (u,v), (v,u) : \set{u,v} \in E \}$ as follows:
for $\set{u, v} \in E$, let $\nxt(u, v) := (v', u)$ where $v' \in V$ is the unique vertex $v' \neq v$ such that $\set{u, v'} \in E$.
Notice that $\nxt$ is well-defined because all vertices in a cycle have degree $2$.
If $G = C_2$, then $\nxt^{k/2}(u_i, v_i) = (u_i, v_i)$ for any $i \in [k]$.

To set up our encoding of $G$ as a sequence $X$, we first construct a gadget for each edge $e_i$ that will be used to compute a single $\nxt(u_i, v_i)$.
Under the alphabet $\Sigma = [k] \cup \set{\dagger, \star, \_}$, we define the nine-token sequence \[\be_i =\ \star\ u_i\ \dagger\ v_i\ u_i\ \dagger\ v_i\ \star\ \_.\]
This gadget ensures that two hops will swap the values of $u_i$ and $v_i$. That is
\begin{align*}
\find^2_{\be_i \circ u_i}(10) &= \find^1_{\be_i \circ u_i}(6) = 4, & X_{\find^2_{\be_i \circ u_i}(10)} = v_i ,\\
\find^2_{\be_i \circ v_i}(10) &= \find^1_{\be_i \circ v_i}(8) = 2, & X_{\find^2_{\be_i \circ v_i}(10)} = u_i .
\end{align*}
Likewise, concatenating sequences corresponding to overlapping edges facilitates multiple hops. 
For example, if $e_1 = (1, 2), e_2 = (3, 4), e_3 = (2, 3)$, then
\begin{align*}
\find^2_{\be_1 \circ \be_2 \circ \be_3 \circ 2}(28) &= 22, & X_{\find^2_{\be_1 \circ \be_2 \circ \be_3 \circ 2}(28)} = 3 ,\\
\find^4_{\be_1 \circ \be_2 \circ \be_3 \circ 2}(28) &= 13, & X_{\find^4_{\be_1 \circ \be_2 \circ \be_3 \circ 2}(28)} = 4 ,\\
\find^4_{\be_1 \circ \be_2 \circ \be_3 \circ 3}(28) &= 2, & X_{\find^4_{\be_1 \circ \be_2 \circ \be_3 \circ 3}(28)} = 1.
\end{align*}  

Let \[\bE := (\be_1 \circ \be_2 \circ \dots \circ \be_k)^{k/2} \circ 1\]
be a length $N_k := 9k \cdot \frac{k}2 + 1$ sequence and let $X = (\_)^{N - N_k} \circ \bE$.
We show that $\khop(X)_N = \khop(\bE)_{N_k} = 1$ if and only if $G = C_2$.

Without loss of generality, let $\set{j, j+1} = e_{i_j} \in E$ for all $j \in [\frac{k}2-1]$.
Let $e_{i_0} = \set{1, v^*}$, where $v^* = \frac{k}2$ if $G = C_2$ and $v^* = k$ if $G = C_1$.
Assume without loss of generality that $i_1 > i_0$.
We argue inductively that for any $j \in [\frac{k}2]$:
\begin{enumerate}
\item Every two hops simulates a single step of $\nxt$:
\[\jhop{2j}(\bE)_{N_k} = \nxt^j(1, v^*)_1 = \begin{cases} j & \text{if} \ j+1 < \frac{k}2 \ \text{or} \ G = C_1, \\ 1 & \text{if} \ j = \frac{k}{2}, \ G = C_2;\end{cases}\]
\item Every two hops never ``jumps'' by more than one repetition of all edges gadgets:
\[\find^{2j}_{\bE}(N_k) \geq \find^{2j-2}_{\bE}(N_k) - 9(k-1);\]
\item The executed gadget corresponds to the correct edge and the gadget is executed correctly:
\[ \find^{2j}_{\bE}(N_k) \in \set{9kj' + 9 i_j + \iota: j' \in \N, \iota \in \set{2, 4}}.\]
\end{enumerate}

If all three conditions are met, then $\khop(X)_N = 1$ if and only if $G = C_1$ from condition 1.

We first show that the base case holds for $j = 1$.
Since $i_1 > i_0$, the second-last time 1 appears in the $\bE$ is in the final encoding $\be_{i_1}$. 
By the two-case analysis of the $\be_{i_1}$ gadget, we validate that $\jhop{2}(\bE)_{N_k} = 2$ and conditions (1) and (3) hold.
Since $\be_{i_1}$ cannot be the first edge encoding appearing in $\be_1 \circ \be_2 \circ \dots \circ \be_k,$ owing to it following $\be_{i_0}$), condition (2) is satisfied.  

Suppose that the inductive hypotheses holds up to $j < \frac{k}2$.
Then, we argue that it holds for $j+1$.
Since $\jhop{2j}(\bE)_{N_k} = j+1$ (from condition (1)) and $\find^{2j}_{\bE}(N_k)$ resides at the left-most side of the gadget for $\be_{i_{j}}$ (from condition (3)), the two subsequent $\find_{\bE}$ iterations must occur in the gadget $\be_{i_{j+1}}$.
Because $\find^{2j}_{\bE}(N_k) \geq 9k (k - j)$ (from condition (2)), all edges appear in the $k$ gadgets to the left of $\find^{2j}_{\bE}(N_k)$, and all other edges (including $\be_{i_{j+1}}$) must occur before the next occurrence of $\be_{i_j}$.
Thus, the two hops occur in the $\be_{i_{j+1}}$ gadget (within distance $9(k-1)$) and results in a properly positioned $\find^{2j+2}_{\bE}(N_k)$ with $\jhop{2j+2}(\bE)_{N_k} = \nxt^{j+1}(1, v^*)_1$.

Since an MPC protocol can convert $G$ to $X$ using a constant number of layers, and because $\pi'$ outputs $T(X)_N = 1$ if and only if $G = C_1$, we can construct a protocol of $\pi$ by simulating $\pi'$. 
Because the protocol $\pi$ uses $s = O(k^{1 - \epsilon /2})$ local memory and $q = O(k^{2 / \xi})$ machines, \Cref{conj:cycle} implies that the existence of $T$ requires $L = \Omega(\log k)$. 
\end{proof}

\section{Proofs from \Cref{sec:other-models}}

\subsection{Multi-player pointer chasing communication complexity}\label{assec:pc}

We introduce the multi-pass multi-player blackboard communication model studied by \citet{gm09} and \citet{an21} to prove lower bounds for multi-pass streaming algorithms.
A protocol in this model specifies how $k$ players, each possessing a portion of a shared input, can jointly compute a function on the input over the course of $R$ rounds of communication.
In each round, all players take turns to broadcast an $s$-bit message to all other players.
We provide a formal definition of the model as described in Section~6 of \citet{an21}.

\begin{definition}
A \emph{$k$-player $R$-round $s$-space sequential blackboard communication protocol} includes $k$ players $P_1, \dots, P_k$. On input $Z$ that can be partitioned into $(Z_1, \dots, Z_k)$, each player $P_j$ is provided with its respective $Z_j$. 
In each round, players communicate via a shared blackboard. 
That is, in round $r$ and in order $P_k, \dots, P_1$, each player $P_j$ writes a message $\Pi_j^r \in \bit^s$ on the blackboard (which can be viewed by all players) as a potentially randomized function of input $Z_j$ and all information on the blackboard.
After the conclusion of $R$ rounds, the final message $\Pi_1^R$ is the output of the protocol.
\end{definition}

\citet{an21} proves a lower bound on the round complexity necessary to solve the well-studied \emph{multi-party pointer chasing problem} of \citet{nw93}.
We present the problem as defined by \citet{an21}.

\begin{definition}
For $q, k \in \Z_+$, let an \emph{$(q, k)$-layered graph} $G = (V, E)$ have disjoint vertex layers $V_1, \dots, V_{k+1}$ with $V = V_1 \cup \dots \cup V_{k+1}$ and each $|V_j| = q$ and edge layers $E_1, \dots, E_k$ with $E = E_1 \cup \dots \cup E_k$ and each $E_j$ being a perfect matching between $V_j$ and $V_{j+1}$.
The \emph{pointer chasing} task is provides a $(q, k)$-layered graph $G$, an arbitrary $v \in V_1$, and an arbitrary equipartition $V_{k+1}^1$ and $V_{k+1}^2$ of $V_{k+1}$ as input and asks whether $v$ is connected to a vertex in $V_{k+1}^1$ or $V_{k+1}^2$.
\end{definition}

\citet{an21} give the following lower bound.

\begin{proposition}[Proposition~4.12 of \citealp{an21}]\label{prop:pc}
Consider a $k$-player $R$-round $s$-space sequential blackboard protocol that solves the $(q, k)$-pointer chasing task where each player $P_j$ is provided with the matching $E_j$ and $v$ and $V_{k+1}^1, V_{k+1}^2$ are globally known.
Then, the protocol succeeds with probability at least $\frac23$ only if $R \geq k$ or $s = \Omega(\frac{q}{k^5})$.
\end{proposition}

All of the lower bounds in \Cref{sec:other-models} are most naturally proved by reducing from $\khop$, rather than pointer chasing.
So we first prove a lower bound for $\khop$ using the lower bound for pointer chasing from \Cref{prop:pc}.

\begin{proposition}\label{prop:khop-blackboard}
Consider a $k$-player $R$-round $s$-space sequential blackboard protocol that computes $\khop(X)_N$ on any $X \in \Sigma^N$ for $\Sigma = [2q+2]$ with $q = \floor{\frac{N}{2k}}$ where each player $P_j$ is provided with $X^j := (X_{2(k-j)q+1}, \dots, X_{2(k-j+1)q})$, except for $P_1$, who is given $X^1 := (X_{2(k-1)q + 1}, \dots, X_N)$.
Then, the protocol succeeds with probability at least $\frac23$ only if $R \geq k$ or $s = \Omega(\frac{N}{k^6})$.
\end{proposition}
\begin{proof}
Assuming the existence of a $k$-player $R$-round $s$-space sequential blackboard protocol for $\khop(X)_N$ as described above, we design a protocol for solving $(q, k)$-pointer chasing with $R$ rounds and $s$-size messages.
The claimed lower bound will then follow by \Cref{prop:pc}.

Consider any pointer chasing input with universally known $V_1, \dots, V_{k+1}$, $v \in V_{1}$, and $V_{k+1}^1$ and $V_{k+1}^2$, and each player $P_j$ knowing matching $E_j$.
We recursively define $v_1, \dots, v_k+1$ such that $v_1 = v$ and $(v_{j}, v_{j+1}) \in E_{j}$, noting that the output hinges on whether $v_{k+1} \in V_{k+1}^1$.

Without loss of generality, let $v = 1$ and \[V_j = \begin{cases}\set{1, \dots, q} & \text{if $j$ is odd,} \\ \set{q+1, \dots, 2q} & \text{if $j$ is even.}\end{cases}\]
Each player independently determines their substring $X^j$ of a input $X$ to $\khop$ before running the aforementioned protocol:
\begin{itemize}
\item Player $P_1$ encodes $X^1$ by letting $X_N = s = 1$ and for any $i \in {1, \dots, 2q}$, letting \[X^1_{i} = \begin{cases} \frac{i + 1}2 \in V_1 &\text{if $i$ is odd,} \\
i' \in V_2 &\text{if $i$ is even,} \ (\frac{i}2, i') \in E_1.\end{cases}\]
This ensures that that every integer in $\set{1, \dots, 2q}$ appears exactly once in $X^1_{1}, \dots, X^1_{2q}$,
which in turn guarantees that $\find^1_X(N) = (k-1+1)q + 2$ and that $X_{\find^1_X(N)} = v_2$ where $(1, i') \in E_1$.
\item For any $j \in \set{2, \dots, k-1}$, player $P_j$ encodes $E_j$ as $X^j$ as follows. If $j$ is odd, then for every $i \in \set{1, \dots, 2q}$,
\[X^j_{i} = \begin{cases} \frac{i + 1}2 \in V_j &\text{if $i$ is odd,} \\
i' \in V_{j+1} &\text{if $i$ is even,} \ (\frac{i}2, i') \in E_j.\end{cases}\]
Alternatively, if $j$ is even,
\[X^j_{i} = \begin{cases} q + \frac{i + 1}2 \in V_j &\text{if $i$ is odd,} \\
i' \in V_{j+1} &\text{if $i$ is even,} \ (\frac{i}2, i') \in E_j.\end{cases}\]
Since every odd token corresponds to a vertex in $V_j$ and each subsequent token corresponds to the vertex it's connected to by $E_j$, we can ensure that for every $i \in [2q]$:
\[(X_{2(k-j+1) + i}, X_{\find_X^1(2(k-j+1) + i)}) \in E_{j}.\]
Hence, it follows inductively that $X_{\find_X^j(N)} = v_{j+1}$.
\item Player $P_k$ encodes $X^k$ if $k$ is odd by letting \[X^k_i = X_i = \begin{cases}
\frac{i + 1}2 \in V_k &\text{if $i$ is odd,} \\
2q + 1 & \text{if $i$ is even, $(\frac{i}2, v) \in E_{k}$, and $v \in V_{k+1}^1$},\\
2q + 2 & \text{if $i$ is even, $(\frac{i}2, v) \in E_{k}$, and $v \in V_{k+1}^2$}.\end{cases}\]
Likewise, if $k$ is even,
\[X^k_i = X_i = \begin{cases}
q + \frac{i + 1}2 \in V_k &\text{if $i$ is odd,} \\
2q + 1 & \text{if $i$ is even, $(\frac{i}2, v) \in E_{k}$, and $v \in V_{k+1}^1$},\\
2q + 2 & \text{if $i$ is even, $(\frac{i}2, v) \in E_{k}$, and $v \in V_{k+1}^2$}.\end{cases}\]
\end{itemize}
These jointly ensure that \[\khop(X)_N = X_{\find^k_X(N)} = \begin{cases} 2q + 1 & \text{if $v_{k+1} \in V_{k+1}^1$,} \\ 2q + 2 & \text{if $v_{k+1} \in V_{k+1}^2$.}\end{cases}\]

Therefore, by formatting $E_1, \dots, E_k$ appropriately as $X$, running the protocol for $\khop(X)_N$, and observing that the final output of player $P^1$ is $2q + 1$ if and only if $v_{k+1} \in V_{k+1}^1$, there exists a $k$-player $R$-round $s$-space protocol for pointer chasing. 
Hence, by \Cref{prop:pc}, the protocol for $\khop(X)_N$ must use $R \geq k$ rounds or $s = \Omega(\frac{N}{k^6})$ space. 
\end{proof}

\subsection{Proofs of \Cref{ssec:rnn}}\label{assec:rnn}

\corrnn*
\begin{proof}
Suppose there exists a multi-layer RNN computing output $Y$ with $Y_{N, 1} = \khop(X)_N$ from input $X$ with intermediate states $Z_1, \dots, Z_{L-1}$ and hidden states $H^1, \dots, H^L$.
For any $\ell \in [L]$ and $i \leq i'$, note that $Z^{\ell}_{i}, \dots, Z^{\ell}_{i'}$ can be determined exactly from $H_{i-1}^\ell$ and $Z^{\ell-1}_{i}, \dots, Z^{\ell-1}_{i'}$.
Given this RNN, we provide a multi-player blackboard communication protocol for solving $\khop(X)_N$ under the input model of \Cref{prop:khop-blackboard}.

In round $r$, we assume inductively that each player $P_j$ knows $Z^{\ell-1, j} = (Z^{\ell-1}_{2(k-j)q+1}, \dots, Z^{\ell-1}_{2(k-j+1)q})$, except for $P_1$, who knows $Z^{{\ell-1}, 1} = (Z^{\ell-1}_{2(k-1)q + 1}, \dots, Z^{\ell-1}_N)$.
In descending order, each player $P_j$ computes $Z^{\ell, j}$ and $H^\ell_{2(k-j+1)q}$---writing the latter on the blackboard---from $Z^{\ell-1, j}$ and $H^\ell_{2(k-j)q}$,which was written on the blackboard by the previous player.
Thus, player $P^1$ after round $L$ knows and outputs $Z^L_{N,1} = Y_{N, 1} = \khop(X)_N$, which provides an $L$-round protocol $m$-space protocol.

So the claimed lower bounds on width and depth follow from \Cref{prop:khop-blackboard}.
\end{proof}

\subsection{Proofs of \Cref{ssec:sub}}\label{assec:sub}
\corkernel*
\begin{proof}
Under the distribution of input $X = (X^1, \dots, X^k)$ to players $P_1, \dots, P_k$ stipulated in the statement of \Cref{prop:khop-blackboard}, we explain how the players can all compute the outcome of a single layer of $H$-headed kernelized attention in a single round of a blackboard protocol.
It is immediate that a depth $L$ network can be simulated in $L$ rounds.

On input $X$, consider $H$ kernelized self-attention units with embeddings $(Q_1', K_1', V_1), \dots, (Q_H', K_H', V_H)$ and output MLP $\psi$.
Each player $P_j$ immediately computes its embeddings $(Q_h'(X^j), K_h'(X^j), V_h(X^j))_{h \in [H]}$, followed by $(K_h'(X^j)^\T V_h(X^j)) \in \R^{m' \times m}$ for each $h \in [H]$.
Because the object is to compute for each $h$ \[\psi(Q_h'(X) K_h'(X)^\T V_h(X)) = \psi(Q_h'(X) \sum_{j=1}^k K_h'(X^j)^\T V_h(X^j)),\]
each player writes their $(K_h'(X^j)^\T V_h(X^j))_{h \in [H]}$ using message size $s = \Theta(m m' H p)$.
Each can then construct $K_h'(X)^\T V_h(X))$ by reading the board, and use it to compute its respective outputs without requiring supplemental communication.

Hence, $T$ (and thus $\khop(X)_N$) can be simulated using an $L$-round blackboard protocol with message size $s = \Theta(m m' H p)$, and the corollary follows from \Cref{prop:khop-blackboard}.
\end{proof}

\corlong*
\begin{proof}
As in the proof of \Cref{cor:kernel}, we explain how each player can compute their respective outputs of a single unit of self-attention masked by $\Lambda^{w, g}$.

To compute the output corresponding to $X_i$, note that it is necessary to only know the embeddings corresponding to $X_{i - w}, X_{i-w+1}, \dots, X_{i+ w}$ and $X_{g}, X_{2g}, \dots, X_{\floor{N / g} g}$.
Thus, player $X^j$ can compute the outputs of all of their inputs $X^j = (X_{2(k-j)q+1}, \dots, X_{2(k-j+1)q})$ given access to \[X_{2(k-j)q+1-w}, \dots, X_{2(k-j)q}, X_{2(k-j+1)q + 1}, \dots, X_{2(k-j+1)q + w},\] as well as $X_{g}, X_{2g}, \dots, X_{\floor{N / g} g}$.

Therefore, the protocol can be simulated if each player $X^j$ writes inputs \[X_{2(k-j)q+1}, \dots, X_{2(k-j)q+w}, X_{2(k-j+1)q -w +1}, \dots, X_{2(k-j+1)q} \in \R^m,\] in addition to all $X_i \in X^j$ such that $i \equiv 0 \pmod g$.
This can be accomplished by a protocol where each player writes $s = O((w + \frac{N}{gk})mp)$ bits of information on the blackboard.

By repeating this protocol in parallel for every head and sequentially for every layer, $T$ and $\khop(X)_N$ can be simulated, and hence the claim follows from \Cref{prop:khop-blackboard}.
\end{proof}

\subsection{Proofs of \Cref{ssec:cot}}\label{assec:cot}

\corcothard*
\begin{proof}
We reduce to \Cref{prop:khop-blackboard}.
Consider some input $X \in \R^N$ partitioned into $X^1, \dots, X^j$ as specified by the proof of \Cref{prop:khop-blackboard} with chain-of-thought $\Xcot$ and $\khop(X)_N$ determined by some masked transformer $T$.\footnote{We abuse notation to index $X_{N + i} = \Xcoti{i}$ and let $X_i \in X^j$ be true if $i \in \set{2(k-j)q+1, \dots, w(k-j+1)q}$.}
Suppose $T$ has embeddings $(Q_h, K_h, V_h)_{h \in [H]}$ and output MLP $\psi$.
We provide an $(\Ncot + 1)$-round blackboard protocol to compute $\khop(X)_N$ from $X$.

Suppose in the $r$th round of the protocol, all players know $\Xcoti{1}, \dots, \Xcoti{r-1}$ and aim to compute 
\begin{align*}
T(X \circ \Xcot)_{N + r - 1} 
&= \begin{cases} \Xcoti{r} &\text{if} \ r \leq \Ncot \\ \khop(X)_N  &\text{if} \ r = \Ncot+1\end{cases} \\
&= \psi_{N + r - 1}\paren{X_{N+r-1} + \sum_{h=1}^H \frac{\sum_{i=1}^{N+r-1} \exp(Q^h_{N+r-1}(X_{N + r -1})^\T K^h_i(X_i)^\T) V^h_i(X_i)}{\sum_{i=1}^{N+r-1} \exp(Q^h_{N+r-1}(X_{N + r -1})^\T K^h_i(X_i))}}.
\end{align*}
If we let
\begin{align*}
S_{r, h, j} &= \sum_{X_i \in X^j} \exp(Q^h_{N+r-1}(X_{N + r -1})^\T K^h_i(X_i)^\T) V^h_i(X_i) \in \R^m, \\
S_{r, h, \mathrm{CoT}} &= \sum_{i=N+1}^{N+r-1} \exp(Q^h_{N+r-1}(X_{N + r -1})^\T K^h_i(X_i)^\T) V^h_i(X_i) \in \R^m, \\
Z_{r, h, j} &= \sum_{X_i \in X^j} \exp(Q^h_{N+r-1}(X_{N + r -1})^\T K^h_i(X_i)^\T)  \in \R, \\
Z_{r, h, \mathrm{CoT}} &= \sum_{i=N+1}^{N+r-1} \exp(Q^h_{N+r-1}(X_{N + r -1})^\T K^h_i(X_i)^\T)  \in \R,
\end{align*}
then we observe that 
\[T(X \circ \Xcot)_{N + r - 1} =  \psi_{N + r - 1}\paren{X_{N + r-1} + \sum_{h=1}^H \frac{\sum_{j=1}^k S_{r, h, j} + S_{r, h, \mathrm{CoT}}}{\sum_{j=1}^k Z_{r, h, j} + Z_{r, h, \mathrm{CoT}}}}.\] 
Each player $P_k$ computes $(S_{r, h, j}, Z_{r, h, j})_{h \in [H]}$ and writes them on the blackboard with $O(mHp)$-bit messages.
Since $S_{r, h, \mathrm{CoT}}$ and $Z_{r, h, \mathrm{CoT}}$ are known by all players, every player can individually $T(X \circ \Xcot)_{N + r - 1}$.

By induction, all players know $\khop(X)_N$ after $\Ncot + 1$ rounds. 
The claim now follows from \Cref{prop:khop-blackboard}. 
\end{proof}

\section{Proofs of low-level attention constructions}\label{asec:low-level}

\subsection{Hardmax simulation proof of \Cref{assec:transformers}}

\lemmahardmax*
\begin{proof}
For some $p' = \Theta(p + \log \frac1\xi)$ and $c = \Theta(\frac{p'+\zeta}{\xi} \cdot \log N)$
where $\zeta$ is as in \Cref{assec:transformers}), let $f'$ have query embedding $Q'(X) =  cQ(X)$ and identical key $K$ and value $V$ embeddings as $f$.
Therefore, by construction,
these embeddings can be written with precision
$p' = O(\ln(c) + p) = O(\log \frac 1 \xi + \log \log N + p) = O(p)$.

Let $\hat{f}'$ be a valid $p'$-bit implementation of $f'$, meaning
that the two $\|\hat{f'} - f'\|_\infty =O(1/2^{p+1})$ (thus $\hat{f'}$ rounds $f'$ to
$p'$ bits of precision), and fix some $X$. We first show that the softmax matrix is sufficiently close to that of the hardmax and is also a valid $p'$-bit implementation of the hardmax.
Without loss of generality, let $1 \in \imax(A(X)_i)$.
First, note that 
\[\sum_{i' \not \in \imax(A(X)_i)} \exp(c A(X)_{i, i'}) \leq \frac{N}{\exp(c\xi)} \exp(c A(X)_{i, 1}) = 
\frac{1}{N^{O(p' + \zeta)}} \exp(c A(X)_{i, 1}).\]
Then,
\begin{align*}
\abs{\sm(cA(X))_{i, 1} - \hm(A(X))_{i, 1}}
&= \frac1{|\imax(A(X)_i)|} - \frac{\exp(cA(X)_{i, 1})}{\sum_{i'=1}^N  \exp(cA(X)_{i, i'})}  \\
&\leq \frac{\sum_{i'\not\in\imax(A(X)_i)} \exp(c A(X)_{i, i'})}{|\imax(A(X)_i)| \exp(c A(X)_{i, 1})}
= \frac{1}{N^{\Omega(p' + \zeta)}}.
\end{align*}
Likewise, for any $i'' \not\in \imax(A(X)_i)$:
\begin{align*}
\abs{\sm(cA(X))_{i, i''} - \hm(A(X))_{i, i''}}
&\leq \frac{\exp(cA(X)_{i, i''})}{\sum_{i'=1}^N  \exp(cA(X)_{i, i'})}
= \frac{1}{N^{\Omega(p' + \zeta)}}.
\end{align*}
Therefore,
\[\norm[2]{\sm(c A(X))_i - \hm(c A(X))_i} 
\leq \sqrt{N} \cdot \max_{i''} \abs{\sm(c A(X))_{i, i''} - \hm(c A(X))_{i,i''}} 
= \frac{1}{N^{\Omega(p' + \zeta)}}.\]

We conclude that the approximation is sufficiently close,
meaning it is $O(1/ 2^{p'})$, whereby it is exact after rounding:
\begin{align*}
\norm[\infty]{\hat{f}'(X) - \hm(Q(X) K(X)^\T ) V(X)}
&\leq
\norm[\infty]{f'(X) - \hm(Q(X) K(X)^\T ) V(X)}
+
\norm[\infty]{\hat{f}'(X) - f'(X)}
\\
&\leq \max_{i, j} \abs{\sm(c A(X))_i^\T V(X)_{\cdot, j} - \hm(A(X))_i^\T V(X)_{\cdot, j}} + O\paren{\frac{1}{2^{p'}}}\\
& \leq \max_{i, j} \norm[2]{\sm(c A(X))_i^\T - \hm( A(X))_i^\T} \norm[2]{V(X)_{\cdot, j}} + O\paren{\frac{1}{2^{p'}}} \\
& \leq \frac1{N^{\Omega(p' +\zeta)}} \cdot \sqrt{N} \cdot N^\zeta + O\paren{\frac{1}{2^{p'}}}
= O\paren{ \frac{1}{2^{p' }}}.
\end{align*}
Therefore, $\hat{f}'$ is a valid $p'$-bit implementation of $\hm(Q(X)K(X)^\T)V(X)$.
\end{proof}

\subsection{Constructions for \Cref{assec:lemma-routing-block}}

\propqsp*

\begin{proof}

Following the proof of Theorem~2 of \citet{sht23}, there exist $p$-bit precision vectors $u_1, \dots, u_N \in \set{\pm 1 / \sqrt{m}}^m$ and $w_S$ with $w_S \leq 2 \sqrt{Q}$ for all $S \in {N \choose \leq Q}$ such that
\begin{align*}
u_i^\T w_S &= 1, \ \text{for all} \ i \in S \\
u_i^\T w_S &\leq \frac12, \ \text{for all} \ i \not\in S.
\end{align*}
We then design the embeddings of $\qsp_{Q, d}$ with 
\begin{align*}
Q(X)_i &= (u_i, 1), \\
K(X)_i &= \begin{cases}(w_{S_i}, 0) & \text{if} \ i > 0,\\ (\vec0, \frac34) & \text{if} \ i = 0,\end{cases}\\
V(X)_i &= \begin{cases}z_i & \text{if} \ i > 0,\\  \vec0 & \text{if} \ i = 0.\end{cases}
\end{align*}
As a result, 
\begin{align*}
Q(X)_i^\T K(X)_{i'} &= 1 & \text{if} \ i \in S_{i'}, i' > 0, \\
Q(X)_i^\T K(X)_{i'} &\leq \frac12 & \text{if} \ i \not\in S_{i'}, i' > 0, \\
Q(X)_i^\T K(X)_{0} &= \frac34.  \\
\end{align*}
Hence, the largest inner products for query $i$ correspond to $i'$ for all $S_{i'}\ni i$ if any exist, and 0 otherwise.
There exists a margin of at least $\frac14$ between the largest inner product in each row and all others.
By applying \Cref{lemma:hardmax}, we conclude that there exists a self attention unit $f'$ with embedding dimension $p = \Theta(\log N)$ that computes
\[f'(X) = \hm(Q(X) K(X)^\T) V(X) = \qsp(X).\qedhere\]
\end{proof}

\subsection{Constructions for \Cref{assec:thm-mpc-simulation}}

\newcommand{\qin}{q_{\mathrm{in}}}
\newcommand{\qout}{q_{\mathrm{out}}}

\lemmacondense*
\begin{proof}
Let $M = \max(\nin, q)$ and $Q, K, V: \pword^M \to \R^{M \times s}$ be the query, key, and value embeddings of the attention unit $f$ in $\init$, and let $\psi: \R^{M \times s} \to \pword^{s} \times [N]$ be its output MLP.
Let $\qin = \ceil{\frac{\nin}s}$ denote the number of machines used to store the inputs.

Let $\dst_{i'} = \ceil{\frac{i'}s} \in [\qin]$ denote the machine that stores the input token index $i' \in [\nin]$ in the MPC protocol, and let \[\inc_{i} = \set{(s-1)i + 1, \dots, \min(si, \nin)}\]
denote the set of all input tokens indices belonging to $\machin{1}_i$ for machine $i \in [\qin]$.

For each machine $i \in [\qin]$, we define the query embedding as \[Q(\inp)_{i} = \paren{\cos\paren{\frac{2\pi i}{M}}, \sin\paren{\frac{2\pi i}{M}}, \dots, \cos\paren{\frac{2\pi i}{M}}, \sin\paren{\frac{2\pi i}{M}}}.\]
Likewise, for each token index $i' \in [\nin]$, the key and value vectors are 
\begin{align*}
K(\inp)_{i', (2\iota -1, 2\iota)} &= 
\begin{cases} 
\paren{\cos\paren{\frac{2\pi \cdot{\dst_{i'}}}{M}}, \sin\paren{\frac{2\pi \cdot{\dst_{i'}}}{M}}} 
& \text{if} \ i' \leq \nin, \ i' \equiv \iota \pmod{s}, \\
(0, 0) & \text{otherwise,} 
\end{cases} \\
V(\inp)_{i', (2\iota -1, 2\iota)} &= 
\begin{cases} \paren{\inp_{i'}, i'} & \text{if} \ i' \leq \nin, \ i' \equiv \iota \pmod{s}, \\
(0, i') & \text{otherwise.} 
\end{cases}
\end{align*}
These definitions guarantee that large inner products only occur between machine queries $Q(\inp)_i$ and tokens keys $K(\inp)_{i'}$ when $\inp_{i'}$ is allocated to $\machin{1}_{i}$.
That is,
\begin{align*}
Q(\inp)_i^\T K(\inp)_{i'}
&= 1, & \text{if} \ i' \in \inc_i \\ Q(\inp)_i^\T K(\inp)_{i'} &\leq 1 - \Omega\paren{\frac{1}{M^2}}, & \text{otherwise.}
\end{align*}
By applying \Cref{lemma:hardmax} with $\xi = \Omega(\frac1{N^2})$, there exists some self-attention unit $f'$ such that 
\[f'(\inp)_i 
=\hm(Q(\inp) K(\inp)^\T)
= \frac{(\inp_{i'}, i')_{i' \in \inc_i}}{|\inc_i|}.\]A proper choice of $\psi$ and an invocation of the definition of $\machin{1}$ ensures that $\init(\inp)_i = \psi(f(\inp))_i = \machin{1}_i$.
\end{proof}

\lemmaspread*
\begin{proof}
This argument inverts that of \Cref{lemma:condense}, after applying the $\local_R$ to transform $\machin{R}$ to $\machout{R}$.
Let $Q, K, V: \pword^M \to \R^{M \times s}$ be the query, key, and value embeddings of the only attention unit $f$ in $\final$, and let $\psi: \R^{M \times s} \to \pword^{s} \times [N]$ be its output MLP.
Let $\qout = \ceil{\frac{\nout}s}$ denote the number of machines storing relevant information for the output of the MPC protocol.

For each machine $i' \in [\qout]$, let \[\outg_{i'} = \set{(s-1)i' + 1, \dots, \min(si', \nout)}\]
denote the set of all token indices receiving its output.
Likewise, for each token index $i \in [\nout]$, let $\src_i = \ceil{i / s}$ be the machine containing its relevant token.
We define $Q = Q' \circ \local_R, K = K' \circ \local_R, V = V' \circ \local_R$ as follows.
\begin{align*}
Q'(\machout{R})_{i, (2\iota-1, 2\iota)} &= \begin{cases} \paren{\cos\paren{\frac{2\pi \floor{\src_i}}{M}}, \sin\paren{\frac{2\pi \floor{\src_i}}{M}}} & \text{if} \ i \leq \nout, \ i \equiv \iota \pmod{s} \\
(0, 0) & \text{otherwise.} 
\end{cases} \\
K'(\machout{R})_{i'} &= \paren{\cos\paren{\frac{2\pi i'}{M}}, \sin\paren{\frac{2\pi i'}{M}}, \dots, \cos\paren{\frac{2\pi i'}{M}}, \sin\paren{\frac{2\pi i'}{M}}}. \\
V'(\machout{R})_{i'} &= \msgout{R}_{i'}.
\end{align*}
Applying \Cref{lemma:hardmax} as before yields 
\[f(\machin{R})_i = \begin{cases}\machout{R}_{i'} & \text{if} \ i \in \outg_{i'}, \\ 0 & \text{otherwise.} \end{cases}\]
A properly chosen $\psi$ ensures that $\final(\machin{R})_i = \psi(f(\machin{R}))_i = \outp_i$.
\end{proof}

\subsection{Constructions for \Cref{assec:k-hop-construction}}
\label{assec:low-level-k-hop}

\lemmalookup*
\begin{proof}
We let $V(X_i) = (\rho(X_i), \vec0)$ and define sinusoidal embeddings $Q$ and $K$ with
\begin{align*}
Q(X)_i &= \paren{\cos\paren{ \frac{2 \pi \tau(i,X_i)}N}, \sin\paren{\frac{2 \pi \tau(i,X_i)}N}, \vec0}, \\
K(X)_i &= \paren{\cos\paren{ \frac{2 \pi i}N}, \sin\paren{\frac{2 \pi i)}N}, \vec0} .
\end{align*}
Note that
\begin{align*}
Q(X)_i^\T K(X)_{i'} &= 1, & \text{if $\tau(i,X_i) = i'$,} \\
Q(X)_i^\T K(X)_{i'} &\leq \cos\paren{\frac{2\pi}N} = 1 - \Omega\paren{\frac{1}{N^2}}, & \text{otherwise.}
\end{align*}

By applying \Cref{lemma:hardmax} with $\xi = \Omega(\frac1{N^2})$, we conclude that a satisfactory self-attention unit exists.
\end{proof}

\lemmalastoccurrence*

\begin{proof}
Let $N' = N |\Sigma|$. 
We define token embeddings as follows, including start token ``dummy embeddings'' as discussed in \Cref{assec:transformers}.
\begin{align*}
Q(X)_i &= \paren{\cos\paren{\frac{2\pi (N\mu_2(X_i) + i) }{N |\Sigma|}}, \sin\paren{\frac{2\pi (N\mu_2(X_i) + i) }{N |\Sigma|}}, 1, \vec0}, \\
K(X)_i &= \paren{\cos\paren{\frac{2\pi (N\mu_1(X_i) + i) }{N |\Sigma|}}, \sin\paren{\frac{2\pi (N\mu_1(X_i) + i) }{N |\Sigma|}}, 0, \vec0}, \\
K(X)_0 &= \paren{0, 0, \cos\paren{\frac{2\pi (N - \frac12)}{N |\Sigma|}}, \vec0}, \\
V(X)_i &= (\rho(X_i), \vec0), \\
V(X)_0 &= \vec0.
\end{align*}
Taken together, these embeddings provide the following characterization of the inner products (with causal masking matrix $\Gamma$):
\begin{align*}
Q(X)_0^\T K(X)_{i'} + \Gamma_{i,i'} &= \cos\paren{\frac{2\pi(i - i')}{N |\Sigma|}} & \text{if} \ i \geq i' > 0, \ \mu_1(X_{i'}) = \mu_2(X_{i}), \\
Q(X)_i^\T K(X)_{i'} + \Gamma_{i,i'} &\leq \cos\paren{\frac{2\pi}{N }} & \text{if} \ i \geq i' > 0, \ \mu_1(X_{i'}) \neq \mu_2(X_{i}), \\
Q(X)_i^\T K(X)_{i'} + \Gamma_{i,i'} &= -\infty & \text{if} \ i < i', \\
Q(X)_i^\T K(X)_i + \Gamma_{i,0} &= \cos\paren{\frac{2\pi (N - \frac12)}{N |\Sigma|}}.
\end{align*} 
As a result, the largest inner product $Q(X)^\T_i K(X)_{i'}$ for some $i$ is the largest $i'$ with $\mu_1(X_{i'}) = \mu_2(X_i)$ if one exists and $i' = 0$ otherwise.
Furthermore, there exists a margin of $\Omega(\frac1{N^2 |\Sigma|^2})$ between this inner product and all others.
We conclude by applying \Cref{lemma:hardmax}.
\end{proof}

\section{Further empirical analysis of $k$-hop induction heads}\label{asec:empirics}

This appendix presents in-depth explanations of the empirical results of \Cref{ssec:khop-emp}, along with further experiments.
Taken together, these results suggest that the relationship between the number of hops $k$ and the depth $L$ of transformers trained on the task is well-characterized by the representational thresholds of \Cref{thm:k-hop-construction} and \Cref{cor:k-hop-hardness}; that the construction described in the proof of \Cref{thm:k-hop-construction} is attainable by trained models; and deep models likely exhibit an inductive bias that favors compositional learning rules in the finite sample regime.  

We define our experimental methodology precisely in \Cref{assec:exp-details} and provide supporting evidence for our claims in the subsequent sections.

\paragraph*{Exponential powers of depth.}
Our principal empirical claim is that incrementing the depth $L$ of a transformer exponentially increases the model's capabilities to learn $k$-hop induction heads tasks.  
We explore this claim primarily in \Cref{assec:exp-depth}, where we compare this empirical claim with the relevant theoretical results (\Cref{thm:k-hop-construction} and \Cref{cor:k-hop-hardness}), which suggest a similar dependence.
We further study the impacts of increasing the embedding dimension $m$ of the transformer in \Cref{assec:exp-width} and find that doubling the width is roughly equivalent in performance to incrementing the depth by one.
\begin{ec}\label{ec:depth}A transformer $T \in \mtran{m, L, H}N$ trained with Adam to solve $\khop$ has small token-wise classification error if $L \log(m) = \Omega(\log k)$ and large error if $L \log m = O(\log k)$.
\end{ec}

\paragraph*{Mechanistic alignment with theoretical construction.}We further demonstrate the empirical salience of our theoretical construction by conducting a study of the interpretability of learned transformers in \Cref{assec:exp-interp}. 
This investigation reveals that the attention matrices of sufficiently deep transformers exhibit an implementation of a circuit that relies on the same ``doubling'' principle of the construction in the proof of \Cref{thm:k-hop-construction}.
The resulting circuit is comprised of the same intermediate products that are used in that $\khop$ construction.
\begin{ec}\label{ec:interp}
The outputs of individual attention matrices of a transformer $T \in \mtran{m, L, H}N$ trained with Adam to solve $\khop$ with $L = \Omega(\log k)$ and evaluated on input $X \in \Sigma^N$ (i) correspond to the $\find^j_X$ intermediate products of the \Cref{thm:k-hop-construction} construction and (ii) demonstrate a ``doubling'' phenomenon where the each head layer $\ell$ corresponds to $\find^j_X$ for some $j = O(2^\ell)$.
\end{ec}

\paragraph*{Beneficial inductive biases of depth.}While most of our experiments belong to the ``infinite-sample'' regime where new samples are randomly generated on each training step, we also evaluate our models in two finite-sample regimes in \Cref{assec:exp-finite}.
We find that a small number of samples is sufficient to approach the performance of the infinite-sample regime.
When the amount of training data is small, we find that deeper models perform better than shallower models, possibly due to an inductive bias that favors compositional hypotheses.
\begin{ec}\label{ec:sample}
$\khop$ can be learned in a sample-efficient manner by transformers $T \in \mtran{m, L, H}N$ trained with Adam with $L = \Omega(\log k)$.
If $T$ overfits to $\khop$ tasks for some $k$,
then increasing the depth $L$ while holding $k$ fixed leads superior performance.
\end{ec}

The experiments detailed here were conducted under limited computational resources.
The authors are interested in future work that would evaluate whether these scaling rules persist on larger architectures and more complex tasks.

\subsection{Experimental details}\label{assec:exp-details}

\newcommand{\kmax}{k_{\max}}
\newcommand{\distk}[1]{\mathcal{D}_{\jhop{#1}}}
\newcommand{\dist}{\mathcal{D}_{\mathrm{multi-hop}}}
\newcommand{\distx}{\mathcal{D}_{\mathcal{X}}}
\newcommand{\err}{\texttt{err}}
\newcommand{\errk}{\texttt{err}_{k}^n}
\newcommand{\bsig}{\overline\Sigma}
\newcommand{\ntr}{n_{\mathrm{train}}}

\paragraph*{Task details.}

We study a multi-task variant of $k$-hop induction heads that predicts $\khop(X) = (0, \khop(X'))$ from input $X = (k, X')$ for $k \in \set{0,1, \dots, \kmax}$\footnote{The task $\jhop0$ is simply the identity mapping: $\jhop0(X') = X'$.} and $X'\in \Sigma^{N-1}$.
We refer to this task as \emph{multi-hop} and provide the task hyper-parameters in \Cref{table:task}.

\begin{table}[h]
\centering
\begin{tabular}{@{}ll@{}}
\toprule
Hyperparameter  & Value \\
\midrule
Context length $N$   & 100 \\
Alphabet size $|\Sigma|$      & 4    \\
Max hops $\kmax$          & 16    \\
\bottomrule
\end{tabular}
\caption{Multi-hop task hyper-parameters}
\label{table:task}
\end{table}

We define the distribution $\dist$ over labeled samples for the multi-hop task and $\distx$ over input sequences $X \in \Sigma^{N-1}$.
We draw a labeled sample $(X, \khop(X)) \sim \dist$ by independently sampling $k \sim \unif(\set{0, 1, \dots, \kmax})$ and $X' \sim \distx$.
Input sequences $X' \sim \distx$ are drawn uniformly from inputs \emph{with no repeating elements}.
That is, we sample $X'_1 \sim \unif(\Sigma)$ and each $X'_{j+1} \sim \unif(\Sigma \setminus \set{X'_{j}})$.
For each $k \in [\kmax]$, let $\distk{k}$ denote the conditional distribution $((k', X'), (0, \jhop{k'}(X'))) \sim \dist \mid (k=k' )$.
Also, let $\dom(\khop) = \set{(k, X'): \pr{X' \sim \distx} > 0}$.

For $\overline{\Sigma} := \Sigma \cup [\kmax]$, we define the $n$-sample \emph{empirical token-wise classification error} of a transformer $T: \bsig^N \to \bsig^N$ on a task $\khop$ as \[\errk(T) = \frac1n \sum_{\iota=1}^n \frac1{|\set{i: \khop(X^\iota)_i \neq \perp}|} \sum_{i=1}^N \indicator{T(X^\iota)_i \neq \khop(X^\iota)_i \neq \perp}, \]
for iid samples $(X^1,\khop(X^1)), \dots, (X^n,\khop(X^n)) \sim\distk{k}$.
We ignore null $\perp$ outputs of $\khop$ when no $k$-hop induction head exists in order to avoid inadvertently over-estimating the performance of transformers on large $k$ tasks, which have a large fraction of null outputs.

\paragraph*{Training details.}

We trained a variety of causally-masked GPT-2 transformers \citep{Radford2019LanguageMA} from HuggingFace to solve the multi-hop task.
The model has an absolute positional encoding.

The transformers are trained with Adam \citep{adam} on the cross-entropy loss.
In the infinite-sample regime, we draw 32 new iid samples from $\dist$ on each training step.
Otherwise, $\ntr$ samples are drawn before training commences and all samples are rotated through batches, before repeating.
We use the hyper-parameters in \Cref{table:model} to train all of the models identified in \Cref{table:all}.

\begin{table}[h]
\centering
\begin{tabular}{@{}ll@{}}
\toprule
Hyperparameter  & Value \\
\midrule
Embedding dimension $m$   & $\set{128, 256}$ \\
Depth $L$      & $\set{2, 3, 4, 5, 6}$    \\
Number of heads $H$          & $\set{4, 8}$    \\
Vocabulary size  & 30 \\
Activation function & GeLU \\
Layer norm $\epsilon$ &  $10^{-5}$\\
Training samples $\ntr$ & $\set{10^3, 3 \cdot 10^3, \infty}$ \\
Learning rate & $10^{-4}$ \\
Training steps & $10^5$ \\
Batch size & 32 \\
\bottomrule
\end{tabular}
\caption{Model and training hyper-parameters}
\label{table:model}
\end{table}

\begin{table}[!htb]
\centering
\begin{tabular}{crrrrr}
\toprule
Identifier & Heads $H$ & Embedding dimension $m$ & Depth $L$ & Training samples $\ntr$ & Total parameters  \\
\midrule
$T_{4, 2}^{\infty}$ & 4 & 128 & 2 & $\infty$ & 413,440 \\
$T_{4, 3}^{\infty}$ & 4 & 128 & 3 & $\infty$ & 611,712 \\
$T_{4, 4}^{\infty}$ & 4 & 128 & 4 & $\infty$ & 809,984 \\
$T_{4, 5}^{\infty}$ & 4 & 128 & 5 & $\infty$ & 1,008,256 \\
$T_{4, 6}^{\infty}$ & 4 & 128 & 6 & $\infty$ & 1,206,528 \\
$T_{8, 2}^{\infty}$ & 8 & 256 & 2 & $\infty$ & 1,613,312 \\
$T_{8, 3}^{\infty}$ & 8 & 256 & 3 & $\infty$ & 2,403,072 \\
$T_{8, 4}^{\infty}$ & 8 & 256 & 4 & $\infty$ & 3,192,832 \\
$T_{8, 5}^{\infty}$ & 8 & 256 & 5 & $\infty$ & 3,982,592 \\
$T_{8, 6}^{\infty}$ & 8 & 256 & 6 & $\infty$ & 4,772,352 \\
$T_{4, 2}^{3000}$ & 4 & 128 & 2 & $3000$ & 413,440 \\
$T_{4, 3}^{3000}$ & 4 & 128 & 3 & $3000$ & 611,712 \\
$T_{4, 4}^{3000}$ & 4 & 128 & 4 & $3000$ & 809,984 \\
$T_{4, 5}^{3000}$ & 4 & 128 & 5 & $3000$ & 1,008,256 \\
$T_{4, 6}^{3000}$ & 4 & 128 & 6 & $3000$ & 1,206,528 \\
$T_{4, 2}^{1000}$ & 4 & 128 & 2 & $1000$ & 413,440 \\
$T_{4, 3}^{1000}$ & 4 & 128 & 3 & $1000$ & 611,712 \\
$T_{4, 4}^{1000}$ & 4 & 128 & 4 & $1000$ & 809,984 \\
$T_{4, 5}^{1000}$ & 4 & 128 & 5 & $1000$ & 1,008,256 \\
$T_{4, 6}^{1000}$ & 4 & 128 & 6 & $1000$ & 1,206,528 \\
\bottomrule
\end{tabular}
\caption{Hyper-parameters of all $\mtran{m, L, H}N$ trained for the empirical analysis.}
\label{table:all}
\end{table}

\paragraph*{Computational resources.}

All experiments were run on a 2021 Macbook Pro with an M1 chip.

\newpage
\subsection{Exponential increases in $k$-hop capacity with depth (\Cref{ec:depth}; \Cref{fig:err-k-4H-infn,fig:err-L-4H-infn,fig:err-table-4H-infn})}\label{assec:exp-depth}

\begin{figure}[h]
\centering
\includegraphics[scale=0.6]{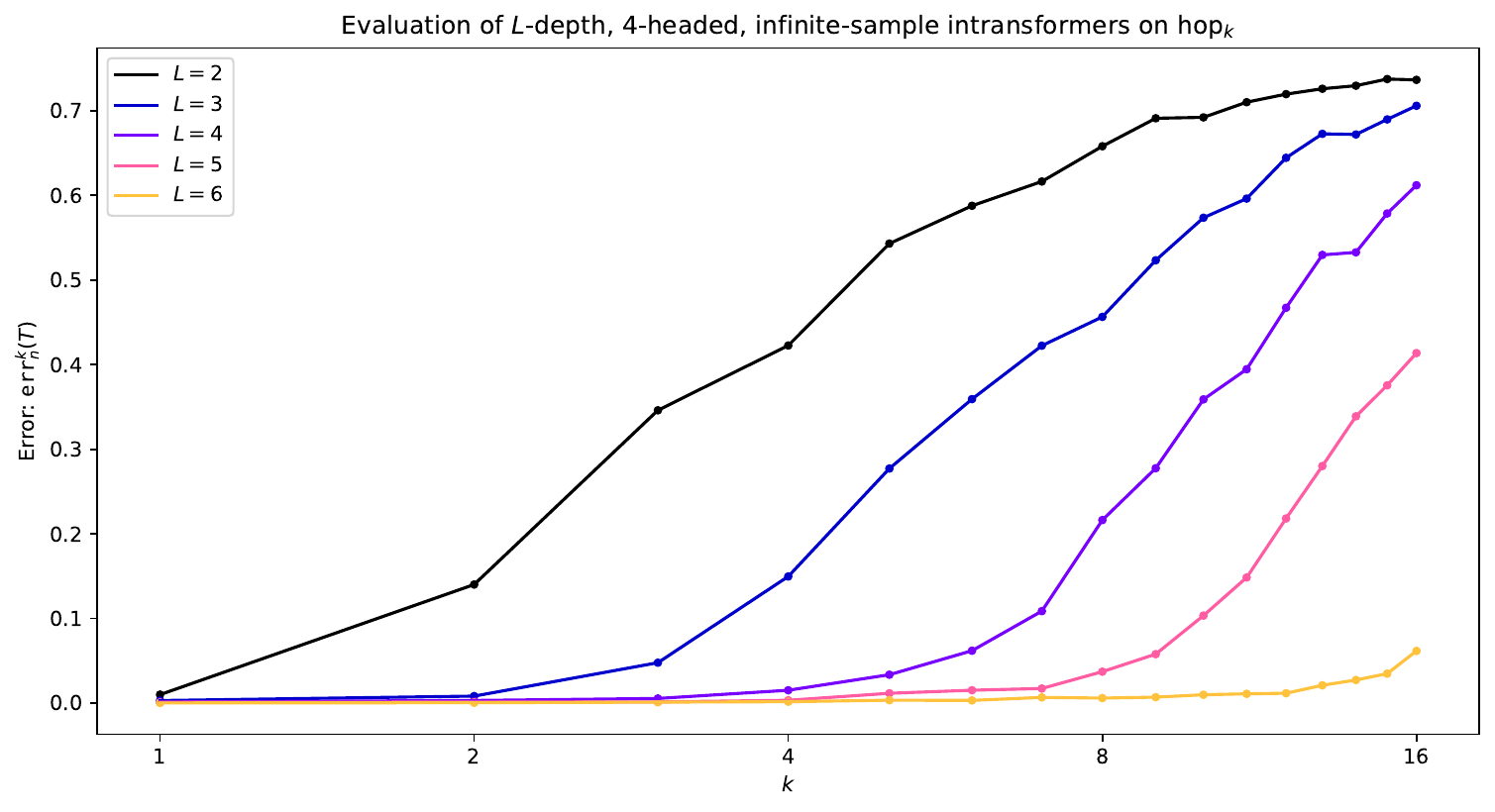}
\caption{Zoomed in version of \Cref{fig:depth-body}. Evaluation of transformers $\errk(T_{4, L}^\infty)$ with depths $L \in \set{2, 3, 4, 5, 6}$, heads $H = 4$, and embedding dimension $m = 128$ trained on the multi-hop task.
This figure plots $\errk(T_{4, L}^\infty)$ on $n=100$ samples as a function of $k$ for each choice of $L$. }
\label{fig:err-k-4H-infn}
\end{figure}

We visualize the relationship between the depth $L$ of a transformer and the largest $k$ such that $\errk(T)$ is small in \Cref{fig:err-k-4H-infn}, \Cref{fig:err-L-4H-infn}, and \Cref{fig:err-table-4H-infn}.
We exhibit the relationship in its simplest form by considering transformers with heads $H =4$, embedding dimension $m = 128$, and new training samples on every epoch.
The figures provide alternate views of $\errk(T_{4, L}^\infty)$ for each $L \in \set{2, 3, 4, 5, 6}$ with $n = 100$ samples for each $k \in [\kmax]$.

Together, these plots illustrate a sharp phase transition when $D = \floor{\log_2 k} + 2$, which identically matches the depth scaling in \Cref{thm:k-hop-construction}.
Increasing the depth of a transformer by one approximately doubles the number of values $k \in [\kmax]$ with bounded error.
For instance, following the theoretical and empirical intuition of \cite{bcbjg23}, the depth $L=2$ transformer $T_{4, 2}^\infty$ succeeds in solving the standard induction heads task, but attains at least $10\%$ error on all other tasks.
Likewise, a depth $L = 3$ model has error bounded by $1\%$ for $k \in\set{1,2}$, which increases rapidly for larger values of $k$.

This doubling phenomenon suggests that simple compositional tasks with a larger number of compositions than the depth of the model are easily learnable if the model can employ a doubling trick, similar to the one used in the proof of \Cref{thm:k-hop-construction}.
This relationship between compositionality and depth reflects the results of \citet{zbbegw23}, where the learnable task complexity also scales super-linearly in depth.

Given the lower bounds of \Cref{cor:k-hop-hardness}, one may ask why models with depth $L < \floor{\log_2 k}$ achieve non-trivial success on $\khop$ tasks that cannot be represented in a compositional manner.
There are several relevant explanations:
\begin{enumerate}
\item In these experiments, the embedding dimension $m = 128$ is actually larger than the context $N = 100$, which may enable the model to memorize more of its preceding samples and offload logical work to the MLP, rather than executing a pointer-doubling strategy. 
While practical models regularly have the opposite (and our theoretical results are oriented around that parametric scaling), we used a larger $m$ than is necessary for representational purpose to improve the optimization landscape and speed convergence.  
\item This is made further plausible by the small alphabet size $|\Sigma|$ and randomly drawn sequences $X'$, which place effective bounds on how much look-back from each token $i$ is necessary to compute $\khop(X)_i$.
\end{enumerate}

Nonetheless, these results provide strong support that models are substantially easier to train to low classification error in the regime where the depth is sufficient to implement a pointer-doubling construction.
In the following subsection, we further investigate this phenomenon by examining the intermediate attention matrices produced by trained models.

\begin{figure}\centering
\includegraphics[scale=0.6]{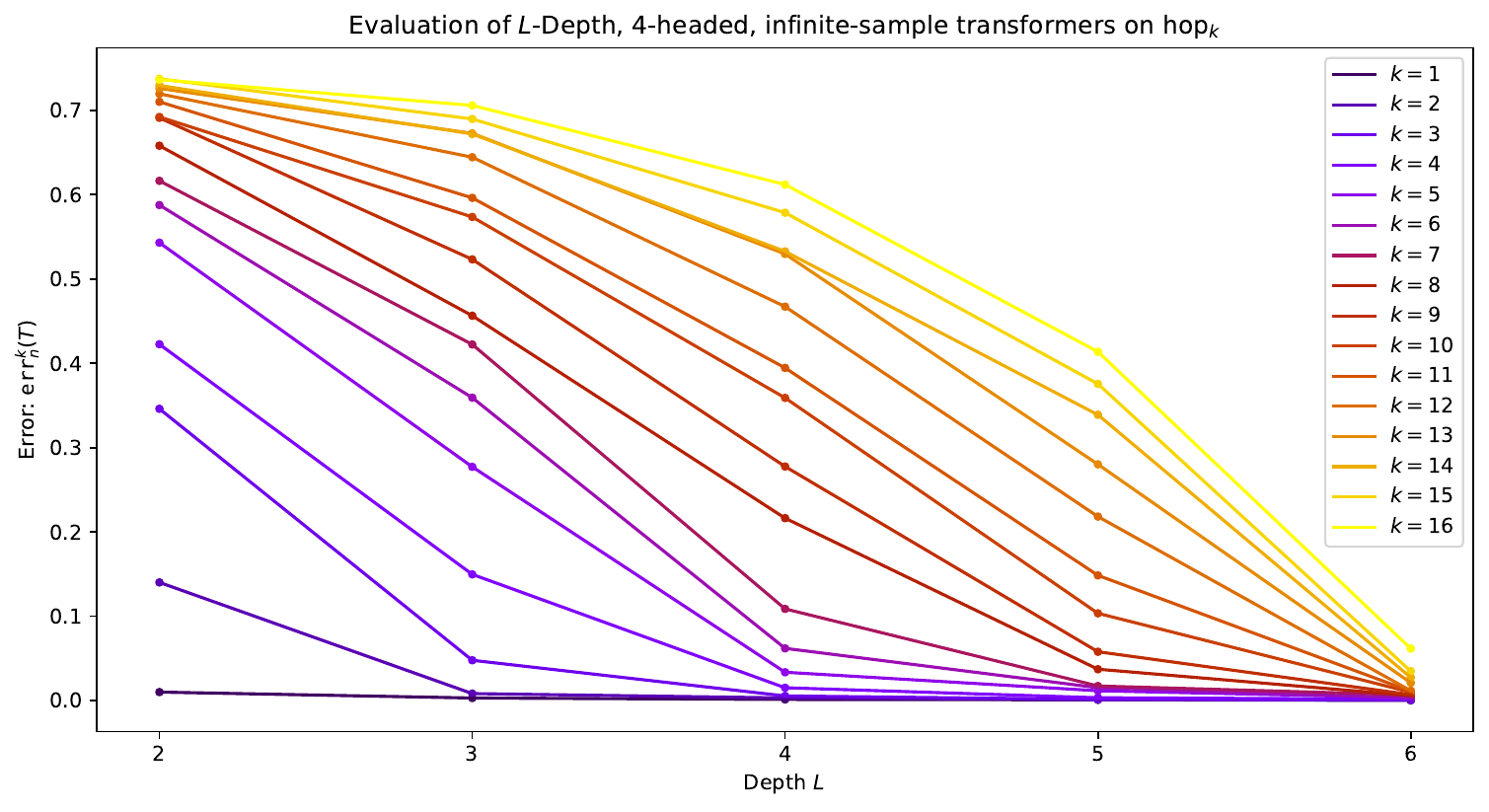}
\caption{Alternate view of \Cref{fig:err-k-4H-infn} including $\errk(T_{4, L}^\infty)$ plotted as a function of $L$ for each $k$.}
\label{fig:err-L-4H-infn}
\end{figure}

\begin{figure}\centering
\includegraphics[scale=0.6]{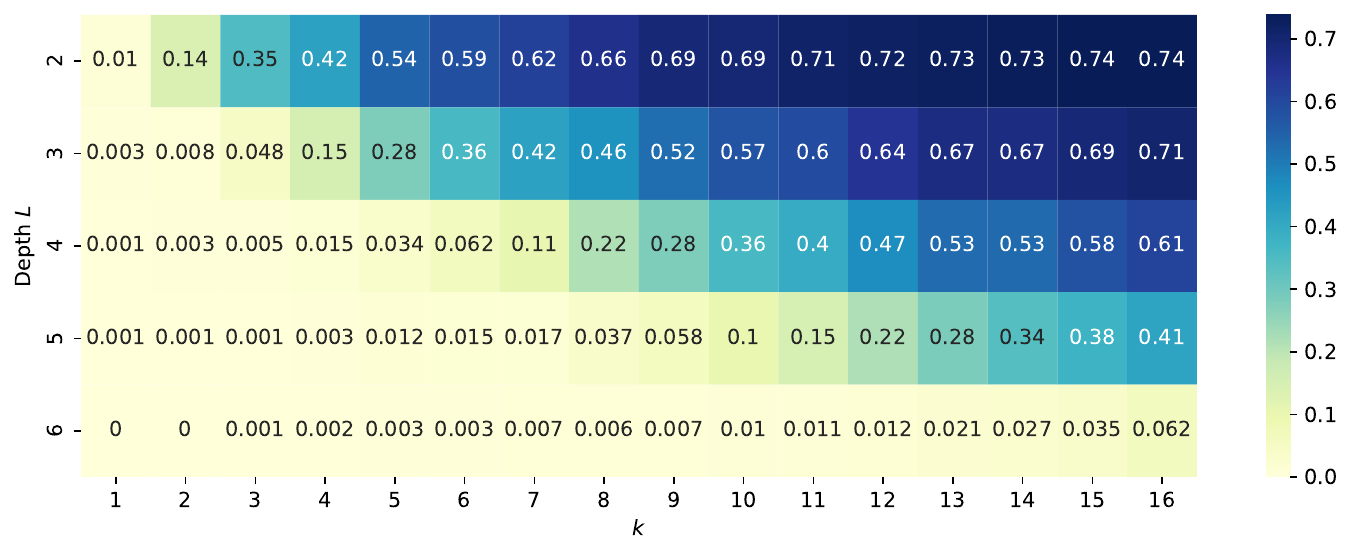}
\caption{Alternate views of \Cref{fig:err-k-4H-infn} including $\errk(T_{4, L}^\infty)$ as a table with one cell for each $(L, k)$ pair.}
\label{fig:err-table-4H-infn}
\end{figure}

\newpage
\subsection{Width variation (\Cref{ec:depth}; \Cref{fig:err-k-48H-infn})}\label{assec:exp-width}

While the primary focus of these empirical results and the paper as a whole is on the role of depth in the ability of transformer to learn parallelizable and compositional tasks, we also aim to understand the interplay of depth and width in learning the multi-hop task.
Here, we contrast the previous transformers $T_{4,L}^\infty$ with models $T_{8,L}^\infty$ that have more heads ($H = 8$) and larger embedding dimensions ($m = 256$).
We plot the classification errors of all 10 architectures over 16 $\khop$ sub-tasks in \Cref{fig:err-k-48H-infn}.

Here, we observe a rough correspondence in performance between the transformers $T_{H, L}^\infty$ and $T_{2H, L-1}$ and the same doubling phenomenon as is evident models with $H=4$ heads. 
That is, while increasing the width improves the classification error of learned models, it does so in a far less parameter-efficient manner than incrementing the depth.
As mentioned before, the relative success of wide and shallow transformers is likely contingent on the relatively short context length $N$ and alphabet size $|\Sigma|$.
However, these results still suggest an important role for wider models to play beyond representational capabilities of transformers.

\begin{figure}[H]
\centering
\includegraphics[scale=0.6]{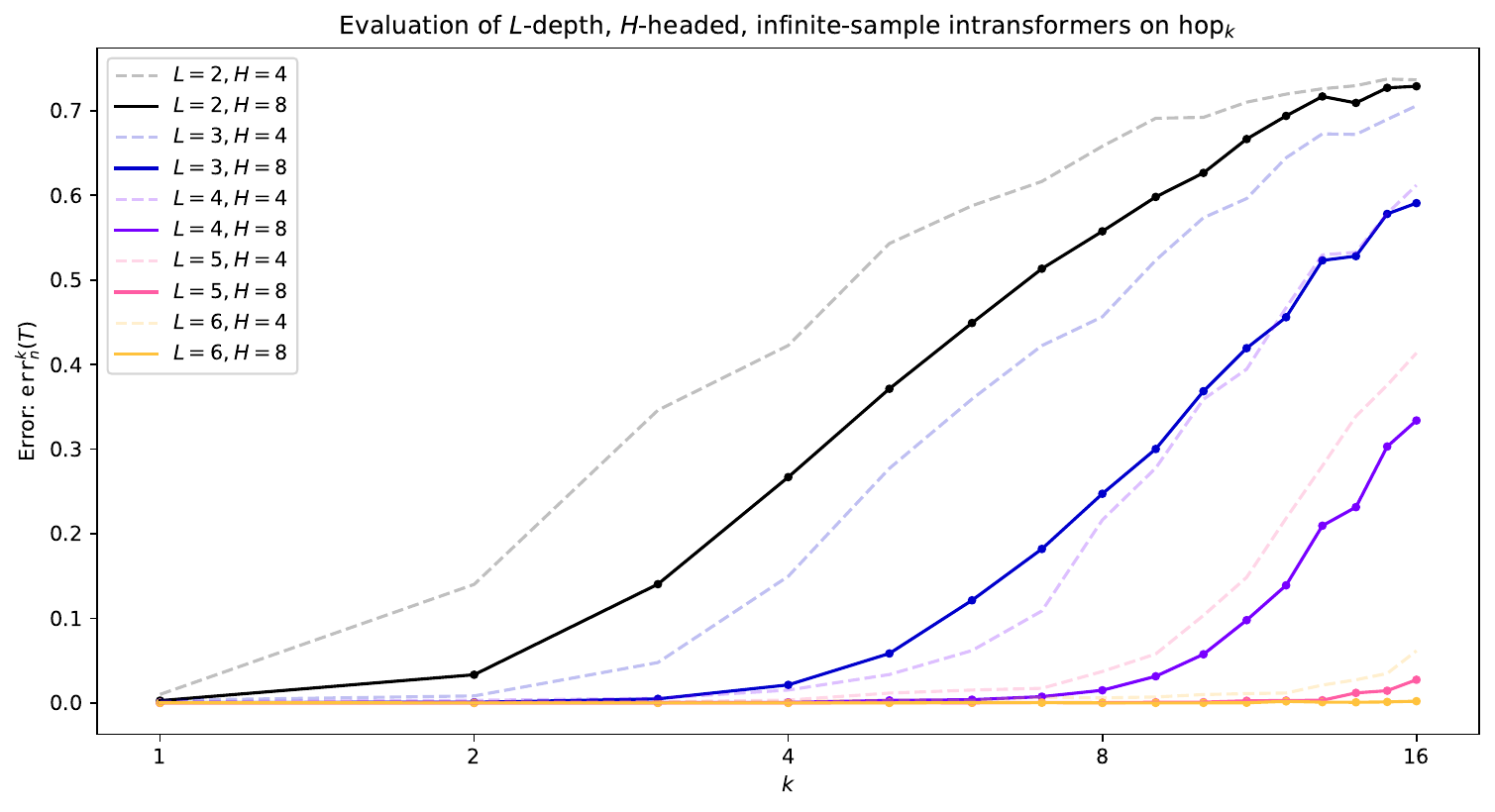}
\caption{Comparison between the errors $\errk(T_{H, L}^\infty)$ of transformers with embedding dimension and heads  $(m, H) = (4, 128)$ (dashed line, same plots as \Cref{fig:err-k-4H-infn}) and $(m, H) = (8, 256)$ (solid line) trained on the multi-hop task, evaluated on $n = 100$ samples per $\khop$ task.}
\label{fig:err-k-48H-infn} 
\end{figure}

\newpage
\subsection{Mechanistic alignment with theoretical construction (\Cref{ec:interp}, \Cref{fig:attens-6L-4H-infn,fig:inner-products-4-4,fig:inner-products-6-16,fig:inner-products-4-16,fig:max-inner-products-4,fig:max-inner-products-6})}\label{assec:exp-interp}

We use standard attention-based interpretability techniques to better understand what particular logical circuits are implemented by transformers trained to solve the multi-hop task.
By qualitatively inspecting the attention matrices produced by trained models and by measuring the alignment between those inner products and partial solutions $\find^j$ of $\khop$, we uncover a striking correspondence between the behaviors of the trained models and the transformer construction designed in the proof of \Cref{thm:k-hop-construction}. 
We further observe that trained transformers with high accuracy have ``decisive'' self-attention units with particularly strong correlations to some $\find^j$ intermediate, while poorly performing models have less predictable attention activations.

For a fixed trained model $T \in \tran{m, L, H}N$, we let $A^{\ell, h}[T](X)$ represent the output of the $h$th self-self attention matrix in the $\ell$th layer for $h \in[H]$ and $\ell \in [L]$, evaluated at some input $X \in \dom(\khop)$.
That is, we let \[A^{\ell, h}[T](X) = \sm\paren{Q^{\ell, h}(X^{\ell-1}) K^{\ell, h}(X^{\ell-1})^\T + \Gamma} \in \R^{N \times N},\]
where $X^{\ell - 1}$ is the intermediate state representing the output of layer $\ell -1$ of $T$ on input $X$ and $\Gamma$ is the causal masking matrix.
Each row $i$ in the matrix represents the coefficients of the convex combination of value vectors affiliated with each query, which can be used as a signifier of which embeddings $i$ receives information from. 

\paragraph*{Visualization of $\find^j$ alignment for $\jhop{16}$ and depth $L = 6$ (\Cref{fig:attens-6L-4H-infn}).}

The outputs of self-attention matrices are often highly structured matrices that reveal which relationships between tokens are encoded and how information is shared within the model \citep{lm22,clark2019does,rogers2021primer}. 
We plot several self-attention matrices associated with a depth $L = 6$, heads $H = 4$ transformer trained in the infinite-sample regime and evaluated on a single sample $X \in \dom(\jhop{16})$ in \Cref{fig:attens-6L-4H-infn}.

By looking at the six self-attention matrices, one can infer that all heads are ``decisive'' and obtain nearly all of their relevant information from a single value embedding, rather than averages of a large number of embeddings.
The top-left self-attention matrix, which belongs to the first self-attention head, clearly associates elements with their predecessors, which is identical the to the function of our $\lookup$ attention head in the first layer of the $\khop$ construction of \Cref{thm:k-hop-construction}.

While the roles of the other heads are not immediately obvious, they can be understood by overlaying colored matrices with non-zero cells at $(i, \find^j_X(i))$ for some $j \leq k$.
For instance, the top-right attention matrix in layer $\ell =2$ corresponds almost exactly with $\find^1_X$ (as suggested by the second-layer of our construction), and the others are closely associated with $\find^1_X$, $\find^2_X$, $\find^3_X$, and $\find_X^8$ for layers $\ell = 3, 4, 5, 6$ respectively.
This is a remarkably close correspondence to our construction, which includes a self-attention matrix in the $\ell$th layer whose activations correspond to $\find^{2^{\ell-2}}_X$.

While not conclusive, this experiment suggests a strong alignment between the behaviors of this particular transformer and our theoretical construction. 
This suggests a high likelihood that the transformer successfully learns to solve $\jhop{16}$ by employing a pointer-doubling primitive.
However, these results apply to only a single model, a single task, and a single input; in the subsequent section, we generalize this interpretability analysis.

\begin{figure}\centering
\includegraphics[scale=0.575]{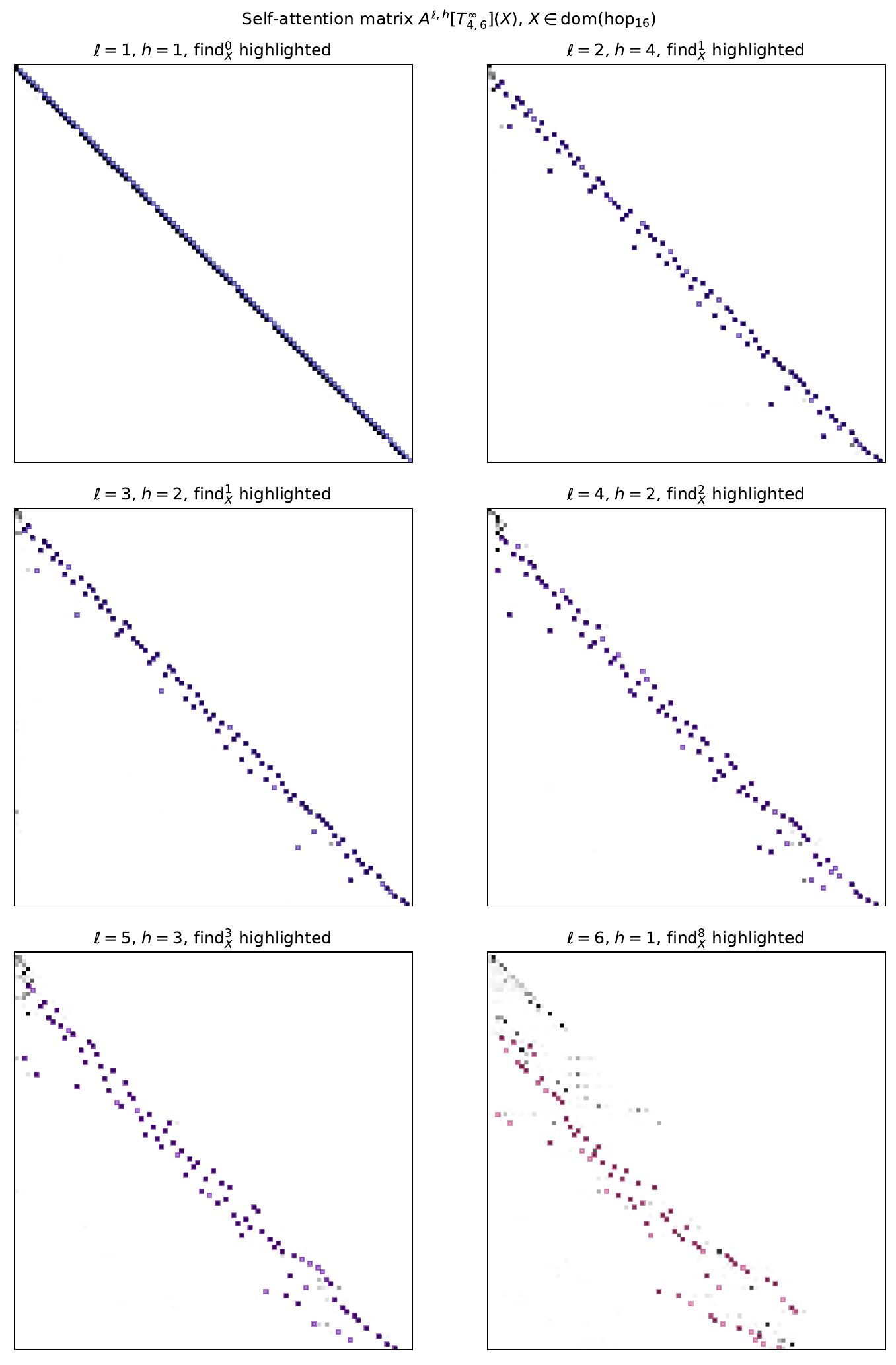} \caption{The outputs of several internal self-attention matrices $A^{\ell, h}[T_{4, 6}^\infty](X) \in \R^{100 \times 100}$ of a trained multi-task transformer of depth $D = 6$ evaluated on a single sample $X \sim \distk{16}$ are plotted in grayscale.
In each cell, the matrix with non-zero entries $(\find_X^j(i), i)_{i \in [N]}$ for some $j$ is included in transparent color to visualize the function of each self-attention unit.
}
\label{fig:attens-6L-4H-infn}
\end{figure}

\paragraph*{Alignment between attention heads and $\find^j$ for a single $\khop$ sub-task (\Cref{fig:inner-products-4-4,fig:inner-products-6-16,fig:inner-products-4-16}).}
To broaden and quantify the analysis of the previous section, we measure the extent to which each self-attention head mimics the functionality of $\find^j$, which are partial computations of $\khop$ that are employed in the proof of \Cref{thm:k-hop-construction}.
We use cell-wise matrix inner products to quantify the strength of correlation between a self-attention matrix and a fixed function potentially relevant to interpretability.

For two matrices $A, B \in \R^{N \times N}$, let \[\inner{A, B} = \frac{\norm[F]{A \odot B}^2}{\norm[F]{A} \norm[F]{B}}\] be their normalized element-wise inner-product, where $\norm[F]{\cdot}$ is the Frobenius norm and $\odot$ denotes element-wise multiplication.
For some function $g: [N] \to \set{0} \cup [N]$, we let $\inner{g, B} := \inner{A^g, B}$, where \[A^g_{i, j} = \begin{cases}1 & \text{if} \ g(j) = i, \\ 0 & \text{otherwise.} \end{cases}\]

We use this notation to analyze experimentally how closely the self-attention matrices $A^{\ell, h}$ encode the intermediate products of the proof of \Cref{thm:k-hop-construction}, $\find^j_X$.
For $n$ iid samples $X^1, \dots, X^n \in \sim \distk{k}$, let \[\inner{A^{\ell, h}, \find^j}_{n, k} := \frac1n \sum_{\iota=1}^n \inner{\find^j_{X^\iota}, A^{\ell, h}(X^\iota)}.\]
Due to the non-negativity of $A^{\ell, h}$ and $\find^j$, $\inner{A^{\ell, h}, \find^j}_{n, k} \in [0, 1]$, and $\inner{A^{\ell, h}, \find^j}_{n, k} = 1$ only if $\forall \iota \in [n]$: \[A^{\ell, h}(X^\iota)_{i, i'} = 1 \ \iff \ \find^j_{X^\iota}(i) = i'.\] 

These inner products make it possible to visualize the strength of correlations of all heads in a particular model $T \in \mtran{m, L, H}N$ with all target functions $\find^j$ on a collection of random samples drawn from some $\distk{k}$.
\Cref{fig:inner-products-4-4} visualizes the functionality of all attention units in the 4-layer, 4-head transformer $T_{4, 4}^\infty$ when evaluated on the sub-task $\jhop4$. 
The figure gives several clues about how $\jhop4$ is successfully computed by the trained model: the second layer and third layer both utilize $\find^1$ to determined $\find^2$ jointly by the end of the third layer.
The fourth layer uses the ability to create a stable $\find^2$ construction to obtain $\find^4$ and hence $\jhop4$.

This plot also indicates the relative stability of this circuit interpretation of the procedure: a large number of heads are very strongly correlated with $\find^1$ or $\find^2$ across the 10 samples, which indicates they are likely utilized consistently to compute those intermediates regardless of input.

\Cref{fig:inner-products-6-16} is a similar plot for the transformer $T_{4, 6}^\infty$ with depth $L = 6$, evaluated on the task $\jhop{16}$.
The functionalities of the heads visualized in \Cref{fig:attens-6L-4H-infn} can be observed in the corresponding inner products.
The collection of all inner products presents further evidence that the pointer-doubling phenomenon occurs in the trained models, due to the increase in compositions present in the largest inner products of deeper attention units.

While \Cref{fig:inner-products-4-4,fig:inner-products-6-16} showcase the decisive alignment between self-attention heads and particular partial computations $\find^j$ in successfully trained models, \Cref{fig:inner-products-4-16} demonstrates the loss of that decisiveness in poorly performing transformers.
There, we visualize the alignments of the trained depth-4 transformer $T_{4,4}^\infty$ evaluated on $\jhop{16}$, in which it attains a 61\% token error.
While a self-attention units in the second layer coincides with $\find^1$, no strong correlations emerge deeper in the model. 
Unlike the other figures, the deeper self-attention units are ``indecisive,'' lacking any large inner products and failing in particular to correlate with any highly compositional targets.
This provides a visual explanation of the transformer's failure, since it lacked the effective representational capacity needed to learn a circuit with consistent and highly-compositional outputs.\footnote{Since these experiments are in the small alphabet size $|\Sigma| = 4$ regime, this task performs better than random guessing due to inferential capabilities that are are powered by the high embedding dimension and do not require implementing a pointer-chasing algorithm. We suspect that the ``checkerboard'' patterns are powered by this inference.}

\paragraph*{Alignment between attention heads and $\find^j$ for all $\khop$ sub-tasks (\Cref{fig:max-inner-products-4,fig:max-inner-products-6}).}

For an even more global lens on the mechanistic interpretability of these trained models, we visualize how the maximum inner products of each self-attention unit change for a fixed transformer for different sub-tasks $\khop$.
Figures~\ref{fig:max-inner-products-4} and \ref{fig:max-inner-products-6} do so for the depth-4 and depth-6 networks respectively.
The hue of each cell (and its numerical label) corresponds to the $j^*$ with the most correlated inner product with corresponding attention unit $A^{\ell, h}$ in samples from $\dom(\khop)$, and the opacity corresponds to the magnitude of that inner product.

The takeaways of the previous inner product figures are apparent in these: 
the approximate doubling for the depth $L=6$ transformer can be visualized by the vertically changing opaque colors. 
Conversely, a separation can be observed between the tasks where the depth $L=4$ transformer performs well and has ``decisive'' self-attention units deeper in the network and those where it does not.

Moreover, the figures (especially \Cref{fig:max-inner-products-6}) demonstrate that several self-attention units have a consistent function among samples from the same task, while adapting in function to different $\khop$ tasks.
This is most apparent in head $h = 4$ of layer $\ell = 6$, where the self-attention head functions as $\find^1, \find^3, \find^5$ or $\find^7$ depending on the complexity of the task.

\begin{figure}
\centering
\includegraphics[scale=0.6]{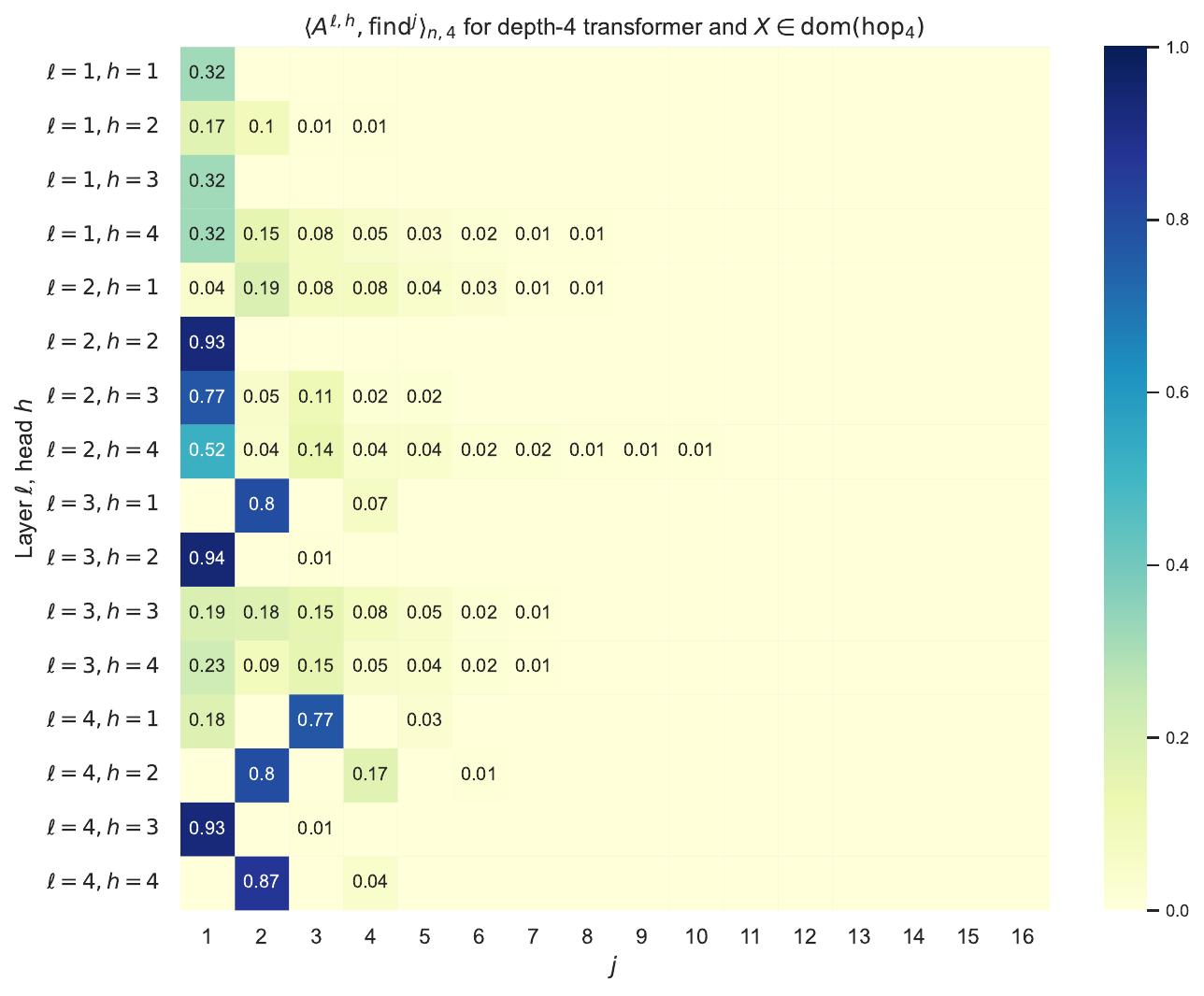}
\caption{Plots of all inner products $\inner{A^{\ell, h}[T_{4, 4}^\infty], \find^j}_{10, 4}$ for $n=10$ samples $X^1, \dots, X^{10} \in \dom(\jhop4)$ for the 4-layer transformer $T_{4, 4}^\infty$.}
\label{fig:inner-products-4-4} 
\end{figure}

\begin{figure}
\centering
\includegraphics[scale=0.6]{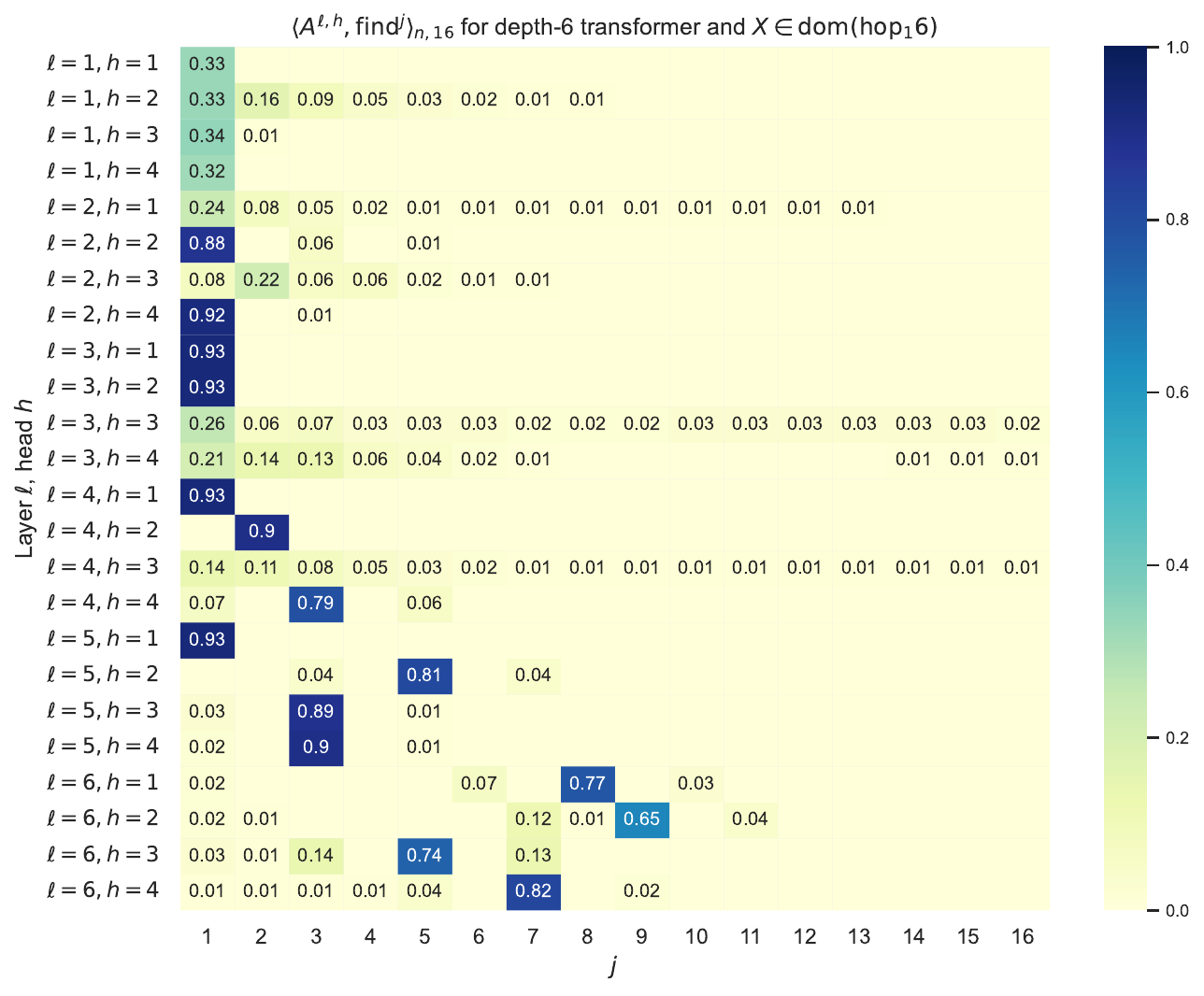}
\caption{Plots of all inner products $\inner{A^{\ell, h}[T_{4, 6}^\infty], \find^j}_{10, 16}$ for $n=10$ samples $X^1, \dots, X^{10} \in \dom(\jhop{16})$ for the 6-layer transformer $T_{4, 6}^\infty$.}
\label{fig:inner-products-6-16} 
\end{figure}

\begin{figure}
\centering
\includegraphics[scale=0.6]{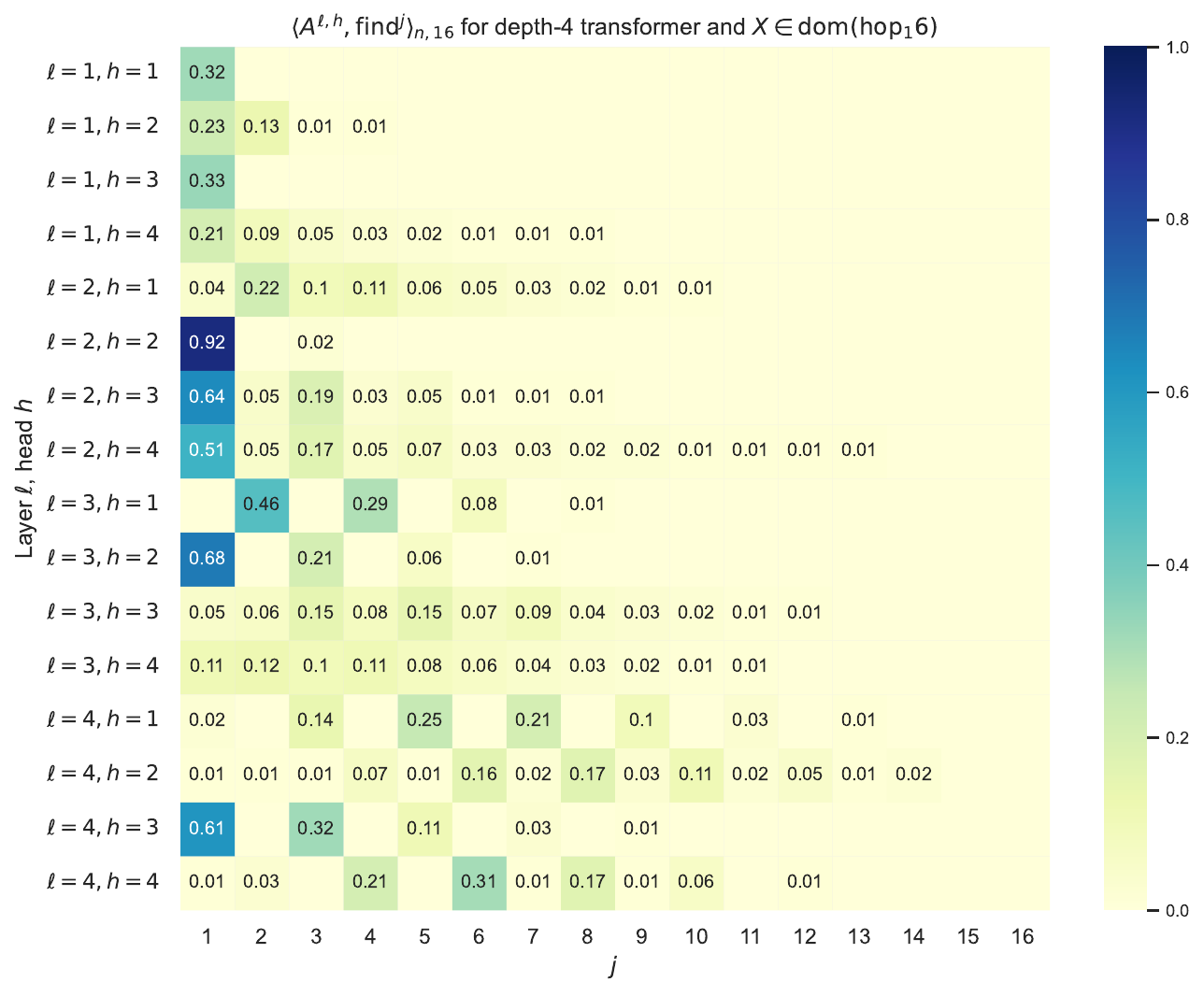}
\caption{Plots of all inner products $\inner{A^{\ell, h}[T_{4, 4}^\infty], \find^j}_{10, 16}$ for $n=10$ samples $X^1, \dots, X^{10} \in \dom(\jhop{16})$ for the 4-layer transformer $T_{4, 4}^\infty$.}
\label{fig:inner-products-4-16} 
\end{figure}

\begin{figure}
\centering
\includegraphics[scale=0.6]{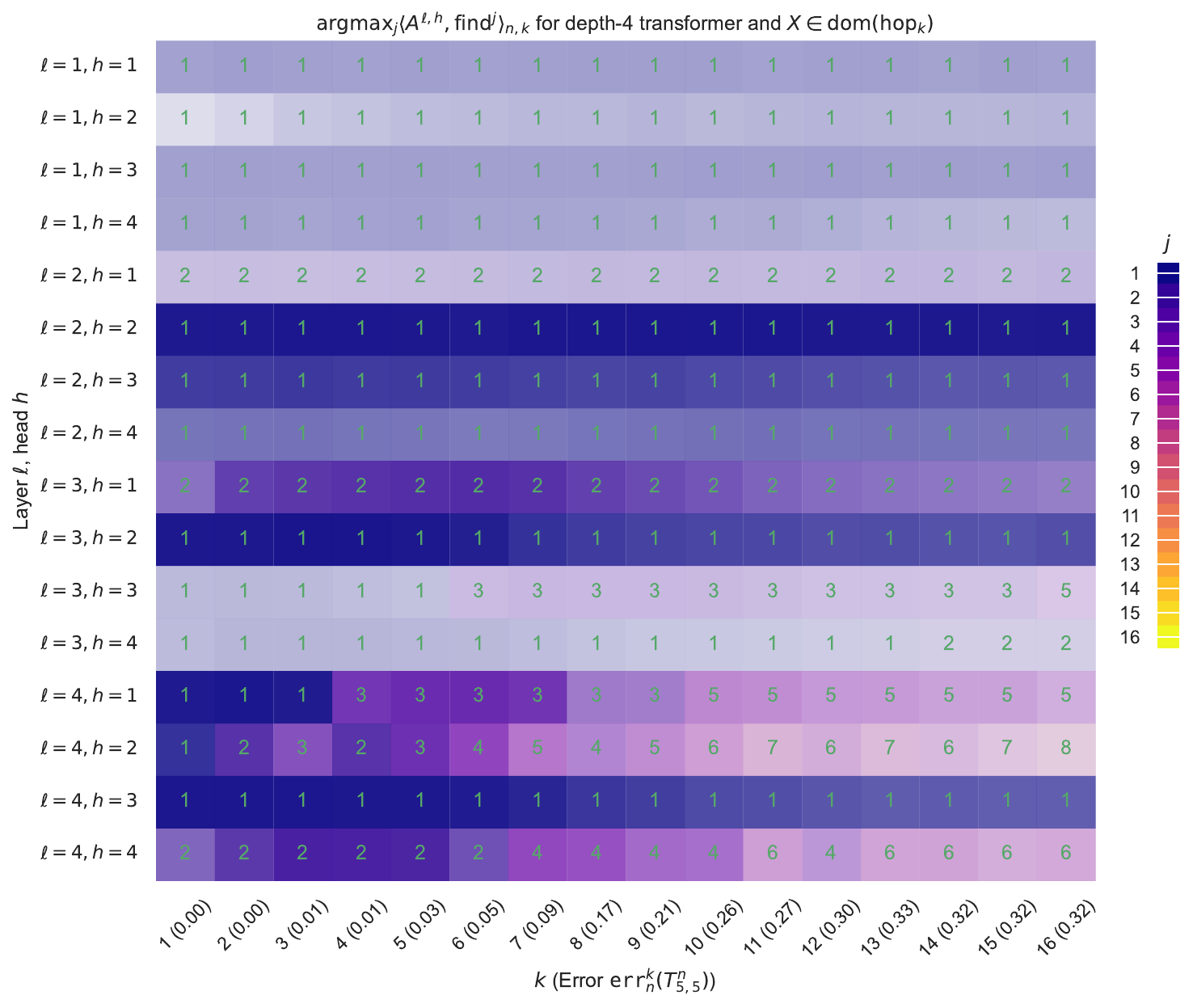}
\caption{Plots of all the maximum inner products $\inner{A^{\ell, h}[T_{4,4}^\infty], \find^j}_{n, k}$ for $n=10$ fixed samples $X^1, \dots, X^{10} \in \dom(\jhop{k})$ for each $k \in [16]$ for the 4-layer transformer $T_{4,4}^\infty$. The hue corresponds to the index of the largest inner product $j^* = \argmax_j \inner{A^{\ell, h}[T_{4,4}^\infty], \find^j}_{n, k}$, while the opacity is determined by the magnitude of the correlation.}
\label{fig:max-inner-products-4} 
\end{figure}

\begin{figure}
\centering
\includegraphics[scale=0.6]{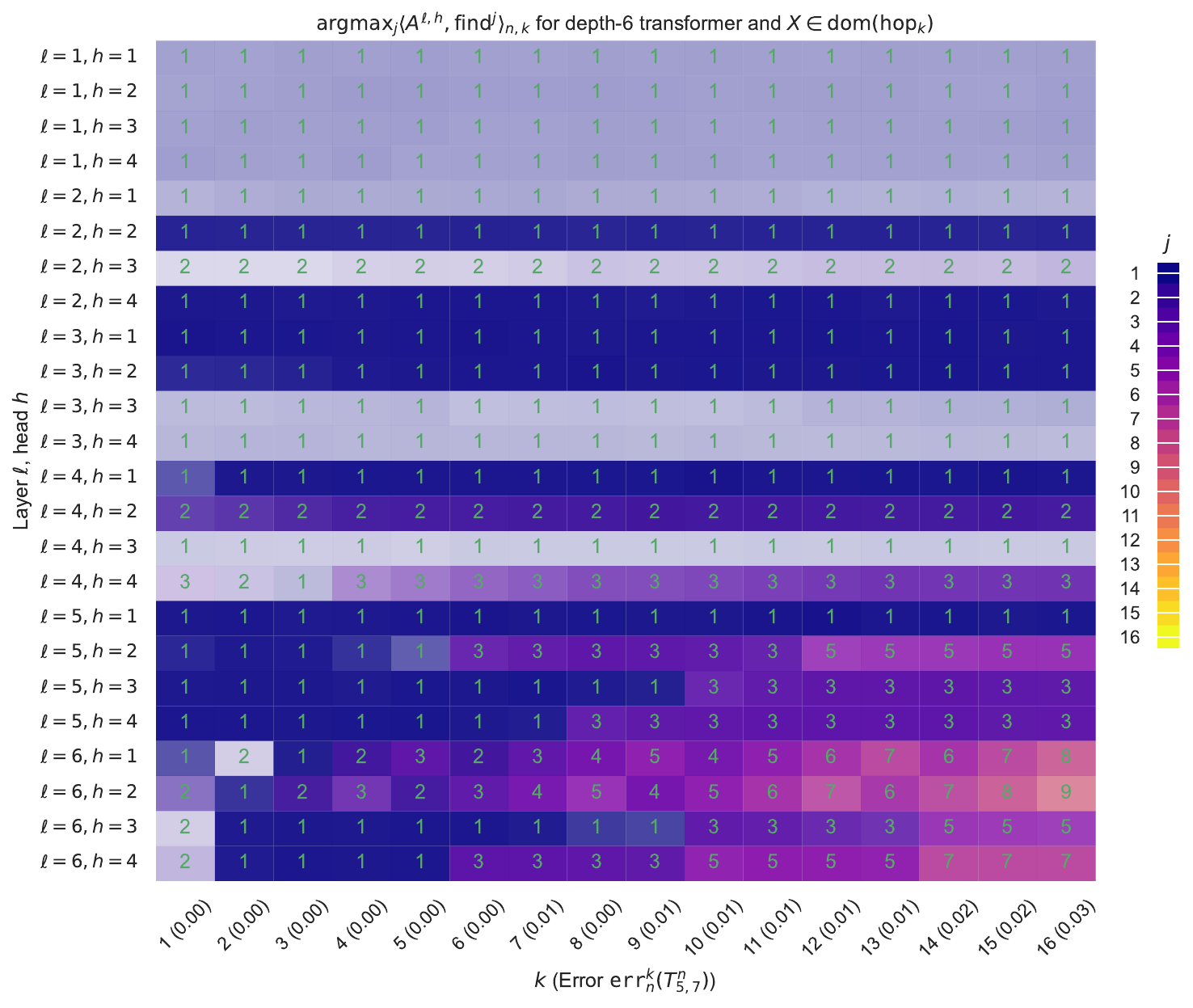}
\caption{Plots of all the maximum inner products $\inner{A^{\ell, h}[T_{4,6}^\infty], \find^j}_{n, k}$ for $n=10$ fixed samples $X^1, \dots, X^{10} \in \dom(\jhop{k})$ for each $k \in [16]$ for the 6-layer transformer $T_{4,6}^\infty$.}
\label{fig:max-inner-products-6} 
\end{figure}

\newpage
\subsection{Finite-sample experiments (\Cref{ec:sample}; \Cref{fig:err-k-4H-3000infn,fig:err-k-4H-1000infn,fig:inner-products-4-3-1000,fig:inner-products-6-3-1000})}\label{assec:exp-finite}

While most of our multi-hop experiments reside in the infinite-sample regime (where new samples are generated for every batch), we also trained several transformers on $\ntr \in \{1000, 3000\}$ samples to evaluate whether generalization is possible in this domain, especially when the number of model parameters far exceeds the number of training samples.
The two training set sizes expose a sharp threshold between two different generalization modes: low accuracy due to overfitting for most models on most tasks when $\ntr = 1000$ and high accuracy approaching the infinite-sample regime when $\ntr=3000$.

\Cref{fig:err-k-4H-3000infn} compares the infinite-sample transformers $T_{4,L}^\infty$ with the 3000-sample models $T_{4, L}^{3000}$.
3000 training samples are sufficient to obtain comparable (if slightly worse) generalization error rates across model depths $L$ and task complexities $k$.
This supports a hypothesis that the existence of a small transformer that perfectly fits the data enables larger transformers to actually realize such architectures in the over-parameterized regime.

On the other hand, \Cref{fig:err-k-4H-1000infn} demonstrates that transformers trained on $\ntr=1000$ samples suffer poor performance on most tasks due to overfitting.
While all models perform poorly on $\khop$ sub-tasks for large $k$, a depth-separation exists for simpler sub-tasks like $\jhop3$.
This suggests a positive inductive bias of deep transformers for simple compositional decision rules, which enables far better performance than other models in the overfitting regime.

To investigate this gap in performance, we contrast the self-attention inner products of depth-4 $T_{4, 4}^{1000}$ and depth-6 $T_{4, 6}^{1000}$ on the task $\jhop3$ in Figures~\ref{fig:inner-products-4-3-1000} and \ref{fig:inner-products-6-3-1000}.
The 6-layer model obtains a far superior classification error on the sub-task, and the interpretability plot establishes a plausible circuit it implements: It uses self-attention heads with $\find^1$ functionality consecutively in layers 4, 5, and 6, which enables the robust retrieval of $\find^3$ and $\jhop3$.
On the other hand, the 4-layer plot exhibits poor performance and only has two layers with $\find^1$ functionality; this justifies the relatively strong performance of $T_{4, 4}^{1000}$ on $\jhop2$ and its poor performance on $\jhop3$.

While neither model learns any kind of pointer-doubling construction, the 6-layer model is still able to learn a simple construction of $\jhop3$ that the 4-layer model misses.
The representational suitability of deeper models to compositional reasoning may thus provide a favorable inductive bias for learning the task in a setting with little data.

\begin{figure}[h]
\centering
\includegraphics[scale=0.6]{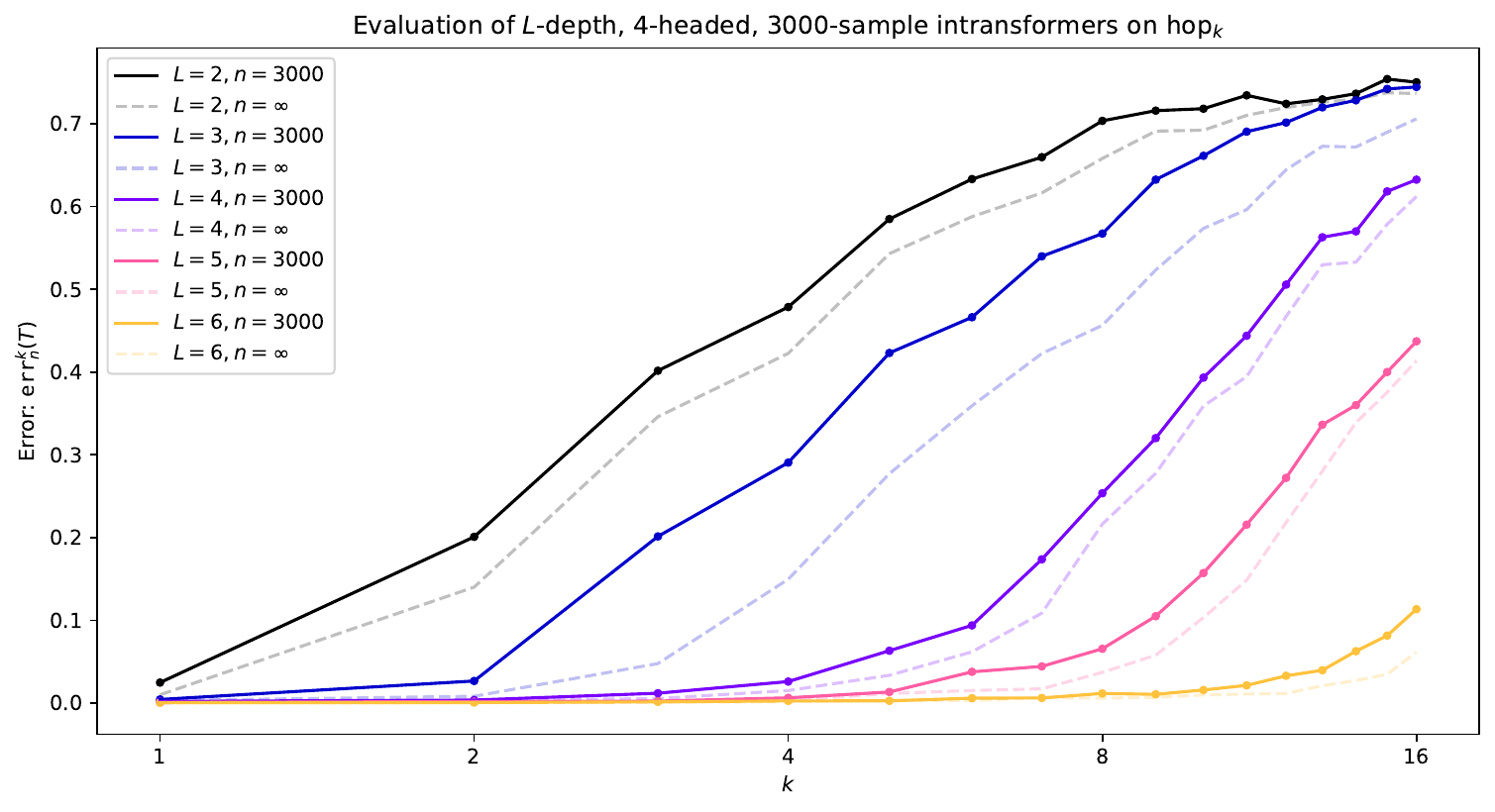}
\caption{Comparison between the errors $\errk(T_{4, L}^n)$ of transformers trained in the infinite sample regime (dashed line) and on $\ntr=3000$ samples (solid line) on the multi-hop task, evaluated on $n = 100$ samples per $\khop$ task.}
\label{fig:err-k-4H-3000infn}
\end{figure}

\begin{figure}[h]
\centering
\includegraphics[scale=0.6]{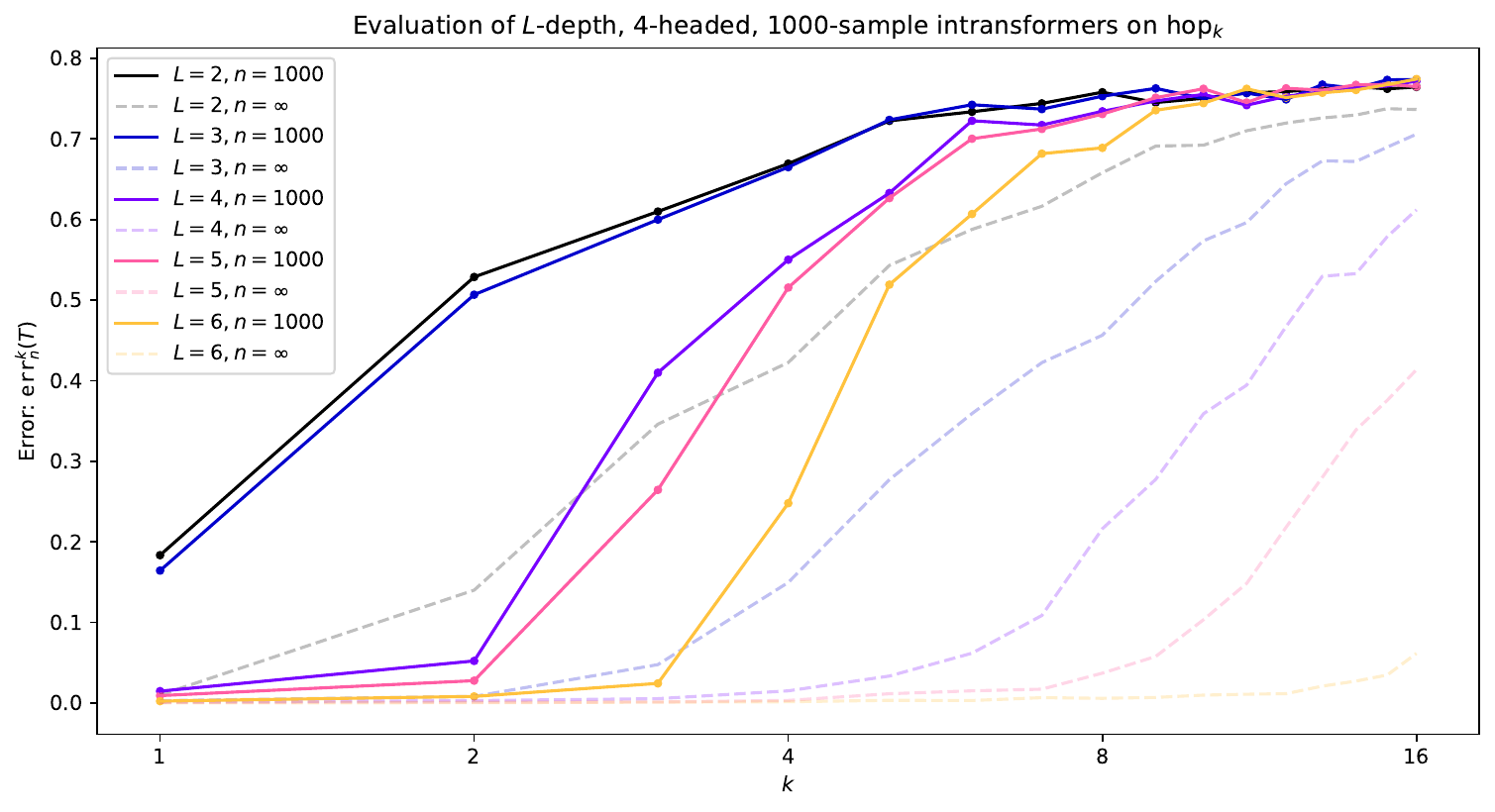}
\caption{Comparison between the errors $\errk(T_{4, L}^n)$ of transformers trained in the infinite sample regime (dashed line) and on $\ntr=1000$ samples (solid line) on the multi-hop task, evaluated on $n = 100$ samples per $\khop$ task.}
\label{fig:err-k-4H-1000infn}
\end{figure}

\begin{figure}[h]
\centering
\includegraphics[scale=0.6]{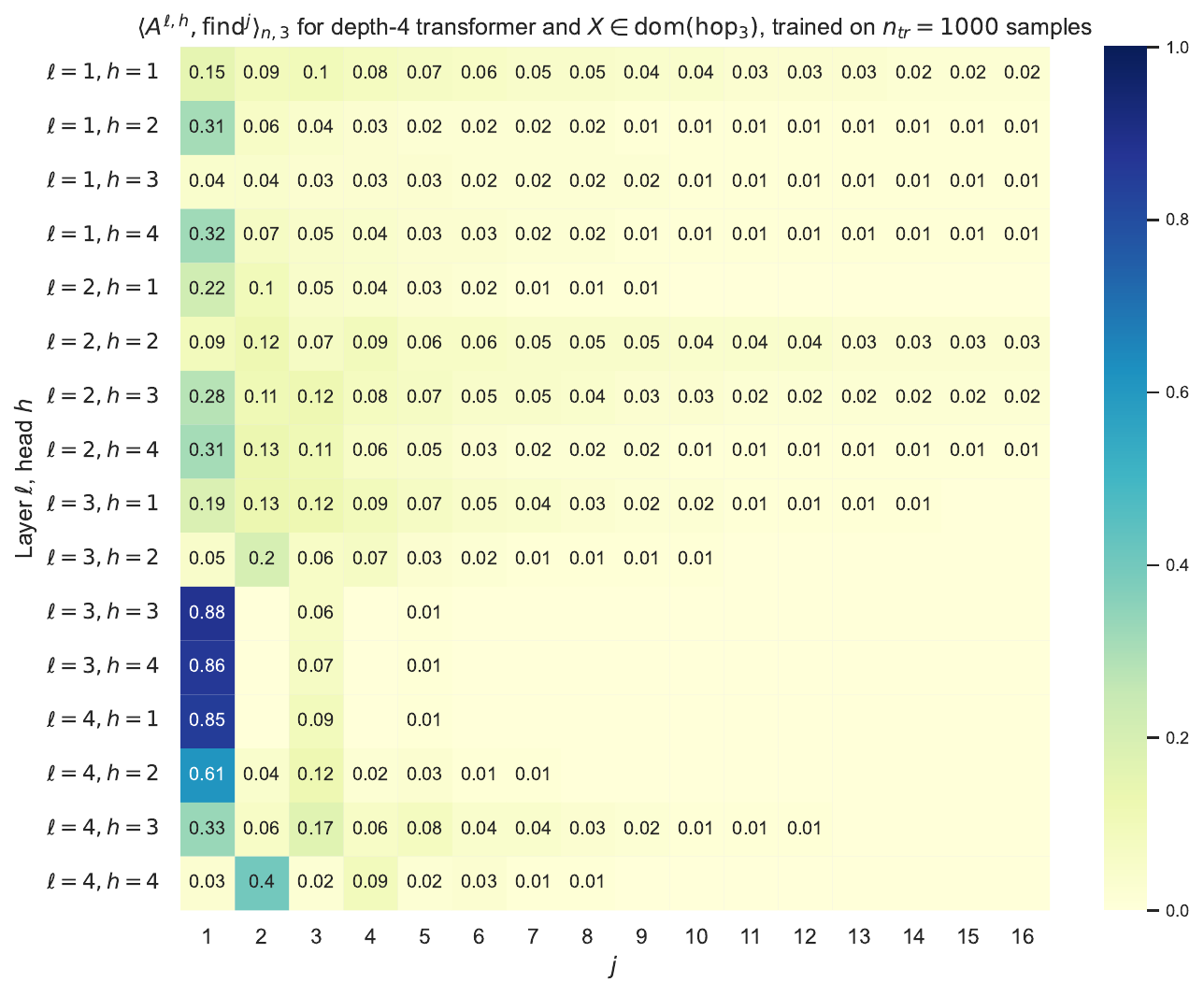}
\caption{Plots of all inner products $\inner{A^{\ell, h}[T_{4, 4}^{1000}], \find^j}_{10, 3}$ for $n=10$ samples $X^1, \dots, X^{10} \in \dom(\jhop{3})$ for the 4-layer transformer $T_{4, 4}^{1000}$.}
\label{fig:inner-products-4-3-1000} 
\end{figure}

\begin{figure}[h]
\centering
\includegraphics[scale=0.6]{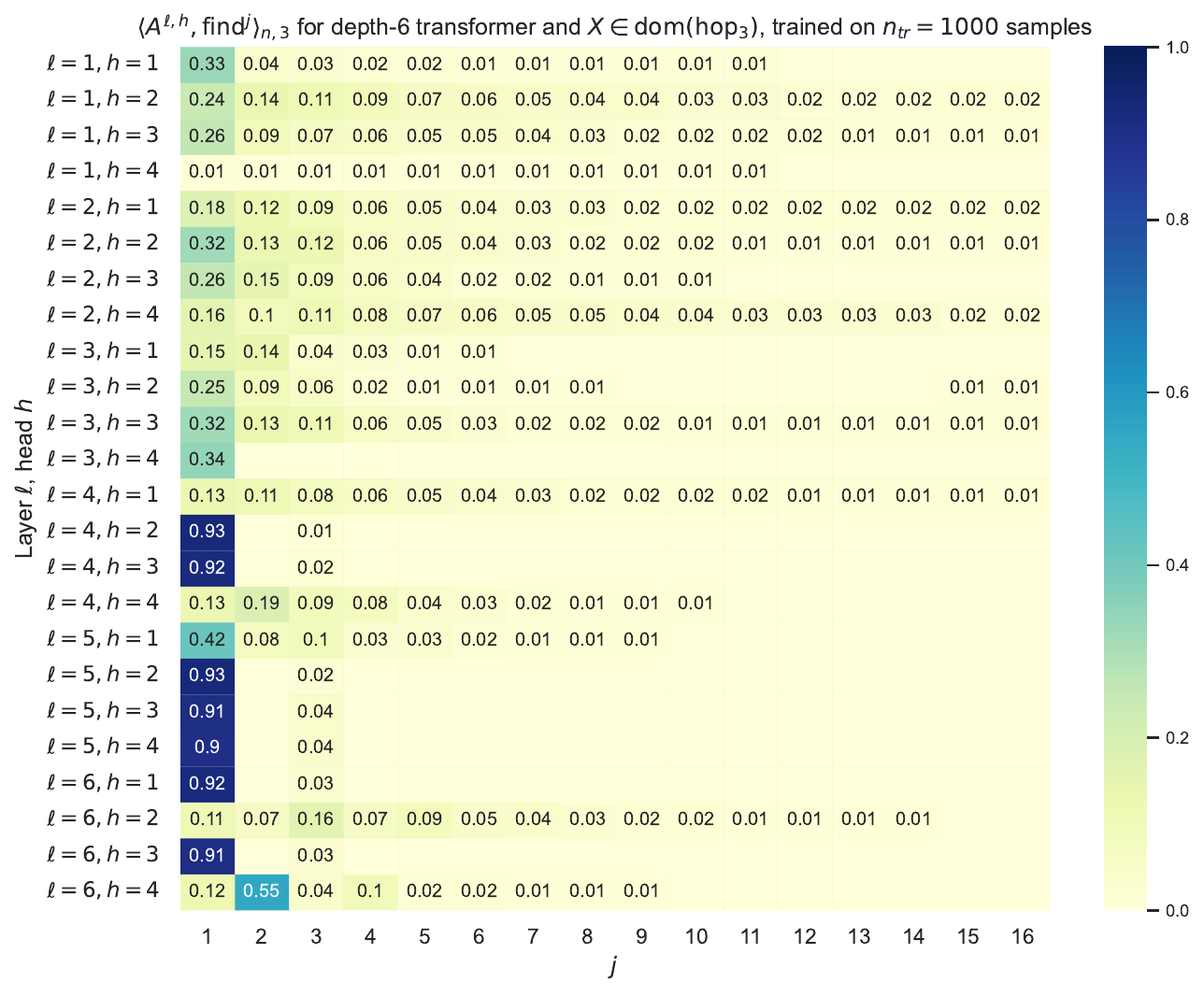}
\caption{Plots of all inner products $\inner{A^{\ell, h}[T_{4, 6}^{1000}], \find^j}_{10, 3}$ for $n=10$ samples $X^1, \dots, X^{10} \in \dom(\jhop{3})$ for the 6-layer transformer $T_{4, 6}^{1000}$.}
\label{fig:inner-products-6-3-1000} 
\end{figure}

\end{document}